\documentclass[twoside]{article}

\usepackage[accepted]{aistats2022}

\usepackage[round]{natbib}

\usepackage{xcolor}
\usepackage{verbatim}
\usepackage[colorlinks,citecolor=cyan,linkcolor=blue]{hyperref}

\usepackage[ruled,vlined]{algorithm2e}
\usepackage{microtype}
\usepackage{graphicx}
\usepackage{subfigure}
\usepackage{booktabs} %
\usepackage{todonotes}
\usepackage{bbm}
\usepackage{hyperref}
\usepackage{wrapfig} 
\usepackage{todonotes}
\usepackage{microtype}
\usepackage{graphicx}
\usepackage{booktabs} %
\usepackage{wrapfig}
\usepackage{tabularx}
\usepackage{amsmath}
\usepackage{amssymb}
\usepackage{mathrsfs}
\usepackage{amsthm}
\usepackage{stmaryrd}
\usepackage{mathtools}
\DeclarePairedDelimiter{\ceil}{\lceil}{\rceil}
\usepackage{bm}
\usepackage{dblfloatfix}
\usepackage{enumitem}
\usepackage[thinc]{esdiff}
\usepackage{bigints}
\usepackage{mdwlist}
\usepackage{thmtools, thm-restate}
\newtheorem{lemma}{Lemma}

\declaretheorem{definition}
\declaretheorem{assumption}
\usepackage{amsmath}
\usepackage{centernot}
\usepackage[nameinlink,capitalise]{cleveref}
\crefformat{equation}{#2 Equation~(#1) #3}
\crefname{figure}{Figure}{Figures}
\Crefname{figure}{Figure}{Figures}
\crefname{assumption}{Assumption}{Assumptions}
\crefname{table}{Table}{Tables}
\Crefname{table}{Table}{Tables}
\crefname{equation}{Equation}{Equations}
\Crefname{equation}{Equation}{Equations}
\crefname{section}{Section}{Sections}
\Crefname{section}{Section}{Sections}
\crefname{subsection}{Subsection}{Subsections}
\Crefname{subsection}{Subsection}{Subsections}

\usepackage[switch]{lineno}

\usepackage{tikz-cd}

\renewcommand{\ge}{\geqslant}
\renewcommand{\geq}{\geqslant}
\renewcommand{\le}{\leqslant}
\renewcommand{\leq}{\leqslant}

\newcommand{\deff}{d_{\mathrm{eff}}}
\newcommand{\norm}[1]{\lVert #1 \rVert}
\newcommand{\unorm}[1]{\lVert #1 \rVert_{S, 2}}

\newcommand{\ip}[2]{\ensuremath{\langle #1, #2 \rangle}}

\newcommand{\E}{\mathbb{E}}

\newcommand{\ind}{\mathbf{1}}

\DeclareMathOperator*{\Tr}{\mathrm{tr}}

\renewcommand{\Pr}{\mathbb{P}}
\newcommand{\T}{\mathsf{T}}

\newcommand{\cE}{\mathcal{E}}

\newcommand{\opnorm}[1]{\norm{#1}_{\mathrm{op}}}

\newcommand{\vfapprox}{V_{\phi,w}}
\newcommand{\vfhatapprox}{V_{\phi,\hat w}}
\newcommand{\vfoptapprox}{V_{\phi,w^*}}

\newcommand{\cG}{\mathcal{G}}

\newcommand{\cP}{\mathcal{P}}

\newcommand{\cS}{\mathcal{S}}

\newcommand{\rR}{\mathbb{R}}

\newcommand{\Pperpphi}{P^\perp_{\Phi}}

\newcommand{\Rmax}{R_{\textrm{max}}}
\newcommand{\Vmax}{V_{\textrm{max}}}
\newcommand{\rpi}{r_\pi}
\newcommand{\Ppi}{P_\pi}

\DeclareMathOperator*{\argmin}{arg\,min}
\DeclareMathOperator*{\expect}{{\huge \mathbb{E}}}

\newcommand{\mdp}{\mathcal{M}}
\newcommand{\sspace}{\mathcal{S}}
\newcommand{\rew}{\mathcal{R}}
\newcommand{\prob}{\mathcal{P}}

\newcommand{\aspace}{\mathcal{A}}

\newcommand{\cbar}{\, | \,}

\begin{document}
\runningauthor{Charline Le Lan, Stephen Tu, Adam Oberman, Rishabh Agarwal, Marc Bellemare}
\twocolumn[

\aistatstitle{On the Generalization of Representations in Reinforcement Learning}

\aistatsauthor{ Charline Le Lan \And Stephen Tu \And  Adam Oberman}

\aistatsaddress{ University of Oxford \And  Google Brain \And McGill University}

\aistatsauthor{Rishabh Agarwal \And Marc Bellemare}
\aistatsaddress{Google Brain \And Google Brain} ]

\begin{abstract}
In reinforcement learning, state representations are used to tractably deal with large problem spaces. State representations serve both to approximate the value function with few parameters, but also to generalize to newly encountered states. Their features may be learned implicitly (as part of a neural network) or explicitly (for example, the successor representation of \citet{dayan1993improving}). While the approximation properties of representations are reasonably well-understood, a precise characterization of how and when these representations generalize is lacking. In this work, we address this gap and provide an informative bound on the generalization error arising from a specific state representation. This bound is based on the notion of effective dimension which measures the degree to which knowing the value at one state informs the value at other states.
Our bound applies to any state representation and quantifies the natural tension between representations that generalize well and those that approximate well. We complement our theoretical results with an empirical survey of classic representation learning methods from the literature and results on the Arcade Learning Environment, and find that the generalization behaviour of learned representations is well-explained by their effective dimension.
\end{abstract}

\section{INTRODUCTION}
\label{sec:intro}
At the heart of reinforcement learning~(RL) is the problem of predicting the expected return that can be obtained from different states. In most practical situations, these predictions are made on the basis of parametric function approximation, needed in order to make accurate predictions on the basis of limited samples -- technically speaking, to estimate the \emph{value function} \citep{sutton18reinforcement}. 
Linear function approximation, for example, estimates the value function using a fixed state representation $\phi$ which maps states to vectors in $\rR^k$; general-purpose algorithms for constructing state representations include tile coding \citep{sutton1996generalization}, the Fourier basis \citep{konidaris11value}, local basis functions \citep{ratitch04sparse},  and methods based on properties of the transition function \citep{mahadevan2007proto,ghosh2020representations}. Common deep RL network architectures such as DQN \citep{mnih15human} use multiple layers of nonlinear transformations to map perceptual inputs to a final layer which is linearly transformed into a value function prediction (\cref{fig:highlevel}); accordingly, we may also view this final layer as a (time-varying) state representation $\phi$ \citep{levine17shallow,chung18two}.

\begin{figure}
  \centering
  \includegraphics[width=.45\textwidth]{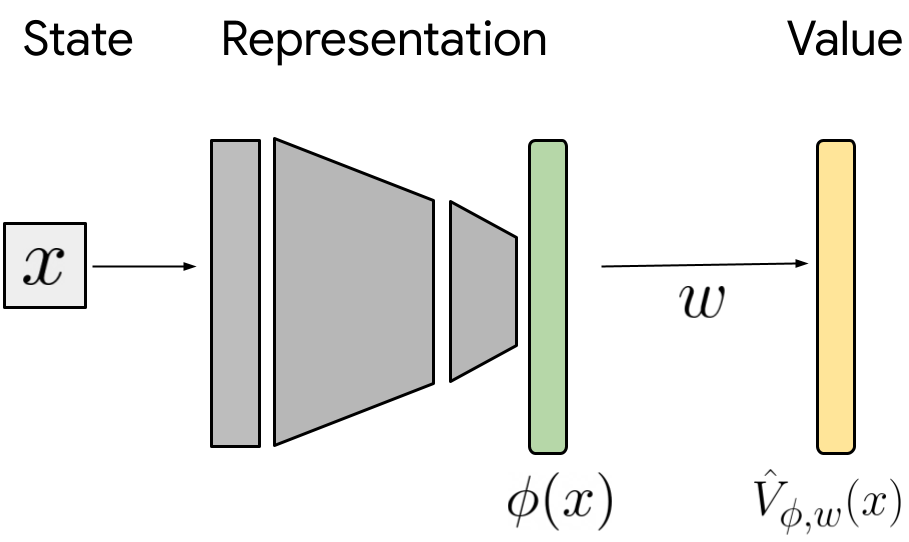}
  \caption{A deep RL architecture seen as a deep representation $\phi$ and a value prediction $\hat{V}_{\phi, w}$.}
  \vspace{-1.em}
  \label{fig:highlevel}
\end{figure}

It is generally believed that auxiliary tasks, known to improve performance in deep reinforcement learning \citep{jaderberg17reinforcement, bellemare17distributional}, play an important role in shaping the learned state representation \citep{bellemare2019geometric,dabney2020value,lyle2021effect}. This motivates the need to understand how representation learning impacts policy evaluation.
In this paper, we give a theoretical characterization of the generalization properties of a given or learned representation. While there are a number of results characterizing the approximation error due to a representation \citep{petrik2007analysis,parr2008analysis}, its effect on statistical error is relatively unknown.

Our first contribution is a bound on the generalization error (approximation + estimation) that arises when performing Monte Carlo value function estimation with a given $k$-dimensional representation $\phi$ (\cref{sec:genbounds}). Critically, this bound depends on the (in)coherence of the feature matrix $\Phi$ \citep{candes2009exact}, which in turns defines the \emph{effective dimension} of the representation. This effective dimension determines how many samples are needed to obtain a good generalization of the value function with the chosen representation; it may be as low as $k$, indicating that generalization is as good as possible, or as high as $|S|$, the number of states, indicating no generalization at all. The bound applies more broadly to the generalization error incurred in least-squares regression problems where a subset of a larger set of points is observed.

\looseness=-1 In \cref{sec:sr-bound}, we demonstrate the usefulness of our bound by specializing it to study the generalization properties of the successor representation (SR) \citep{dayan1993improving}. Specifically, we consider the state representation constructed from the top $k$ singular vectors of the SR \citep{stachenfeld14design,machado17laplacian,behzadian2018low}. Empirically, we find that the effective dimension of this representation -- and consequently its generalization characteristics -- can vary substantially according to the transition structure of the environment. We also show empirically that the effective dimension is important to determine the generalization capacity of different theoretically-motivated representations in the four room domain \citep{sutton99between}.

\looseness=-1 In an empirical study on the Arcade Learning Environment \citep{bellemare2013arcade}, we find that the notions of incoherence and effective dimension correlate with the observed empirical performance of existing value-based deep RL agents~(\cref{sec:deep-rl}). Furthermore, we find that a simple auxiliary loss motivated by our bound shows promising gains in the offline deep RL setting.

\section{BACKGROUND}
\label{sec:background}
We consider a Markov Decision Process (MDP) $\mdp = \langle\sspace, \aspace, \rew, \prob, \gamma\rangle$ \citep{puterman1994markov} with finite state space $\sspace$, discrete set of actions $\aspace$, transition kernel $\prob : \sspace \times \aspace \to \mathscr{P}(\sspace)$, deterministic reward function $\rew : \sspace \times \aspace \to [-\Rmax, \Rmax]$, and discount factor $\gamma \in [0, 1)$. For simplicity, we make the correspondence $\cS = \{ 1, ..., S \}$. We write $\prob_s^a$ to denote the next-state distribution over $\sspace$ resulting from selecting action $a$ in $s$ and write $\rew_s^a$ for the corresponding reward.

A stationary policy $\pi:\sspace \rightarrow \mathscr{P}(\aspace)$ is a mapping from states to distributions over actions, describing a particular way of interacting with the environment. We denote the set of all policies by $\Pi$. For any policy $\pi \in \Pi$, the value function $V^\pi(s)$ measures the expected discounted sum of rewards received when starting from state $s\in\sspace$ and acting according to $\pi$:
\begin{equation*}
	V^\pi(s) := \expect_{\pi, \prob} \left[ \sum_{t = 0}^\infty \gamma^t \rew_{s_t}^{a_t} \cbar s_0 = s, a_t \sim \pi(\cdot \cbar s_t) \right] .
\end{equation*}
The upper-bound value is $\Vmax := \frac{\Rmax}{1-\gamma}$. In vector notation \citep{puterman1994markov}, let $\rpi \in \rR^S$ denote the vector of expected rewards, and let $\Ppi \in \rR^{S \times S}$ be the transition matrix whose entries are
\begin{equation*}
    \Ppi(s, s') = \sum_{s' \in \sspace} \prob^a_s(s') \pi(a \cbar s) .
\end{equation*}
We then have
\begin{equation*}
    V^\pi = \sum_{t=0}^\infty (\gamma \Ppi)^t \rpi = (I - \gamma \Ppi)^{-1} \rpi.
\end{equation*}
In this paper we consider approximating the value function $V^\pi$ using a linear combination of features. We call the map $\phi: \sspace \rightarrow \mathbb{R}^k$ a \emph{$k$-dimensional state representation}; $\phi(s)$ is the feature vector for a state $s \in \sspace$.
In general, we will be interested in the setting where $k \ll S$. The value function approximation at $s$ is
\begin{equation*}
    \vfapprox (s) = \phi(s)^\top w,
\end{equation*}
where $w \in \rR^k$ is a weight vector. We collect the per-state feature vectors into a feature matrix $\Phi \in \rR^{S \times k}$. For simplicity, we assume $\Phi$ has full column rank.
In vector form, the value function approximation (a $S$-dimensional vector) is more directly expressed as
\begin{equation*}
    \vfapprox = \Phi w .
\end{equation*}

\subsection{Statistical Learning Theory}

We consider the \emph{batch Monte Carlo policy evaluation} setting, in which we are given a sample of training examples $D = \{ (s_1, y_1), \dots, (s_n, y_n) \} \in (\sspace \times \rR)^n$ and wish to determine a good linear approximation to $V^\pi$ on the basis of this sample. 
Here, $s_i$ is a state and $y_i$ is a realisation of the random return $G^\pi(s_i)$ \citep{bellemare17distributional,sutton18reinforcement}, defined by the random-variable equation
\begin{equation*}
    G^\pi(s) = \sum_{t=0}^\infty \gamma^t \rew_{s_t}^{a_t}, \quad s_0 = s, a_t \sim \pi(\cdot \cbar s_t) .
\end{equation*}
We assume that $s_i$ is drawn uniformly at random from $\sspace$.\footnote{Results for a larger class of distributions are given in \cref{app:basicboundmoregeneral}} The batch Monte Carlo setting obviates some of the technical challenges in analyzing iterative methods such as least-squares TD~(LSTD) but still allows us to provide practically-relevant theoretical guarantees.

We measure the quality of a linear approximation $\vfapprox$ in terms of the expected squared error
\begin{equation}
    R(\vfapprox) = \frac{1}{S} \sum_{s \in \sspace} \expect_{y \sim G^\pi(s)} \big (\vfapprox(s) - y \big)^2 . \label{eq:squarederror}
\end{equation}
For a value function $V$, we express this error and related quantities in terms of the uniformly-weighted $L^2$ norm
\begin{equation*}
    \unorm{V} = \sqrt{\frac{1}{S} \sum_{s \in \sspace} \big (V(s) \big)^2} .
\end{equation*}
Following terminology from statistical learning theory \citep{vapnik1995nature}, we call $R(\vfapprox)$ the \emph{population risk} of $\vfapprox$. 
One can verify that $R(\vfapprox)$ is minimized when $\vfapprox = V^\pi$.

Given the dataset $D$ and a fixed state representation $\phi$, least-squares regression determines the weight vector $\hat w$ minimizing the \emph{empirical risk function}
\begin{equation*}
    \hat{R}(\vfapprox) = \frac{1}{n}\sum_{i=1}^n(\vfapprox(s_i) - y_i)^2.
\end{equation*}
Notice that $\hat{R}$ is a random function as it depends on the training sample $D$.

We are interested in the performance of the least-squares approximation $\vfhatapprox$ compared to the true value function $V^\pi$. Let us denote by $\vfoptapprox$ the linear approximation minimizing the population risk, such that
\begin{equation*}
    w^* = \argmin_{w \in \rR^k} R(\vfapprox).
\end{equation*}
For clarity of exposition, we will assume this approximation is unique.
The \emph{excess risk} $\cE(\vfhatapprox) = R(\vfhatapprox) - R(V^\pi)$ measures the additional error suffered by the approximation $\vfhatapprox$ compared to the true value function. We decompose it into an estimation error term, measuring the performance gap with the best-in-class, and an approximation error term arising from considering a restricted set of $k$-dimensional value function approximations:
\begin{align*}
 \cE(\vfhatapprox) = \underbrace{R(\vfhatapprox) - R(\vfoptapprox)}_{\text{estimation error}}
 +\underbrace{R(\vfoptapprox) - R(V^\pi)}_{\text{approximation error}} .
\end{align*}
\subsection{The Successor Representation}

The successor representation \citep{dayan1993improving} describes a state in terms of the frequency at which it visits future states; it is also related to the fundamental matrix in the study of Markov chains see \citet{kemeny1961finite, bremaud2013markov, grinstead2012introduction}.
\begin{definition}
The successor representation (SR) 
with respect to a policy $\pi$ for a state $s \in \sspace$ is the expected discounted sum of future occupancies for each state $s' \in \sspace$.
Specifically, $\psi^\pi(s) = ( \psi^\pi(s, s') )_{s' \in \sspace}$, where 
\begin{equation*}
    \psi^{\pi}(s, s') = \expect_{\pi, \prob}\left[\sum_{t=0}^{\infty} \gamma^{t} \mathbb{I}\left[s_{t}=s'\right] \mid s_{0}=s\right].
\end{equation*}
Expressed as a matrix $\Psi^\pi \in \rR^{S \times S}$, the successor representation can be written as:
\begin{equation*}
    \Psi^{\pi} = \left(I-\gamma P_{\pi}\right)^{-1}.
\end{equation*}
\end{definition}
As a consequence of the Bellman equation, 
we can express the value function in terms of the
successor representation as follows:
\begin{equation*}
    V^\pi = \Psi^\pi \rpi .
\end{equation*}
This makes it a particularly appealing candidate to use as a state representation. In particular, it is well-established that the top eigenvectors \citep{mahadevan2007proto} or singular vectors \citep{behzadian2018low} of the successor representation form a useful representation \citep{stachenfeld14design}. \citet{petrik2007analysis} derived an analytical bound on the approximation error for linear value function approximation for a representation made of the top eigenvectors of $\Psi^\pi$ in the particular setting where $P_{\pi}$ is symmetric. By contrast, in this paper, we consider the more general setting of an arbitrary transition matrix $P_{\pi}$ and consider a generalization bound that accounts for the statistical nature of the learning process.

\section{CHARACTERIZING EXCESS RISK}
\label{sec:genbounds}

Our first result characterizes how the choice of representation affects the generalization of value functions. Theorem \ref{thm:main_gen_error} applies beyond the setting of reinforcement learning, and more generally characterizes the excess risk of a broad class of least-squares regression problems. 

To begin, we assume that the labels $y_1, \dots, y_n$ satisfy
\begin{equation*}
    y_i = V(s_i) + \eta_i,
\end{equation*}
where $V : \sspace \to \rR$ and $\eta_i$ is i.i.d.\ zero mean $\sigma$-sub-Gaussian noise \citep{vershynin2010introduction}. This includes the batch Monte Carlo setting, in which case $V = V^\pi$ and $\eta_i \overset{D}{=} G^\pi(s_i) - V^\pi(s_i)$, where $G^\pi(s_i)$ is the random return from $s_i$.

For a feature matrix $\Phi$, we write $P_{\Phi}$ for the orthogonal projector onto its column space, and $\Pperpphi$ for the orthogonal projector onto the corresponding nullspace. We have
\begin{equation*}
    P_\Phi = \Phi (\Phi^\T \Phi)^{-1} \Phi^\T \quad \Pperpphi = I_S - P_\Phi .
\end{equation*}
In particular, the approximation error for a given state representation $\phi$ is
\begin{equation*}
    R(\vfoptapprox) - R(V) = \unorm{\Pperpphi V}^2.
\end{equation*}
A key quantity in our analysis is the notion of the \emph{effective dimension} of a state representation, which dictates the number of samples required to achieve a low estimation error.

\begin{definition}[Effective dimension]
Let $\Phi \in \rR^{S \times k}$ be a feature matrix.
The effective dimension of $\Phi$ (vis-a-vis the standard basis $(e_i)$)
is defined as the quantity
\begin{equation*}
    \deff(\Phi) := S\max_{i=1,...,S} \norm{P_{\Phi} e_i}^2_2, 
\end{equation*}
where $P_{\Phi}$ is the orthogonal projector onto the column space of $\Phi$.
\end{definition}
It is simple to check that the effective dimension is only a function of the column
space of $\Phi$ and that $\deff(\Phi)$ satisfies
\begin{equation*}
    \mathrm{rank}(\Phi) \leq \deff(\Phi) \leq S.
\end{equation*}
Our notion of effective dimension is derived from the \emph{coherence} of $\Phi$, defined as
\begin{equation*}
    \mu(\Phi)=\frac{d_{\mathrm{eff}}}{\mathrm{rank}(\Phi)} .
\end{equation*}
The notion of coherence is from \cite{candes2009exact}, who demonstrate that coherence can be used to characterize the feasibility of low-rank matrix recovery. Informally,
$\mu(\Phi)$ (and $\deff(\Phi)$) measure the (lack of) sparsity of the column space of $\Phi$.
At one extreme, if $\Phi \in \rR^{S \times 1}$ is the all-ones vector, then $\deff(\Phi) = {\mathrm{rank}(\Phi)}$,
saturating the lower bound. %
On the other hand, if $\Phi = e_i$ for some $i \in \{1, \dots, S\}$ then $\deff(\Phi) = S$, saturating the upper bound. As we now show, the effective dimension of $\Phi$ can be used to bound the excess risk of least-squares regression applied to the state representation $\phi$.
\begin{restatable}[Excess risk]{theorem}{maingenerror}
\label{thm:main_gen_error}
Fix any $\delta \in (0, 1)$. 
Suppose that $n \geq 8 \deff(\Phi) \log(6k/\delta)$.
With probability at least $1-\delta$,
the empirical risk minimizer $\vfhatapprox$ satisfies:
\begin{align*}
    &\cE(\vfhatapprox)
    \leq \unorm{\Pperpphi V}^2 + 384 c \frac{ \deff(\Phi)}{n}\unorm{\Pperpphi V}^2 \\
    &+ 48 \sigma^2 \frac{2k + 3 c}{n} + \frac{64}{3} \frac{\deff(\Phi)}{n^2} \norm{ P^\perp_{\Phi} V}_\infty^2 c^2,
\end{align*}
where $c = \log(3/\delta)$ and $\norm{\cdot}_{\infty}$ denotes the usual supremum norm.
\end{restatable}
\begin{proof} 
The proof is given in \cref{app:basicboundmoregeneral},
and follows arguments for the analysis of
random design linear least-squares problems \citep{hsu2012random} and matrix concentration inequalities \citep{tropp15introduction}. The result can also be obtained by instantiating Theorem 1 of \citet{hsu2012random} to our setting, at the cost of added complexity.
\end{proof}

In \cref{thm:main_gen_error}, the term $\unorm{\Pperpphi V}^2$ is the approximation error and reflects the error due to using a $k$-dimensional linear approximation. The remainder of the bound corresponds to the estimation error. The theorem demonstrates that the ability of a representation to generalize is quantified not only by the approximation error but also the effective dimension $\deff(\Phi)$.
Not only does $\deff(\Phi)$ appear in the bound, but it also dictates a minimum number of samples needed to obtain a high probability bound: when $\deff(\Phi)$ is small, the bound holds for fewer samples.

In the specific context of batch Monte Carlo policy evaluation,  \cref{thm:main_gen_error} holds as-is with $V = V^\pi$. Additionally, the noise variance $\sigma^2$ can be bounded as
\begin{equation*}
    \sigma^2 \le \frac{\Vmax^2}{4} .
\end{equation*}

The term $\norm{P_{\Phi} e_i}^2_2$ that drives the effective dimension of $\Phi$ differs (for non orthogonal representations $\Phi$) from the quantity $\max_i \| \phi(s_i) \|_2^2$ that appears in Rademacher complexity bounds for regression in the case of a family of linear predictors \citep{mohri2018foundations}
(see also \cite{maillard2009compressed}). Compared to such bounds, \cref{thm:main_gen_error} is also sharper for all representations as it offers a $O(1/n)$ dependency rather than $O(1/\sqrt{n})$. In subsequent sections, we will provide empirical evidence illustrating how the effective dimension plays a critical role in determining the generalization capability of $\phi$.

\subsection{Illustrative Examples}

To understand how the bound is instantiated in particular settings, consider first the scenario in which $\Phi = I_S$ is the tabular representation. 
This corresponds to using the feature vector $e_{i} \in \rR^S$ for the $i$-th state.
In this case, the approximation error is 0 and the estimation error reduces to the classic $\sigma^2 S/n$ rate for least-squares regression:
\begin{align*}
    R(\vfhatapprox) - R(V) \lesssim  \frac{ \sigma^2 (S + \log(1/\delta))}{n}.
\end{align*}
With this choice of features, good generalization requires a number of samples $n$ linear in $S$.

At the other extreme, it is possible to improve the sample complexity to avoid the dependency on $S$. In the ideal case, $\deff(\Phi) = k$. In the next section we will demonstrate that, in environments with a particular transition structure, representations derived from the successor representation achieve this bound. 

To make this argument more concrete,
suppose that we have a family $(\phi_k)_{k=1}^{S}$ of representations (resp. matrices $(\Phi_k)$) whose effective dimension satisfies
$\deff(\Phi_k) \approx k$. Furthermore, assume that the approximation error $\unorm{P^\perp_{\Phi_k} V}^2$
scales as $\psi(k)$, where $\psi(k)$ is a monotonically
decreasing function of $k$.
Fix $\varepsilon > 0$ and define $\bar{k} = \bar{k}(\varepsilon) := \min\{ k : \psi(k) \leq \varepsilon \}$, and let $\bar w$ be the weight vector found by least-squares regression applied with $\phi_{\bar{k}}$.
Observe that as long as $n$ satisfies:
\begin{align*}
    n \gtrsim \max\left\{ \max\Big\{ \frac{\sigma^2}{\varepsilon}, 1 \Big\} \bar{k}(\varepsilon) \log\frac{\bar{k}(\varepsilon)}{\delta}, \sqrt{\bar{k}(\varepsilon) S} \log \frac{1}{\delta} \right\},
\end{align*}
then we have $\cE(V_{\phi_{\bar k}, \bar w}) \leq 4\varepsilon$.
As a particular example,
let $\psi(k) = \rho^k$ for some $\rho \in (0, 1)$.
Then $\bar{k}(\varepsilon) \leq \ceil{ \frac{1}{1-\rho} \log\left(\frac{1}{\varepsilon}\right)}$,
in which case the sample complexity only depends sublinearly on $S$.

\section{GENERALIZATION FOR THE SUCCESSOR REPRESENTATION}
\label{sec:sr-bound}
An effective approach for constructing a family of representations is to take the $k$ singular vectors of the successor representation (SR) whose singular values are the greatest. For a given policy $\pi$, let $\Psi^\pi$ be the successor representation for $\pi$. We write
\begin{equation*}
    \Psi^\pi = F \Sigma B^\top,
\end{equation*}
where $F, B \in \rR^{S \times S}$ are matrices whose columns are orthogonal and have unit norm. Additionally, $\Sigma = \mathrm{diag}(\sigma_1, ..., \sigma_S)$ where $\sigma_i$ are the singular values of $\Psi$ sorted in decreasing order.

For a fixed integer $k$ satisfying $1 \leq k \leq S$,
let us partition $F$ into two matrices, $F_k \in \rR^{S \times k}$ and $F^\perp_k$, which respectively contains the top $k$ and bottom $S - k$ columns of $F$. Correspondingly, we partition $\Sigma$ into $\Sigma_k \in \rR^{k \times k}$ and $\Sigma_k^\perp$ and $B$ into $B_k$ and $B_k^\perp$. With this notation, we obtain the family of state representations (expressed as feature matrices) $\Phi_k = F_k$.

\begin{figure*}[htb]
  \centering
    \includegraphics[width=0.8\textwidth]{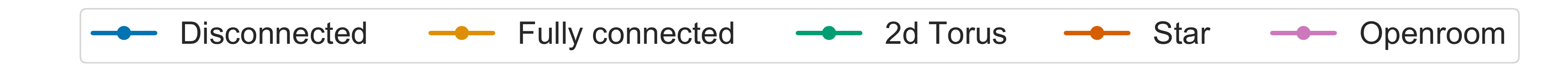}
  \includegraphics[width=0.38\textwidth]{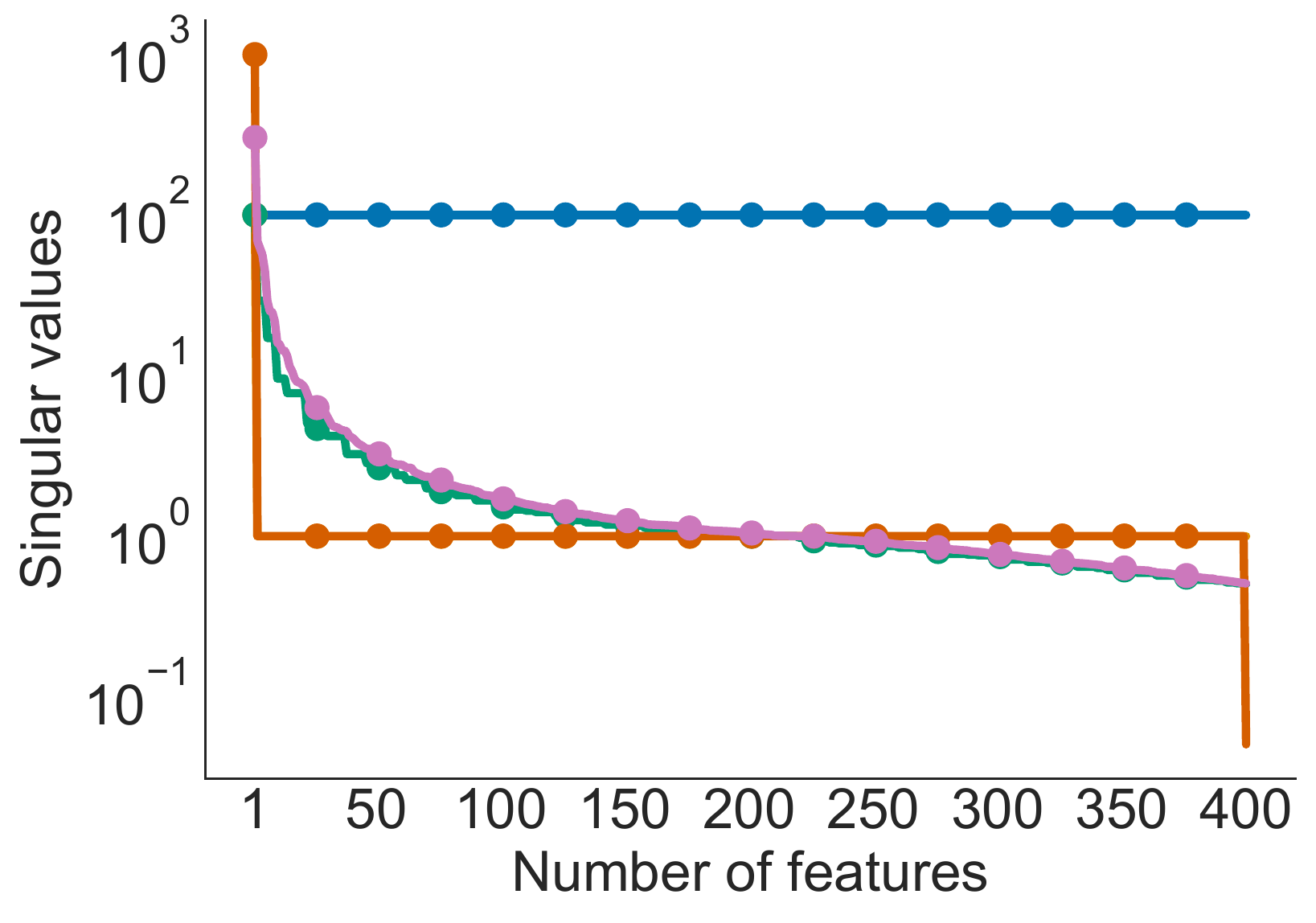}
  \includegraphics[width=0.38\textwidth]{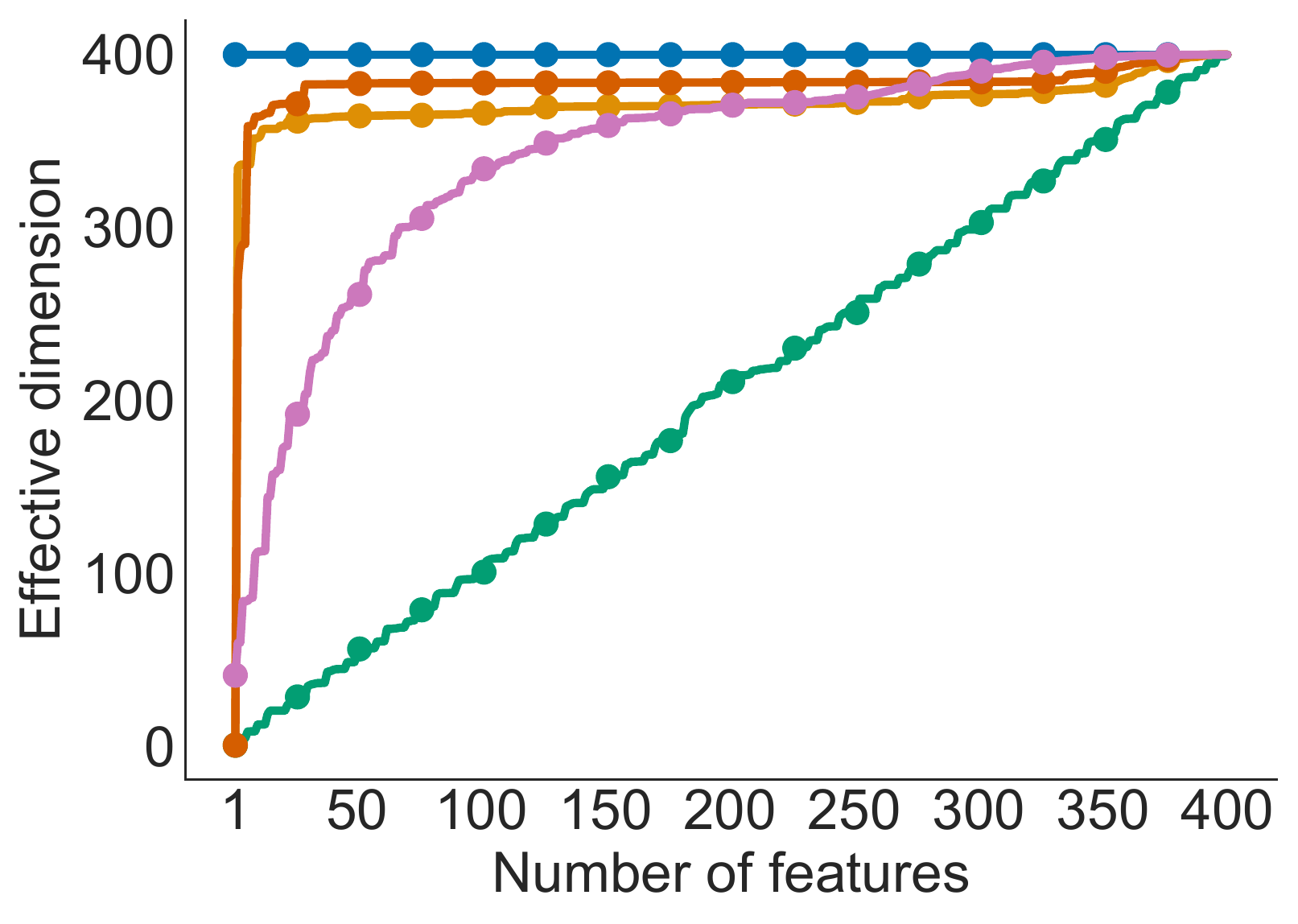}
  \includegraphics[width=0.38\textwidth]{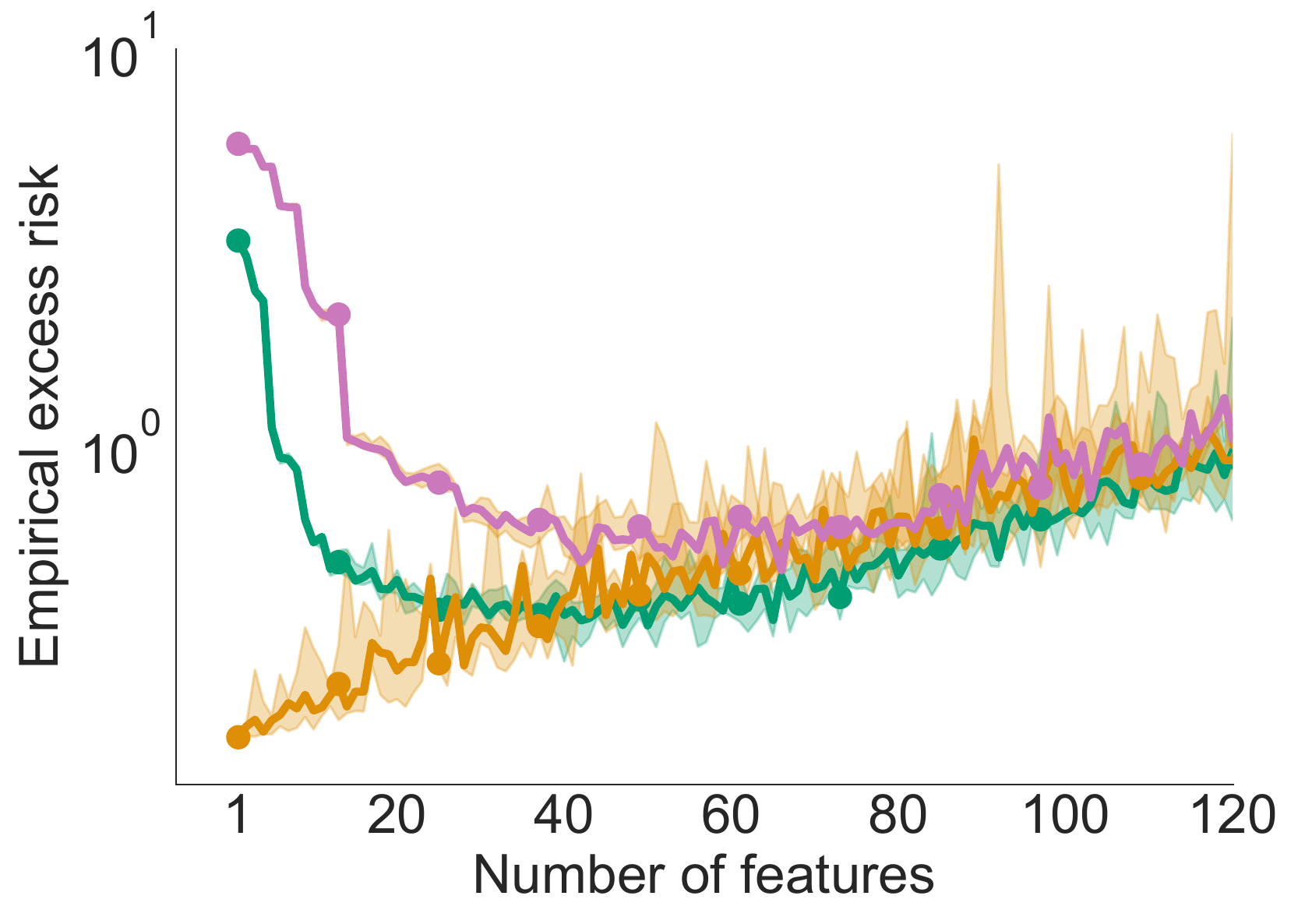}
  \includegraphics[width=0.38\textwidth]{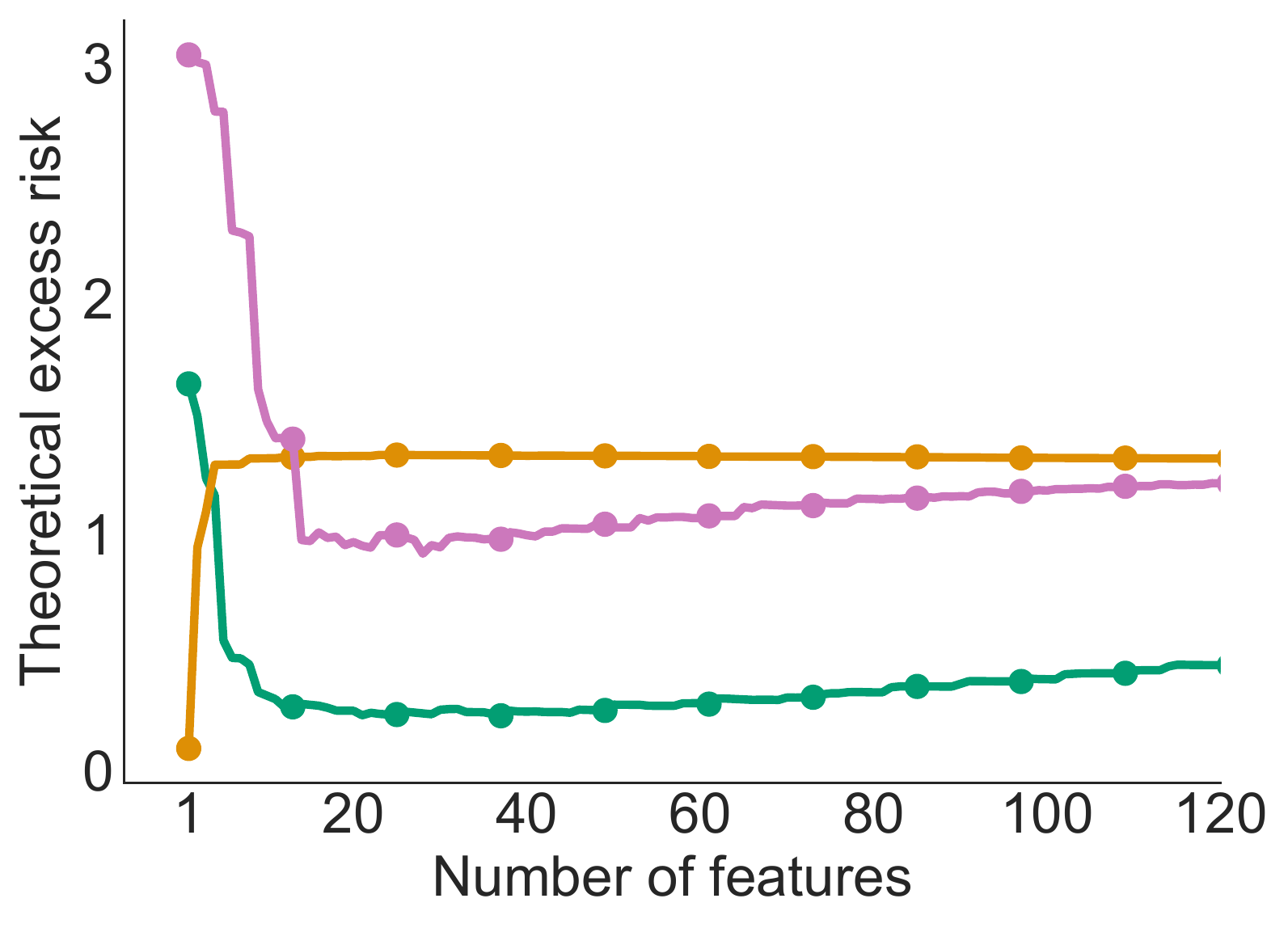}
    \vspace{-1.em}
  \caption{\textbf{Top left}: Singular values of the successor representation $\Psi^\pi$, in decreasing order and for different graphical structures (the fully connected and star graphs' spectra overlap). \textbf{Top right}: Effective dimension of the representation $\Phi_k = F_k$. \textbf{Bottom left and right}: Median empirical excess risk over 10 runs, with 95\% CIs as shaded regions, and theoretical excess risk, respectively, for the open room, torus, and fully connected graphs.}
  \label{fig:effectivedim_toymdps}
\end{figure*}

\subsection{Approximation Error: $\unorm{\Pperpphi V^\pi}^2$}
\label{sec:sr-approxerror}
Given a reward vector $r_{\pi} \in \rR^S$,
the value function $V^{\pi} \in \rR^S$ is given by $V^{\pi} = \Psi^{\pi} r_{\pi}$.
As demonstrated by \cref{thm:main_gen_error},
the first key quantity that appears
in the generalization bound is the approximation error $\unorm{P^\perp_{F_k} V^{\pi}}^2$.
With the successor representation, we can write:
\begin{align*}
    \unorm{P^\perp_{F_k} V^{\pi}}^2 = \unorm{P^\perp_{F_k} \Psi^{\pi} r_{\pi}}^2 = \unorm{F_k^\perp \Sigma_k^\perp (B_k^\perp)^\T r_{\pi}}^2.
\end{align*}
Following the argument from \citet{petrik2007analysis} for the specific case of proto-value functions \citep{mahadevan2007proto}, the worst-case unit-norm reward vector $r_\pi$ in this case approximately corresponds to the $(k+1)$-th vector $b_{k+1}$. This is because
\begin{equation*}
    F_k^\perp \Sigma_k^\perp (B_k^\perp)^\T b_{k+1} = f_{k+1} \sigma_{k+1},
\end{equation*}
and the fact that $\sigma_{k+1} \ge \sigma_{k+i}$, for all $i \ge 1$. To make the bound comparable for different $k$ and MDPs, let us fix $\Rmax$ and write
\begin{equation}\label{eqn:worst-case-reward}
    r_\pi = \frac{b_{k+1} \Rmax}{\norm{b_{k+1}}_\infty} .
\end{equation}
In this case, since $\norm{f_{k+1}}_2^2 = 1$, we have that
\begin{equation*}
    \unorm{P^\perp_{F_k} V^\pi}^2 \le \frac{\sigma_{k+1}^2 \Rmax^2}{S \norm{b_{k+1}}^2_\infty} \le \sigma_{k+1}^2 \Rmax^2 .
\end{equation*}
The dependence on $\norm{b_{k+1}}_\infty$ relates to the operator norm of $\Psi$ from $L^2$ to $L^\infty$, and illustrates how $b_{k+1}$ is only approximately the worst-case reward vector.

A frequent scenario in reinforcement learning occurs when the reward is nonzero in a single state. Suppose that the reward vector $r_{\pi}$ is $r_{\pi} = R_{\max} e_i$ for some $i \in \{ i, \dots, S \}$. Then we have that:
\begin{align*}
    \unorm{P^\perp_{F_k} V^{\pi}}^2 &= \frac{R_{\max}^2 \mathrm{tr}( (\Sigma_k^\perp)^2 ) \norm{(B_k^\perp)^\top e_i}_2^2}{S} \\
    &\leq \frac{\sigma_{k+1}^2 R_{\max}^2  \deff(B_k^\perp)}{S}. 
\end{align*}
\looseness=-1 When the effective dimension of $B_k^\perp$ is $O(S-k)$, the approximation error may be a factor $\tfrac{S-k}{S}$ smaller than the error for the worst-case reward vector (\cref{eqn:worst-case-reward}).

\looseness=-1 These arguments show that the generalization quality of a given family of representations can be partially quantified in terms of its spectrum $(\sigma_i)_{i=1}^S$.
When the transition matrix is symmetric, we can bound the spectrum $(\sigma_i)_{i=1}^S$ in terms of the effective horizon implied by the discount factor. This is given by the following lemma.
\begin{restatable}[]{lemma}{singularvalues}
\label{lemma:singularvalues}
Let $P \in \rR^{|\sspace| \times |\sspace|}$ be a symmetric row stochastic matrix, and let $\gamma \in (0, 1)$.
Let $\sigma(\cdot)$ denote the set of singular values of a matrix.
We have that:
\begin{align*}
    \sigma((I - \gamma P)^{-1}) \subseteq \Big [\tfrac{1}{1+\gamma}, \tfrac{1}{1-\gamma} \Big ].
\end{align*}
\end{restatable}
Because the value function is generally of magnitude $\Vmax = \tfrac{\Rmax}{1-\gamma}$, an approximation error of order $\tfrac{1}{1+\gamma}$ is quite small, suggesting that the corresponding basis functions may be safely omitted from the representation.

Intuitively (and as supported by the analysis above), choosing a representation with a larger number of features $k$ reduces the approximation error. However, as will see in the next section, a larger $k$ necessarily increases the effective dimension, often in a manner that is superlinear in $k$.

\begin{figure*}[t]
    \centering
 \includegraphics[width=0.24\textwidth]{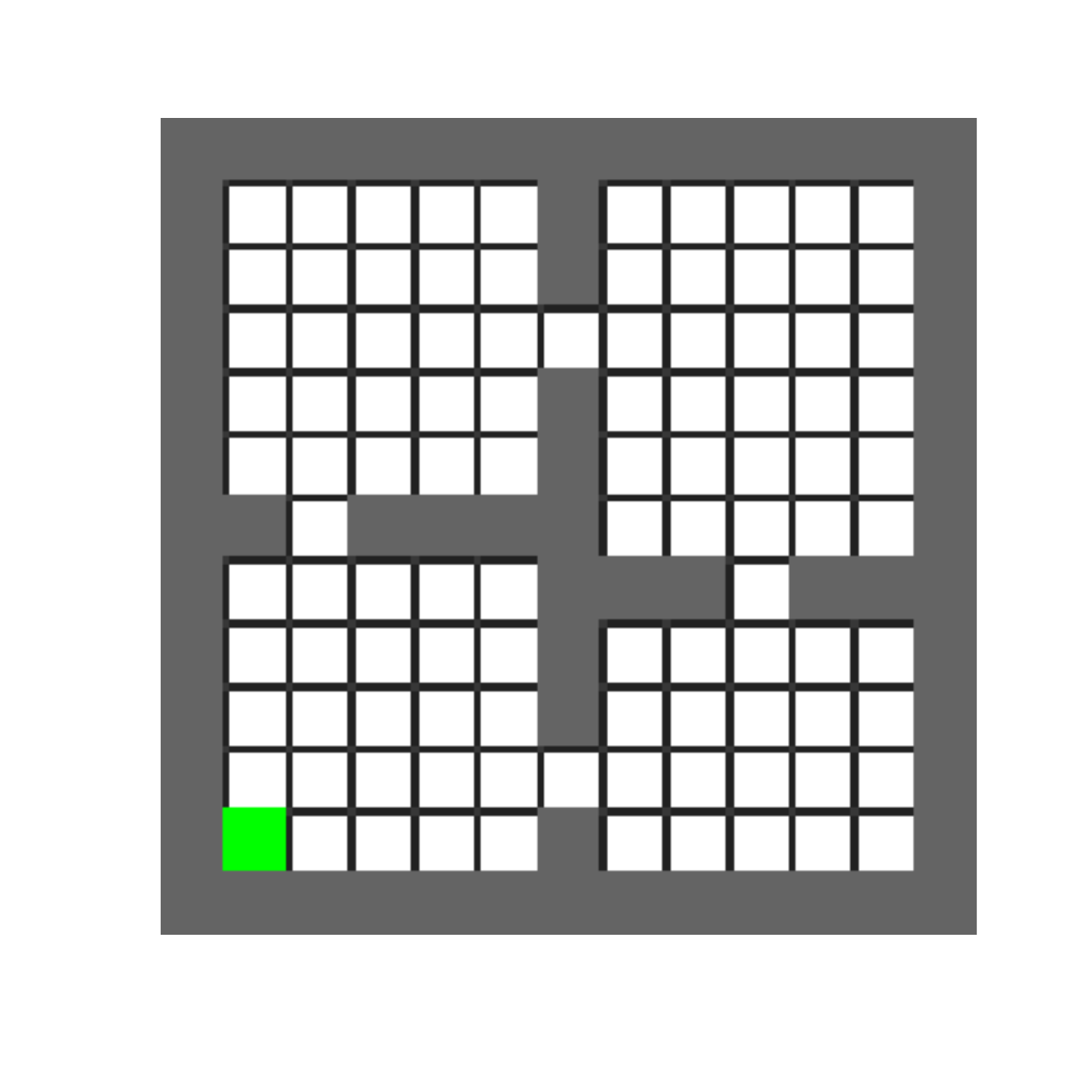}
    \includegraphics[width=0.75\textwidth]{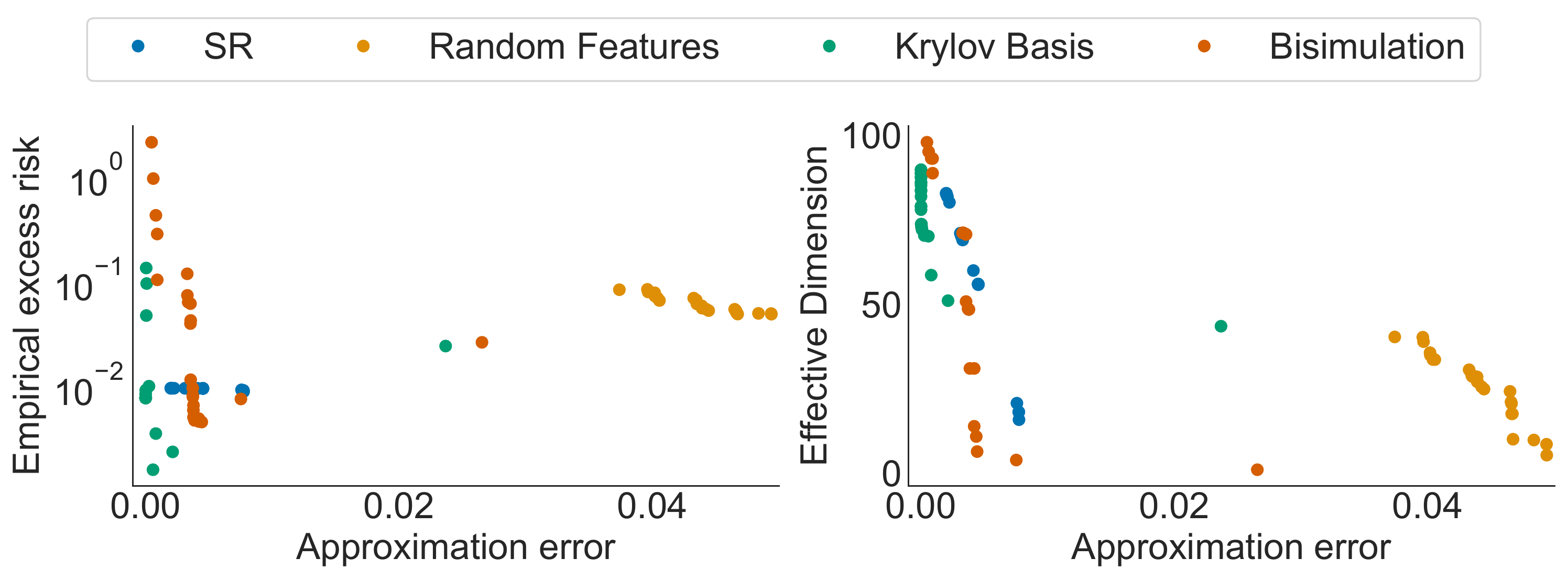}
    \caption{\looseness=-1 The four-room domain (\textbf{Left}). Median empirical excess risk (\textbf{Middle}) and effective dimension (\textbf{Right}) as a function of approximation error for the top $k$ left singular vectors of the SR, random features, the Krylov basis and the bisimulation metric matrix in the four-room domain.}
    \label{fig:representations:comparison}
\end{figure*}

\subsection{Effect of Transition Structure}
\label{sec:effect-structure}
We next study characteristics of families of representations induced by the SVD of the successor representation for different environment transition structures. To this end, we consider different types of graphs over which we define a uniform random walk; the resulting representations are specifically proto-value functions \citep[PVF,][]{mahadevan2007proto}. We consider the two key quantities identified above: the spectrum of the representation, which informs us on the profile of the approximation error $\unorm{P_{F_k}^\top V^\pi}^2$ for different $F_k$, and the effective dimension of $F_k$ as a function of $k$.

\looseness=-1 We consider five graphical structures, each with $S = 400$ states (illustrations of these structures as well as results for additional structures are given in the appendix): a fully-connected graph, Baird's star graph \citep{baird1995residual}, a disconnected graph (on which each node self-transitions), a $20 \times 20$ grid, and a $20 \times 20$ torus. The torus has the same ``shape'' as the grid but allows transitions from one edge to its opposite, while the fully-connected graph is similar to the star graph in that both mix quickly. These graph were chosen to illustrate the diversity in generalization profiles arising from different transition structures. In all cases, $\gamma = 0.99$.

\cref{fig:effectivedim_toymdps}, top left illustrates three types of spectra. The fully-connected and star structures have a flat spectrum, both with an important first component but with a last component that is much smaller in the case of the star structure (see \cref{app:proof-srbound} for a closed-form description of the spectrum of the star graph).
By contrast, the grid and torus exhibit a decaying spectrum, suggesting that attaining a low approximation error may require many features. As expected, the disconnected graph produces a flat spectrum with values $\sigma_i = (1-\gamma)^{-1}$.

\cref{fig:effectivedim_toymdps}, top right shows the effective dimension as a function of the number of features $k$, and paints a relatively different picture. Here, both star and fully-connected graphs exhibit a high effective dimension, despite having relatively simple structure. This is because effective dimension reflects in some sense the degree to which a single sample might give misleading information about the value at other states. Because the first singular vectors capture most of the symmetry in these graphs, additional features must in some sense be misleading. On the other hand, the open room and torus, despite an almost-identical spectrum, exhibit notedly different profiles: while the torus achieves the lower bound $\deff(F_k) \approx k$, the grid results in generally poor features for $k$ large.

To understand the consequences of these characteristic differences, we performed least-squares regression to estimate value functions in three of these structures (fully-connected, grid, and torus). In all cases, we sampled a reward function by assigning rewards to each state-action pair from a normal distribution (see \cref{appendix:experiments}). We then sampled $n = 300$ states with replacement and performed a Monte Carlo rollout to obtain the sample return $(y_i)_{i=1}^n$. We measured the excess risk of the linear approximation found by the least-squares procedure. For each graph structure, we repeated the experiment 10 times.

\cref{fig:effectivedim_toymdps}, bottom depicts the outcome of this experiment.
Experimentally, the PVF of the torus generalizes significantly better than the PVF of the grid (left panel). This is reflected in a heuristic calculation of the theoretical bound (right panel), given more explicitly by the formula
\begin{equation*}
    \unorm{P^\perp_{F_k} V^{\pi}}^2 + \frac{\deff(F_k)}{n} + \frac{\deff(F)}{n^2} \norm{ P^\perp_{F_k} V^\pi}_\infty^2 .
\end{equation*}
The number of features $k$ minimizing the empirical and theoretical excess risk differ, but follow the same qualitative pattern: for small $k$, the open room PVF generalizes poorly, while the minimum is achieved in the fully-connected graph by $k = 1$, highlighting again its high degree of symmetry.

\subsection{Analysis of the One-dimensional Torus}

As evidenced by the experiments of the previous section, the proto-value functions of the two-dimensional torus have particularly appealing generalization characteristics. Analytically, similarly good generalization can be demonstrated on the one-dimensional torus, as we now show.

The one-dimensional torus consists in $S$ states arranged on a chain, such that $s_i$ connects to $s_{i-1}, s_{i+1} \mod S$. As such, the random walk on this torus induces a transition function $P_\pi$ described by a circulant matrix. Since $P_\pi$ is symmetric, we may write\footnote{We ignore the issue of real diagonalizable versus complex diagonalizable.}
\begin{equation*}
    (I-\gamma P_\pi)^{-1} = U_S \Sigma U_S^* .
\end{equation*}
Following \cite{gray06toeplitz}, the $k$-th singular value of $(I - \gamma P_\pi)^{-1}$ is given by
\begin{equation*}
    \sigma_k = \frac{1}{1-\gamma \cos(\frac{2\pi}{S} \ceil{\frac{k-1}{2}})}
\end{equation*}
for $k = 1, ..., S$.\footnote{The spectrum of the torus is briefly mentioned in \citet{blier2021learning}.}
Additionally, we have that $U_S = \frac{1}{\sqrt{S}} F_S^*$, with $(F_S)_{j,k} = \exp(-2\pi i jk/S)$ the discrete Fourier transform matrix in dimension $S$. From this we deduce that each entry of $U_S$ has modulus $1/\sqrt{S}$, and therefore any orthogonal matrix formed from any $k$ distinct columns of $U_S$ will have coherence $1$ and effective dimension $k$. This shows that the proto-value functions of the one-dimensional torus give in some sense an ideal state representation.

\section{EXPERIMENTS}
\label{sec:experiments}
\begin{figure*}[t]
  \centering
  \includegraphics[width=0.96\textwidth]{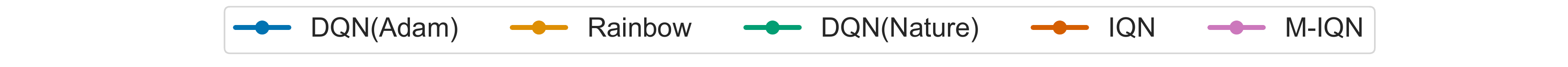}
     \includegraphics[width=0.4\textwidth]{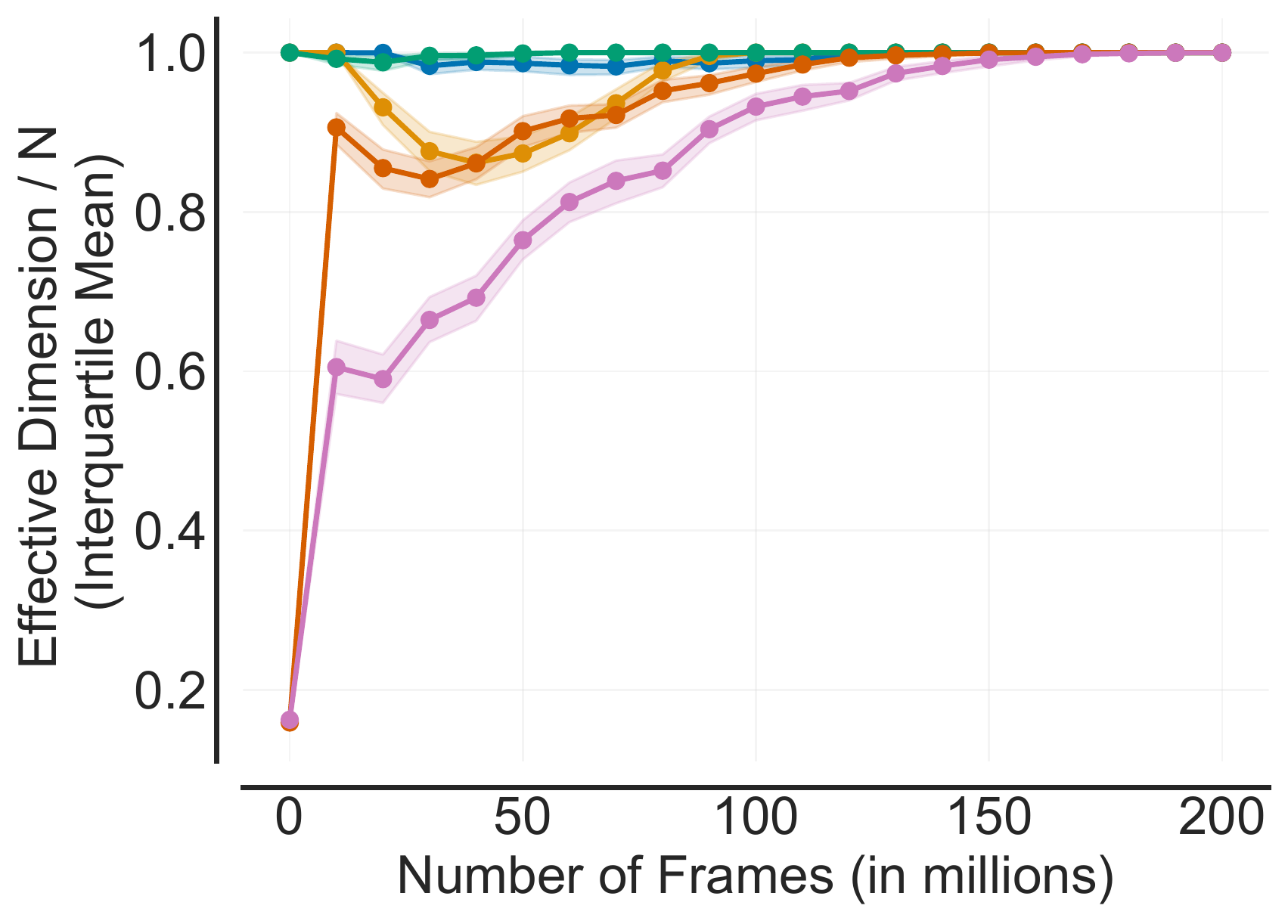}~~~~
       \includegraphics[width=0.4\textwidth]{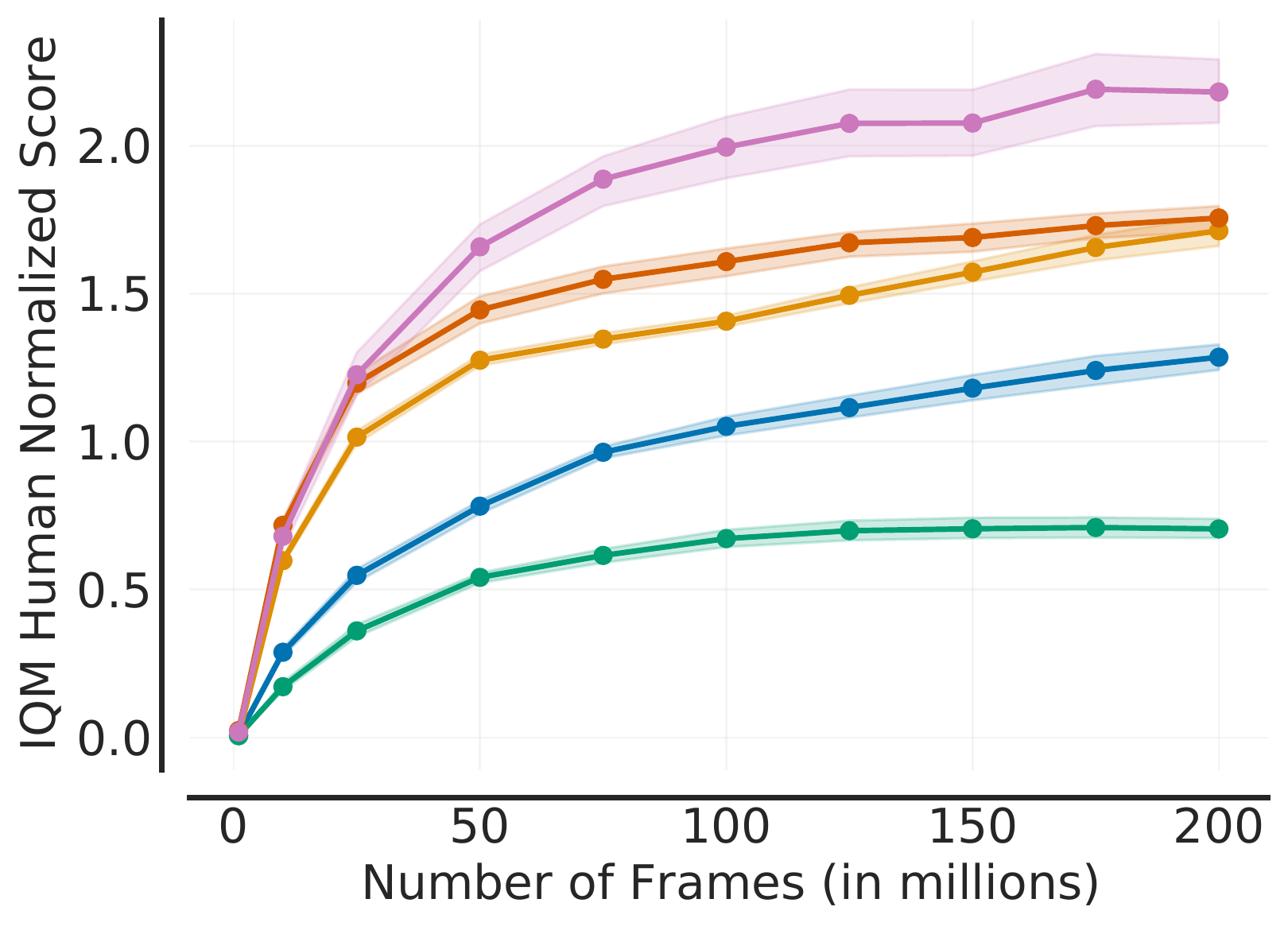}
  \caption{ \looseness=-1 \textbf{Left:} Interquartile mean (IQM)~\citep{agarwal2021deep} for the effective dimension, normalized by the batch size used $N=2^{15}$.  \textbf{Right:} for human-normalized scores over the course of training across 60 Atari games. IQM measures the mean on the middle 50\% of the data points combined across all runs and games. These statistics are over 5 independent runs and shading gives 95\% stratified bootstrap confidence intervals based on Rliable~\citep{agarwal2021deep}.}\label{fig:aggregate_atari}
\end{figure*}

\subsection{Comparing State Representations}
\label{sec:representations}

We now compare the Successor Representation to other theoretically-motivated representations: the bisimulation metric matrix (Ferns et al., 2004), the Krylov basis (Petrik, 2007) and some random features, in terms of effective dimension and excess risk, in the setting of Section 4.2. \cref{fig:representations:comparison} shows some of these results on the four room domain \citep{sutton99between,solway14optimal}. These give further weight to the idea that effective dimension plays an important role in determining the usefulness of a representation, as for a given approximation error better effective dimension corresponds to better excess risk.
 
\looseness=-1 The SR of the four-room domain is fairly well-studied and have been shown to give rise to effective representations \citep{machado17laplacian,bellemare2019geometric}.It generalizes well but has worse approximation error compared to the Krylov basis or the Bisimulation metric which take into account the reward. For small approximation errors, the krylov basis has smaller effective dimension and is performing best. Finally, random features which are agnostic to the structure of the MDP have very high approximation error making them unappealing.

\begin{figure*}[t!]
  \centering
  \includegraphics[width=0.96\textwidth]{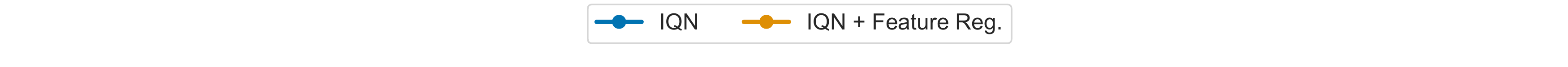}
        \includegraphics[width=0.39\textwidth]{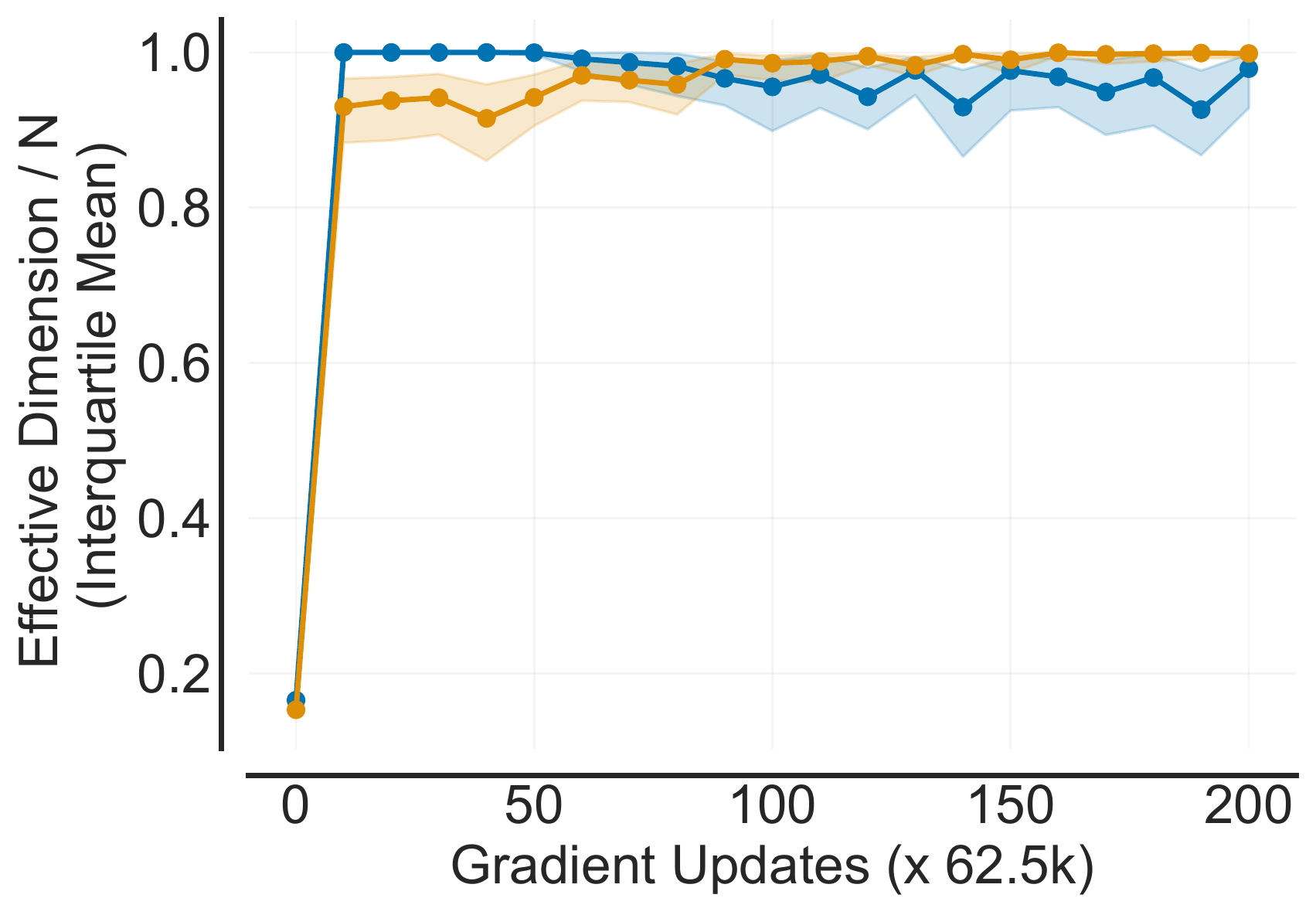}
     \includegraphics[width=0.39\textwidth]{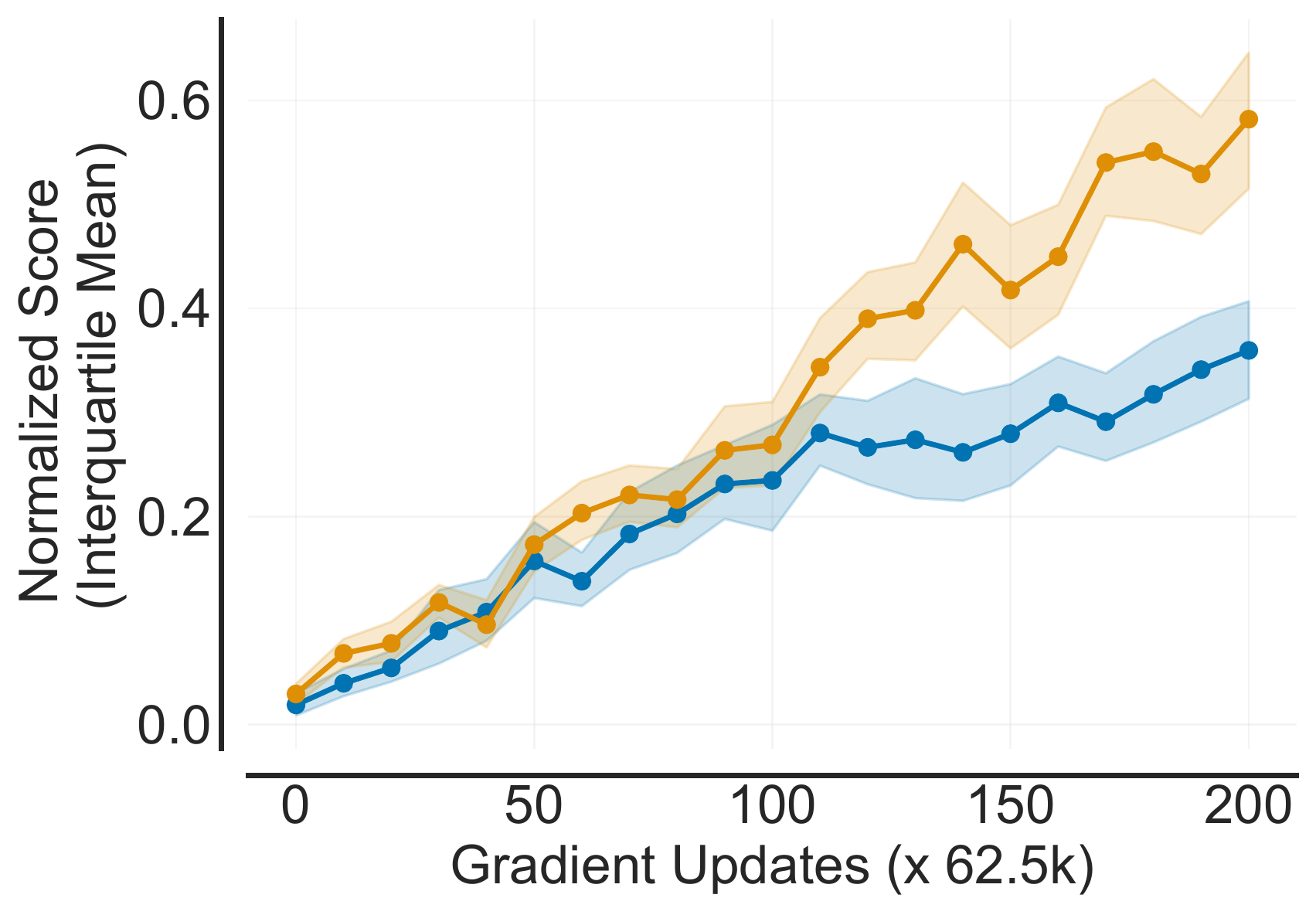}
  \caption{Effective dimension, normalized by the batch size $N=2^{15}$ and performance of IQN and IQN with feature regularization $L_\phi$ on 17 Atari games in the offline RL setting.}
  \label{fig:atari_offline}
\end{figure*}
\subsection{Deep Reinforcement Learning}
\label{sec:deep-rl}
We conclude with an empirical evaluation demonstrating the usefulness of our results in characterizing generalization in a larger setting.
Specifically, we measure the effective dimensions of a representation $\phi$ implied by a deep neural network. We consider the hidden layer of 512 rectified linear units learnt by five deep RL agents, namely DQN~\citep{mnih15human}, DQN with Adam optimizer, Rainbow~\citep{hessel18rainbow}, IQN~\citep{dabney2018implicit}, and Munchausen-IQN (M-IQN)~\citep{vieillard2020munchausen}.
We are interested in how the notion of effective dimension explains the relative performance of these deep RL agents aggregated across 60 Atari 2600 games~\citep{bellemare2013arcade} and at different points in training until 200M environment frames~\citep{castro18dopamine}.

We compare estimates of the effective dimension of these representations throughout training and reported results in \cref{fig:aggregate_atari} (Left) (see per game comparison in \cref{app:fullatariresults}). When computing such estimates, we use a large batch size (=$2^{15}$), sampled uniformly from the offline Atari-replay datasets~\citep{agarwal2020optimistic}, as a proxy for the ambient dimension $S$ used in the definition of the effective dimension.  

\looseness=-1 We observe that higher performance on a game typically correlates with lower effective dimension.
The relative ordering of effective dimension~(\cref{fig:aggregate_atari}, left) matches the performance ranking of different agents(\cref{fig:aggregate_atari}, right). Furthermore, we can notice a rise in the effective dimension from iteration 50 which suggests an overfitting of the representation to the current value function, in line with the evidence of late-training overfitting found by \citet{dabney2020value}.

To further corroborate that low effective dimension corresponds to better generalization, we investigate whether optimizing an auxiliary loss $\mathcal{L}_\phi$, motivated by the idea of reducing the effective dimension of the learned representation, improves performance. To do so, we use $\mathcal{L}_\phi = \log \sum_i \exp(\norm{\phi(s_i)}_2^2)$ for states $s_i$ in a randomly sampled mini-batch of size 32. To avoid confounding effects from exploration, we study the offline RL setting~\citep{levine2020offline}.
Specifically, we use the 5\% Atari-replay dataset~\citep{agarwal2020optimistic} on 17 games and evaluate IQN, one of the top performing agents on the offline Atari dataset~\citep{gulcehre2020rl}. As shown in \cref{fig:atari_offline}, right, combining IQN with the loss $\mathcal{L}_\phi$ results in significantly higher average returns compared to IQN on all 17 games. We also compare estimates of the effective dimension of the representations induced by these two agents in \cref{fig:atari_offline}, left, and find the auxiliary loss $\mathcal{L}_\phi$ results in lower effective dimension during the first 80 iterations. Surprisingly, we also notice that IQN with feature regularization prevents the substantial loss in rank of the feature matrix observed previously by \citet{kumar2021implicit, kumar2021dr3} (see \cref{fig:atari_offline_IQM_rank} and \cref{fig:atari_offline_pergame_rank}), making it hard to disentangle between approximation and estimation error effects.
Further study
of this phenomenon would be an interesting direction for future work.

\section{CONCLUSION}\label{sec:discussion}

In this paper we provided a theoretical characterisation of how a given representation affects generalization in reinforcement learning. While we focused here on the batch Monte Carlo setting for simplicity, a similar but more involved analysis can in theory also be performed to analyze algorithms such as LSTD.

Providing fresh evidence regarding the benefits of successor representations in shaping an agent's representation, both our analysis and experiments on synthetic environments demonstrate that indeed, the left-singular vectors of SRs generally provide good generalization. While natural given the successor representation's close relationship with the value function, one surprising result is that the effective dimension of such a representation is relatively sensitive to the particular transition structure, as illustrated by the differences between the torus and open room representations. In addition, the effective dimension of this representation does not immediately correlate with mixing time, as one might have expected. These findings suggests that it should be possible to devise algorithms inspired by the same principles, but that work well across a variety of transition structures, for example by leveraging contrastive graph representations \citep{madjiheurem19representation}.

Our analysis of Atari 2600-playing agents gives further evidence of the important role played by the representation in deep reinforcement learning. While not a surprise in itself, we find a strong correlation between effective dimension and performance, this suggests that generalization is key to explaining many performance improvements. In particular, it is by now well-understood that auxiliary tasks \citep{jaderberg17reinforcement,bellemare17distributional} shape the learned representation of the agent, and under ideal conditions cause it to match the SVD of an auxiliary task matrix \citep{bellemare2019geometric,lyle2021effect}. Controlling the bound of \cref{thm:main_gen_error} by means of such tasks or deep learning mechanisms such as hindsight experience replay \citep{andrychowicz2017hindsight} may provide further performance improvements. Our results also suggest that it may be possible to derive theoretical guarantees regarding transfer between policies or MDPs \citep{taylor09transfer}, in particular with a learned representation \citep{agarwal2021contrastive}.

\subsubsection*{Acknowledgements}
The authors would like to thank Matthieu Geist, Mark Rowland, Pablo Samuel Castro, Ahmed Touati, Marlos Machado, Dale Schuurmans, Robert Dadashi, Tomas Vaskevicius, Olivier Pietquin, Martha White, Hanie Sedghi, Damien Vincent, Dominic Richards, Nino Vieillard, Leonard Hussenot, Amartya Sanyal, Sephora Madjiheurem, Laura Toni and the anonymous reviewers for useful discussions and feedback on this paper.

We would also like to thank
the Python community \citep{van1995python,oliphant2007python} for developing
tools that enabled this work, including
{NumPy}~\citep{oliphant2006guide,walt2011numpy, harris2020array},
{SciPy}~\citep{jones2001scipy},
{Matplotlib}~\citep{hunter2007matplotlib} and JAX \citep{bradbury2018jax}.

\bibliographystyle{plainnat}
\bibliography{bib}
\clearpage
\onecolumn
\begin{appendix}
\pdfoutput=1
\hsize\textwidth
\linewidth\hsize \toptitlebar {\centering
{\Large\bfseries On the Generalization of Representations in Reinforcement Learning: \\Appendices \par }}
\bottomtitlebar

\section{PROOFS FOR SECTION \ref{sec:genbounds}}
\label{app:basicboundmoregeneral}

\newcommand{\numin}{\underline{\nu}}
This section is dedicated to proving the main theorem on the paper, \cref{thm:main_gen_error}. Before that, we introduce and prove a more general result from which \cref{thm:main_gen_error} can be deduced as a corollary.

Let $s_1, ..., s_n$ denote iid draws from an arbitrary distribution $\nu \in \cP(\sspace)$ and $(e_i)_{i=1}^S \subset \rR^S$ the standard basis.
\begin{assumption}
\label{assump:distribution}
We assume that $\nu(s)>0$ for all state state $s \in \{1, ..., S\}.$
\end{assumption}

Let $N := \E_{i \sim \nu} [e_ie_i^\T]$,
and let $\norm{x}_{\nu,2} := \norm{N^{1/2}x}_2$
for $x \in \rR^S$.
Put $\numin := \min_{i=1, ..., S} \nu_i > 0$.
Let $w^* := (\Phi^\T N \Phi)^{-1} \Phi^\T N V$,
and also define $\Xi := \Phi^\T N \Phi$. $\Xi$ is the steady-state feature covariance matrix. $w^*$ represents the best $k$-dimensional model.
Since we assume that $\numin > 0$, we have that
$\Xi$ is positive definite.

The excess risk $\cE(V_{\phi, {w}})$ of a hypothesis $V_{\phi, {w}} : \mathcal{S} \rightarrow \rR$ is defined as:
\begin{align*}
    \cE(V_{\phi, {w}}) := \E_{s_i \sim \nu} (V_{\phi, {w}}(s_i) - V(s_i))^2.
\end{align*}

For any $\hat{w} \in \rR^k$, we have
the decomposition:
\begin{align*}
  \cE(V_{\phi, \hat{w}}) =  \norm{\Phi \hat{w} - V}_{\nu,2}^2 = \norm{\Phi(\hat{w}-w^*)}_{\nu,2}^2 + \norm{\Phi w^* - V}_{\nu,2}^2.
\end{align*}

Note we have the identity:
\begin{align*}
    \norm{\Phi w^* - V}_{\nu,2}^2 = \norm{P^\perp_{N^{1/2} \Phi} N^{1/2} V}^2_2.
\end{align*}
\begin{restatable}[]{theorem}{maingenerrorgen}
\label{thm:main_gen_error_gen}
Fix any $\delta \in (0, 1)$. 
Suppose that $n \geq 8 \deff(\Phi) \log(6k/\delta)$. Under \cref{assump:distribution},
with probability at least $1-\delta$,
the empirical risk minimizer $\vfhatapprox$ satisfies:
\begin{align*}
     \cE(V_{\phi, \hat{w}}) &= \norm{P^\perp_{N^{1/2} \Phi} N^{1/2} V}^2_2 + 384 \frac{\deff(\Phi)}{\numin n S} \norm{P^{\perp}_{N^{1/2} \Phi} N^{1/2} V}_2^2 \log({3}/{\delta})\\
     &+ 48\frac{\sigma^2}{n}[2k + 3 \log(3/\delta)]
    + \frac{64}{3} \frac{\deff(\Phi)}{\numin n^2S} \norm{ N^{-1/2} P^\perp_{N^{1/2} \Phi} N^{1/2} V }_\infty^2 \log^2({3}/{\delta}).
\end{align*}
where $\norm{\cdot}_{\infty}$ denotes the usual supremum norm.
\end{restatable}
\begin{proof}
The empirical risk minimizer $\hat{w} \in \rR^k$ is defined as
the random vector
$\hat{w} = (E_n \Phi)^{\dag} Y$.
Next, 
we write:
\begin{align*}
  N^{1/2} \Phi(\hat{w} - w^*) &= N^{1/2} \Phi (E_n \Phi)^{\dag}(E_n V + \eta) - N^{1/2} \Phi w^*
\end{align*}
Therefore, assuming $E_n \Phi$ has full column rank (which will be the case by \cref{lemma:chernoff-wn}),
\begin{align*}
    &N^{1/2} \Phi (E_n \Phi)^{\dag} E_n V - N^{1/2} \Phi w^* \\
    &=N^{1/2} \Phi (E_n \Phi)^{\dag} E_n V - P_{N^{1/2} \Phi} N^{1/2} V \\
    &= N^{1/2} \Phi (\Phi^\T E_n^\T E_n \Phi)^{-1} \Phi^\T E_n^\T E_n V - P_{N^{1/2} \Phi} N^{1/2} V \\
    &= N^{1/2} \Phi (\Phi^\T E_n^\T E_n \Phi)^{-1} \Phi^\T E_n^\T E_n N^{-1/2} ( P_{N^{1/2} \Phi} + P^\perp_{N^{1/2} \Phi} ) N^{1/2} V - P_{N^{1/2} \Phi} N^{1/2} V \\
    &= N^{1/2} \Phi (\Phi^\T E_n^\T E_n \Phi)^{-1} \Phi^\T E_n^\T E_n N^{-1/2} P^{\perp}_{N^{1/2} \Phi} N^{1/2} V \\
    &\qquad+ N^{1/2} \Phi (\Phi^\T E_n^\T E_n \Phi)^{-1} \Phi^\T E_n^\T E_n N^{-1/2} P_{N^{1/2} \Phi} N^{1/2} V - P_{N^{1/2} \Phi} N^{1/2} V \\
    &= N^{1/2} \Phi (\Phi^\T E_n^\T E_n \Phi)^{-1} \Phi^\T E_n^\T E_n N^{-1/2} P^{\perp}_{N^{1/2} \Phi} N^{1/2} V \\
    &= N^{1/2} \Phi \Xi^{-1/2} (\Xi^{-1/2} \Phi^\T E_n^\T E_n \Phi \Xi^{-1/2})^{-1} \Xi^{-1/2} \Phi^\T E_n^\T E_n  N^{-1/2} P^{\perp}_{N^{1/2} \Phi} N^{1/2} V.
\end{align*}
Similarly,
\begin{align*}
    N^{1/2} \Phi (E_n \Phi)^{\dag} \eta &= N^{1/2} \Phi (\Phi^\T E_n^\T E_n \Phi)^{-1} \Phi^\T E_n^\T \eta \\
    &= N^{1/2} \Phi \Xi^{-1/2} (\Xi^{-1/2} \Phi^\T E_n^\T E_n \Phi \Xi^{-1/2})^{-1} \Xi^{-1/2} \Phi^\T E_n^\T \eta.
\end{align*}
We first claim that
$\opnorm{ N^{1/2} \Phi \Xi^{-1/2} } \leq 1$.
To see this, observe that:
\begin{align*}
    \opnorm{ N^{1/2} \Phi \Xi^{-1/2} }^2 = \lambda_{\max}( N^{1/2} \Phi (\Phi^\T N \Phi)^{-1} \Phi^\T N^{1/2} ) 
    = \lambda_{\max}( P_{N^{1/2} \Phi} ) \leq 1.
\end{align*}
Hence:
\begin{align*}
    \norm{ N^{1/2} \Phi (E_n \Phi)^{\dag} E_n V - N^{1/2} \Phi w_* }_2 
    \leq \frac{\norm{\Xi^{-1/2} \Phi^\T E_n^\T E_n  N^{-1/2} P^{\perp}_{N^{1/2} \Phi} N^{1/2} V}_2 }{\lambda_{\min}(\Xi^{-1/2} \Phi^\T E_n^\T E_n \Phi \Xi^{-1/2} )},
\end{align*}
and similarly
\begin{align*}
    \norm{ N^{1/2} \Phi (E_n\Phi)^{\dag} \eta }_2 \leq \frac{ \norm{\Xi^{-1/2} \Phi^\T E_n^\T \eta}_2 }{\lambda_{\min}(\Xi^{-1/2} \Phi^\T E_n^\T E_n \Phi \Xi^{-1/2} ) }.
\end{align*}
Therefore, 
\begin{align*}
    \|N^{1/2} \Phi(\hat{w} - w^*)\|_2 \leq \frac{1}{\lambda_{\min}(\Xi^{-1/2} \Phi^\T E_n^\T E_n \Phi \Xi^{-1/2})}[\norm{\Xi^{-1/2} \Phi^\T E_n^\T E_n  N^{-1/2} P^{\perp}_{N^{1/2} \Phi} N^{1/2} V}_2+ \norm{\Xi^{-1/2} \Phi^\T E_n^\T \eta}_2]
\end{align*}
By \cref{lemma:chernoff-wn}, as long as $n \geq  \frac{8\deff(\Phi)}{\numin S} \log(6k/\delta)$, then with probability at least $1-\delta/3$, 
\begin{align*}
    \frac{n}{2} I_k \preccurlyeq \Xi^{-1/2} \Phi^\T E_n^\T E_n \Phi \Xi^{-1/2} \preccurlyeq {4n} I_k.
\end{align*}
Furthermore, by \cref{lemma:bernstein-wn},
with probability at least $1-\delta/3$,
\begin{align*}
    &\norm{\Xi^{-1/2} \Phi^\T E_n^\T E_n N^{-1/2} P^{\perp}_{N^{1/2} \Phi} N^{1/2} V}_2 \\
    &\leq 2 \sqrt{ \frac{8n \deff(\Phi)}{\numin S} \norm{P^{\perp}_{N^{1/2} \Phi} N^{1/2} V}^2_2 \log\left(\frac{3}{\delta}\right)} + \frac{4}{3} \sqrt{\frac{\deff(\Phi)}{\numin S}} \norm{ N^{-1/2} P^\perp_{N^{1/2} \Phi} N^{1/2} V }_\infty \log\left(\frac{3}{\delta}\right).
\end{align*}
Finally, by \cref{lemma:mds-wn},
with probability at least $1-\delta/3$,
\begin{align*}
    \ind\left\{ \Xi^{-1/2} \Phi^\T E_n^\T E_n \Phi \Xi^{-1/2} \preccurlyeq 4n I_k \right\} \cdot \norm{ \Xi^{-1/2} \Phi^\T E_n^\T \eta }_2 \leq \sqrt{\sigma^2 n [ 8k + 12 \log(3/\delta)]}.
\end{align*}
Therefore, by a union bound, with probability at least $1-\delta$,
\begin{align*}
    \|N^{1/2} \Phi(\hat{w} - w^*)\|_2 &\leq \frac{2}{n} \left[2 \sqrt{ \frac{8n \deff(\Phi)}{\numin S} \norm{P^{\perp}_{N^{1/2} \Phi} N^{1/2} V}^2_2 \log\left(\frac{3}{\delta}\right)}\right] \\
    &+ \frac{2}{n} \left[\frac{4}{3} \sqrt{\frac{\deff(\Phi)}{\numin S}} \norm{ N^{-1/2} P^\perp_{N^{1/2} \Phi} N^{1/2} V }_\infty \log\left(\frac{3}{\delta}\right)\right]
   + \frac{2}{n} \left[\sqrt{\sigma^2 n [ 8k + 12 \log(3/\delta)]}\right]\\
    &= 4\sqrt{8} \sqrt{ \frac{\deff(\Phi)}{\numin n S} \log(3/\delta)} \norm{P^{\perp}_{N^{1/2} \Phi} N^{1/2} V}_2 + 4\sqrt{\frac{\sigma^2}{n}[2k + 3\log(3/\delta)]}\\
    &\qquad+ \frac{8}{3} \frac{\sqrt{\frac{\deff(\Phi)}{\numin S}}}{n} \norm{ N^{-1/2} P^\perp_{N^{1/2} \Phi} N^{1/2} V }_\infty \log\left(\frac{3}{\delta}\right).
\end{align*}
Now, from the inequality $
(a+b+c)^{2} \leqslant 3\left(a^{2}+b^{2}+c^{2}\right) \text { for any } a, b, c \in \mathbb{R},
$ it follows that
\begin{align*}
     \cE(V_{\phi, \hat{w}}) &= \norm{P^\perp_{N^{1/2} \Phi} N^{1/2} V}^2_2 + 384 \frac{\deff(\Phi)}{\numin n S} \norm{P^{\perp}_{N^{1/2} \Phi} N^{1/2} V}_2^2 \log({3}/{\delta})\\
     &+ 48\frac{\sigma^2}{n}[2k + 3 \log(3/\delta)]
    + \frac{64}{3} \frac{\deff(\Phi)}{\numin n^2S} \norm{ N^{-1/2} P^\perp_{N^{1/2} \Phi} N^{1/2} V }_\infty^2 \log^2({3}/{\delta}).
\end{align*}
\end{proof}
\begin{lemma}
\label{lemma:chernoff-wn}
Let $\Phi \in \rR^{S \times k}$.
Let $\nu$ denote a distribution over $\{1, ..., S\}$ satisfying \cref{assump:distribution} and $(e_i)_{i=1}^S \subset \rR^S$ the standard basis.
Let $s_1, ..., s_n$ denote iid draws from $\nu$.
Define $Y_n \in \rR^{k \times k}$ as:
\begin{align*}
    Y_n = \sum_{i=1}^{n}  \Xi^{-1/2} \Phi^\T e_{s_i} e_{s_i}^\T \Phi \Xi^{-1/2}.
\end{align*}
Fix any $\delta \in (0, 1)$.
As long as $n \geq  \frac{8\deff(\Phi)}{\numin S} \log(2k/\delta)$,
with probability at least $1 - \delta$,
\begin{align*}
    \frac{n}{2} I_k \preccurlyeq Y_n \preccurlyeq 4n I_k.
\end{align*}
where for two symmetric matrices, $A\preccurlyeq B$ means that the matrice $B-A$ is positive semi-definite.
\end{lemma}
\begin{proof}
This is an application of the Matrix Chernoff inequality.
First, we see that
$\E[Y_n] = n I_k$.
Next, we have:
\begin{align*}
    \max_{i=1, ..., S} \lambda_{\max}( \Xi^{-1/2} \Phi^\T e_{i} e_{i}^\T \Phi \Xi^{-1/2} ) &= \max_{i=1, ..., S} \norm{\Xi^{-1/2} \Phi^\T e_i}_2^2 \\
     &= \max_{i=1, ..., S} \norm{(\Phi^\T N \Phi)^{-1/2} \Phi^\T e_i}_2^2 \\
    &\leq \frac{1}{\numin} \max_{i=1, ..., S} \norm{P_\Phi e_i}_2^2 \\
    &\leq \frac{\deff(\Phi)}{\numin S}.
\end{align*}
We now make two applications of the Matrix Chernoff inequality (see Theorem 5.1.1 in \cite{tropp15introduction}). Denoting $\mathit{e}$ as Euler's number, for the upper tail, we have that for any $t \geq \mathit{e}$,
\begin{align*}
    \Pr(\lambda_{\max}(Y_n) \geq t n ) \leq k (\mathit{e}/t)^{t n \numin S/\deff(\Phi)}.
\end{align*}
Setting $t=4$, we conclude that as long as $n \geq \frac{1}{4 \log(4/\mathit{e})} \frac{\deff(\Phi)}{\numin S} \log(2k/\delta)$, then we have that
with probability at least $1-\delta/2$, $\lambda_{\max}(Y_n) \leq 4n$.
For the lower tail, we have that for any $t \in (0, 1)$,
\begin{align*}
    \Pr(\lambda_{\min}(Y_n) \leq t n ) &\leq k \exp\left(-(1-t)^2 \frac{n}{2}  \frac{\numin S}{\deff(\Phi)} \right).
\end{align*}
Setting $t = 0.5$,
we see that as long as $n \geq 8 \frac{\deff(\Phi)}{\numin S} \log(2k/\delta)$, then
$\lambda_{\min}(Y_n) \geq n/2$ with probability
at least $1 - \delta/2$. Taking a union bound
yields the claim.
\end{proof}

\begin{lemma}
\label{lemma:bernstein-wn}
Put $z_n := \Xi^{-1/2} \Phi^\T E_n^\T E_n N^{-1/2} P^{\perp}_{N^{1/2} \Phi} N^{1/2} V$.
Fix any $ \delta \in (0, e^{-1/8})$.
With probability at least $1-\delta$,
\begin{align*}
    \norm{z_n}_2 \leq 2 \sqrt{ \frac{8n \deff(\Phi)}{\numin S} } \norm{P^{\perp}_{N^{1/2} \Phi} N^{1/2} V}_2 \sqrt{\log(1/\delta)} + \frac{4}{3} \sqrt{\frac{\deff(\Phi)}{\numin S}} \norm{ N^{-1/2} P^\perp_{N^{1/2} \Phi} N^{1/2} V }_\infty \log(1/\delta).
\end{align*}
\end{lemma}
\begin{proof}
Define $q_i := \Xi^{-1/2} \Phi^\T e_{s_i} e_{s_i}^\T N^{-1/2} P^\perp_{N^{1/2} \Phi} N^{1/2} V$.
We have that $\E[q_i] = 0$.
Next,
\begin{align*}
    \E[\norm{q_i}_2^2] &= \E[ \norm{\Xi^{-1/2} \Phi^\T e_{s_i}}_2^2 \ip{e_{s_i}}{N^{-1/2} P^\perp_{N^{1/2} \Phi} N^{1/2} V  }^2 ] \\
    &\leq \frac{\deff(\Phi)}{\numin S} \E[ \ip{e_{s_i}}{N^{-1/2} P^\perp_{N^{1/2} \Phi} N^{1/2} V  }^2 ] \\
    &= \frac{\deff(\Phi)}{\numin S} \norm{ P^\perp_{N^{1/2} \Phi} N^{1/2} V}_2^2.
\end{align*}
Finally, we have the following almost sure bound:
\begin{align*}
    \norm{q_i}_2 \leq \sqrt{ \frac{\deff(\Phi)}{\numin S}} \norm{ N^{-1/2} P^{\perp}_{N^{1/2} \Phi} N^{1/2} V}_\infty.
\end{align*}

Put $z_n := \sum_{i=1}^{n} q_i$.
By the vector Bernstein inequality,
for all $t > 0$,
\begin{align*}
    \Pr\left( \norm{z_n}_2 > \sqrt{ \frac{n \deff(\Phi)}{\numin S} \norm{P^\perp_{N^{1/2} \Phi} N^{1/2} V }_2^2 }(1+\sqrt{8t}) + \frac{4}{3} \sqrt{ \frac{\deff(\Phi)}{\numin S}} \norm{ N^{-1/2} P^{\perp}_{N^{1/2} \Phi} N^{1/2} V}_\infty t \right) \leq e^{-t}.
\end{align*}
The claim now follows by setting $t = \log(1/\delta)$.
\end{proof}

\begin{lemma}
\label{lemma:mds-wn}
Let $\cG$ be the event:
\begin{align*}
    \cG := \left\{ \Xi^{-1/2} \Phi^\T E_n^\T E_n \Phi \Xi^{-1/2} \preccurlyeq 4n I_k \right\}
\end{align*}
With probability at least $1-\delta$, we have:
\begin{align*}
    \ind\{ \cG \} \cdot \norm{ \Xi^{-1/2} \Phi^\T E_n^\T \eta }_2^2 \leq \sigma^2 n [ 8k + 12 \log(1/\delta)]. 
\end{align*}
\end{lemma}
\begin{proof}
Put $M := \ind\{\cG\} \cdot E_n \Phi \Xi^{-1} \Phi^\T E_n^\T$.
Because $\eta$ is assumed to be independent of $E_n$, we can condition on $E_n$
and apply the Hanson-Wright inequality \citep{hsu2012tail} to conclude
that for any $t > 0$,
    \begin{align*}
        \Pr( \eta^\T M \eta > \sigma^2 ( \Tr(M) + 2 \sqrt{\Tr(M^2) t} + 2\opnorm{M} t)) \mid E_n ) \leq e^{-t}.
    \end{align*}
We now compute upper bounds on $\Tr(M)$, $\Tr(M^2)$, and $\opnorm{M}$.
First, we have:
    \begin{align*}
        \Tr(M) = \ind\{\cG\} \Tr(E_n \Phi \Xi^{-1} \Phi^\T E_n^\T) = \ind\{\cG\} \Tr( \Xi^{-1/2} \Phi^\T E_n^\T E_n \Phi \Xi^{-1/2} ) \leq 4nk.
    \end{align*}
Next,
    \begin{align*}
        \Tr(M^2) &= \ind\{\cG\} \Tr( E_n \Phi \Xi^{-1} \Phi^\T E_n^\T E_n \Phi \Xi^{-1} \Phi^\T E_n^\T ) \\
        &= \ind\{\cG\} \Tr( \Xi^{-1/2}\Phi^\T E_n^\T E_n \Phi \Xi^{-1/2} \cdot \Xi^{-1/2} \Phi^\T E_n^\T E_n \Phi \Xi^{-1/2} ) \\
        &\stackrel{(a)}{\leq} \ind\{\cG\} \Tr(\Xi^{-1/2} \Phi^\T E_n^\T E_n \Xi^{-1/2}\Phi) \opnorm{\Xi^{-1/2} \Phi^\T E_n^\T E_n \Phi \Xi^{-1/2}} \\
        &\leq 4nk \cdot 4n = 16n^2 k.
    \end{align*}
Above, (a) follows from H\"older's inequality.
Finally,
    \begin{align*}
        \opnorm{M} = \ind\{\cG\} \opnorm{  E_n \Phi \Xi^{-1} \Phi^\T E_n^\T  } = \ind\{\cG\} \opnorm{ \Xi^{-1/2} \Phi^\T E_n^\T E_n \Phi \Xi^{-1/2} } \leq 4n.
    \end{align*}
We now plug these bounds in along with the choice of $t = \log(1/\delta)$,
which tells us that conditioned on $E_n$, 
with probability at least $1-\delta$,
    \begin{align*}
        \eta^\T M \eta &\leq \sigma^2 \left[  4nk + 8n \sqrt{k \log(1/\delta)} + 8n \log(1/\delta) \right] \\
        &\leq \sigma^2 \left[  8nk + 12n \log(1/\delta) \right] \\
        &= \sigma^2 n \left[ 8k + 12 \log(1/\delta) \right].
    \end{align*}
We now remove the conditioning on $E_n$.
Let $\bar{t} := \sigma^2 n \left[ 8k + 12 \log(1/\delta) \right]$.
By the tower property,
\begin{align*}
    \Pr( \eta^\T M \eta \geq \bar{t} ) = \E[ \ind\{ \eta^\T M \eta \geq \bar{t} \}] = \E[ \E[ \ind\{ \eta^\T M \eta \geq \bar{t} \} \mid E_n ] ] = \E[ \Pr( \eta^\T M \eta \geq \bar{t} \mid E_n) ] \leq \E[ \delta ] = \delta.
\end{align*}
\end{proof}
\cref{thm:main_gen_error} is a corollary of \cref{thm:main_gen_error_gen} in the case where the distribution $\nu$ is uniform.
\maingenerror*
\begin{proof}
$\nu$ being uniform, we have $\numin=S$. The result follows by plugging $\numin$ in \cref{thm:main_gen_error_gen}.
\end{proof}

\section{PROOFS FOR SECTION \ref{sec:sr-bound}}
\label{app:proof-srbound}
\singularvalues*
\begin{proof}
Let $\lambda(\cdot)$ denote the eigenvalues of a matrix.
Because $P$ is symmetric, we have that:
\begin{align*}
    \sigma((I - \gamma P)^{-1}) = \left\{ \frac{1}{1-\gamma \lambda} : \lambda \in \lambda(P) \right\}.
\end{align*}
Because $P$ is a row stochastic matrix, we have that the
spectral radius of $P$ satisfies $\rho(P) = 1$, and therefore $\lambda(P) \subseteq [-1, 1]$.
Hence:
\begin{align*}
    \frac{1}{1-\gamma \lambda} \in [1/(1+\gamma), 1/(1-\gamma)].
\end{align*}
\end{proof}

\textbf{Eigenstructure of the Star Graph (\cref{sec:effect-structure})}\\
A random walk on the Star graph induces a rank-two transition matrix $P_\pi \in \rR^S$. We may write 
$P_\pi = v_1 e_S^\T + \frac{e_S v^\T}{S-1}$ where $v$ is an all-ones vector except on its last coordinate where it takes value 0 and $e_S$ a one-hot vector taking value $1$ on its last coordinate.
It is easy to prove by induction that 
\begin{itemize}
    \item for any $k \geq 1, P_\pi^{2k}=\frac{vv^\T}{S-1}+ e_Se_S^\T$
    \item for any $k \geq 0, P_\pi^{2k+1}=P_\pi$
\end{itemize}
From this, it follows that 
\begin{align*}
    (I - \gamma P_\pi)^{-1} &= I + \sum_{t=1}^\infty (\gamma \Ppi)^t \\
   &=  I +  \sum_{2k\geq2} \gamma^{2k}P_\pi^{2k} + \sum_{2k+1\geq1} \gamma^{2k+1} (P_\pi)^{2k+1}\\
     &=  I +  \sum_{2k\geq2} \gamma^{2k}(\frac{vv^\T}{S-1}+ e_Se_S^\T) + \sum_{2k+1\geq1} \gamma^{2k+1} P_\pi\\
    &= I + \frac{\gamma^2}{1-\gamma^2} (\frac{vv^\T}{S-1}+ e_Se_S^\T) + \frac{\gamma}{1-\gamma^2} P_\pi.
\end{align*}
Define $\eta := \frac{\gamma}{1-\gamma^2}$.
The non-zero singular values of $(I - \gamma P_\pi)^{-1}$ are the square roots of the eigenvalues of $A=(I - \gamma P_\pi)^{-1}\left((I - \gamma P_\pi)^{-1}\right)^\T$. We have
\begin{align*}
    A=(I - \gamma P_\pi)^{-1}\left((I - \gamma P_\pi)^{-1}\right)^\T &=  \left(I + \gamma \eta P_\pi^2 + \eta P_\pi \right) \left(I + \gamma \eta (P_\pi^2)^\T + \eta P_\pi^\T \right)\\
    &= I + B,
\end{align*}
where $B := a vv^\T + b e_Se_S^\T + c (e_Sv^\T + ve_S^\T)$ with $a=\frac{2\eta\gamma+\eta^2\gamma^2}{S-1}+\eta^2$, $b={2\eta\gamma+\eta^2\gamma^2}+\frac{\eta^2}{S-1}$ and $c=(\eta+\eta^2\gamma)\frac{S}{S-1}.$

Moreover, if $\{\lambda_1, ..., \lambda_k \}$ are the eigenvalues of $B$ then the eigenvalues of $A$ are $\{1+\lambda_1, ..., 1+\lambda_k \}$.

Consider the basis $\{e_S, v\}$.
For any $a_1, a_2$, 
\begin{align*}
    B(a_1 e_S + a_2 v) &= a vv^\T (a_1 e_S + a_2 v) + b e_Se_S^\T (a_1 e_S + a_2 v) + c (e_Sv^\T + ve_S^\T) (a_1 e_S + a_2 v)\\
    &=  a_1a \ip{v}{e_S} v + a_2 a\norm{v}_2^2 v + a_1be_S+a_2b\ip{v}{e_S} e_S +
    c(a_1\ip{v}{e_S} e_S + a_1 v + a_2 \norm{v}_2^2 e_S + a_2 \ip{v}{e_S}v)\\
    &= (a_1b+c a_1 \ip{v}{e_S}+a_2b\ip{v}{e_S}+a_2c\norm{v}_2^2) e_S + (a_1a \ip{v}{e_S}+ca_1 +a_2 \ip{v}{e_S}+a_2 a\norm{v}_2^2)v.
\end{align*}
Since $\norm{v}_2^2 = S-1$ and $\ip{v}{e_S} = 0$,
$B$ has the representation in $\{e_S, v\}$ as:
\begin{align*}
    \begin{bmatrix}
     b  & c(S-1) \\
     c & a(S-1)
    \end{bmatrix}&=
     \begin{bmatrix}
     2\eta\gamma+\eta^2\gamma^2+\frac{\eta^2}{S-1}  & (\eta+\eta^2\gamma)S \\
     (\eta+\eta^2\gamma)\frac{S}{S-1} & 2\eta\gamma+\eta^2\gamma^2+\eta^2(S-1)
    \end{bmatrix}=
     \begin{bmatrix}
     \frac{\eta^2}{S-1}  & (\eta+\eta^2\gamma)S \\
     (\eta+\eta^2\gamma)\frac{S}{S-1} & \eta^2(S-1)
    \end{bmatrix}+ (2\eta\gamma+\eta^2\gamma^2) I\\
   & =C + (2\eta\gamma+\eta^2\gamma^2) I
\end{align*}
Hence, the eigenvalues of C are given by $\frac{1}{2}\left(\eta^2\left((S-1)+ \frac{1}{S-1}\right) \pm \sqrt{\eta^4 \left((S-1)+ \frac{1}{S-1}\right)^2 + 4(\eta+\eta^2\gamma)^2\frac{S^2}{S-1}-4\eta^4}\right)$. The non-zero singular values of $(I - \gamma P_\pi)^{-1}$ are thus $1$ with multiplicity $S-2$ and
$$\sqrt{\frac{1}{2}\left(\eta^2\left((S-1)+ \frac{1}{S-1}\right) \pm \sqrt{\eta^4 \left((S-1)+ \frac{1}{S-1}\right)^2 + 4(\eta+\eta^2\gamma)^2\frac{S^2}{S-1}-4\eta^4}\right) + 2\eta\gamma+\eta^2\gamma^2+1}$$ For $\gamma=0.99$ and $S=400$, we can check numerically that the two extreme singular values are equal to 996 and 0.05 respectively which matches the spectrum obtained for the Star graph in \cref{fig:effectivedim_toymdps}.

\section{EMPIRICAL EVALUATION: ADDITIONAL DETAILS}\label{appendix:experiments}
\subsection{Graphical Structures}
\label{app:toys}
\begin{figure}[b!]
  \centering
      \includegraphics[width=0.19\textwidth]{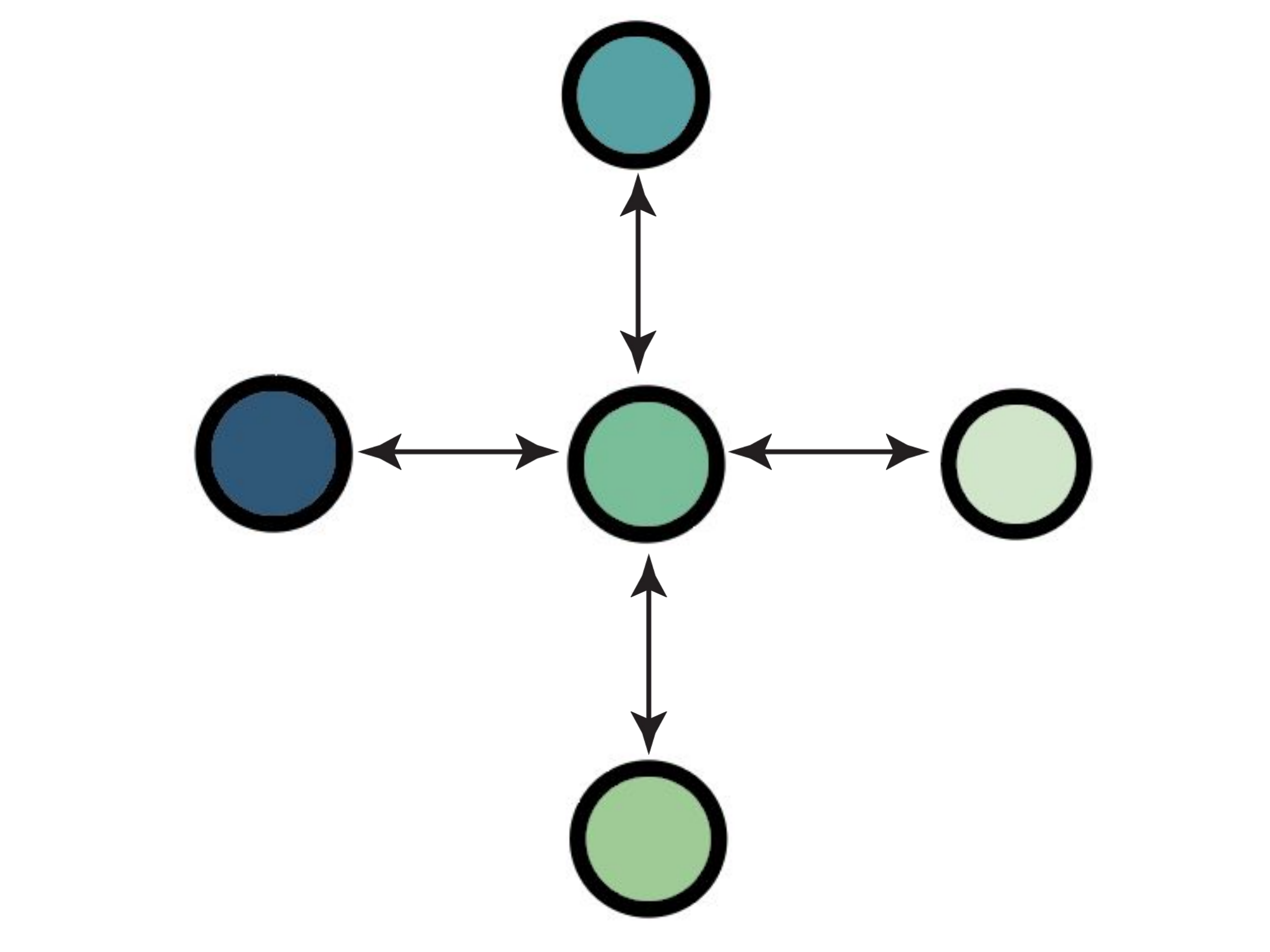}
              \includegraphics[width=0.19\textwidth]{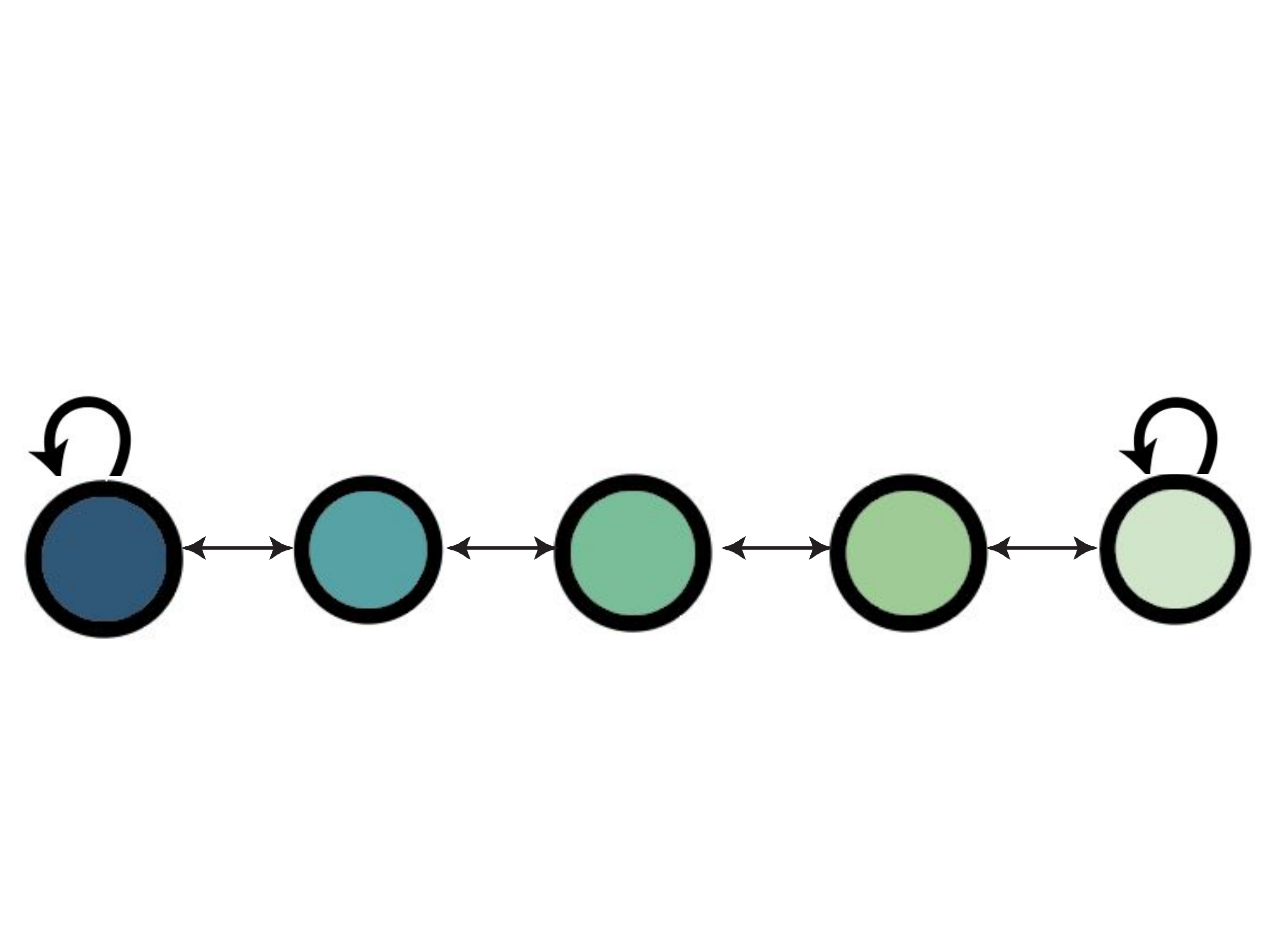}
       \includegraphics[width=0.19\textwidth]{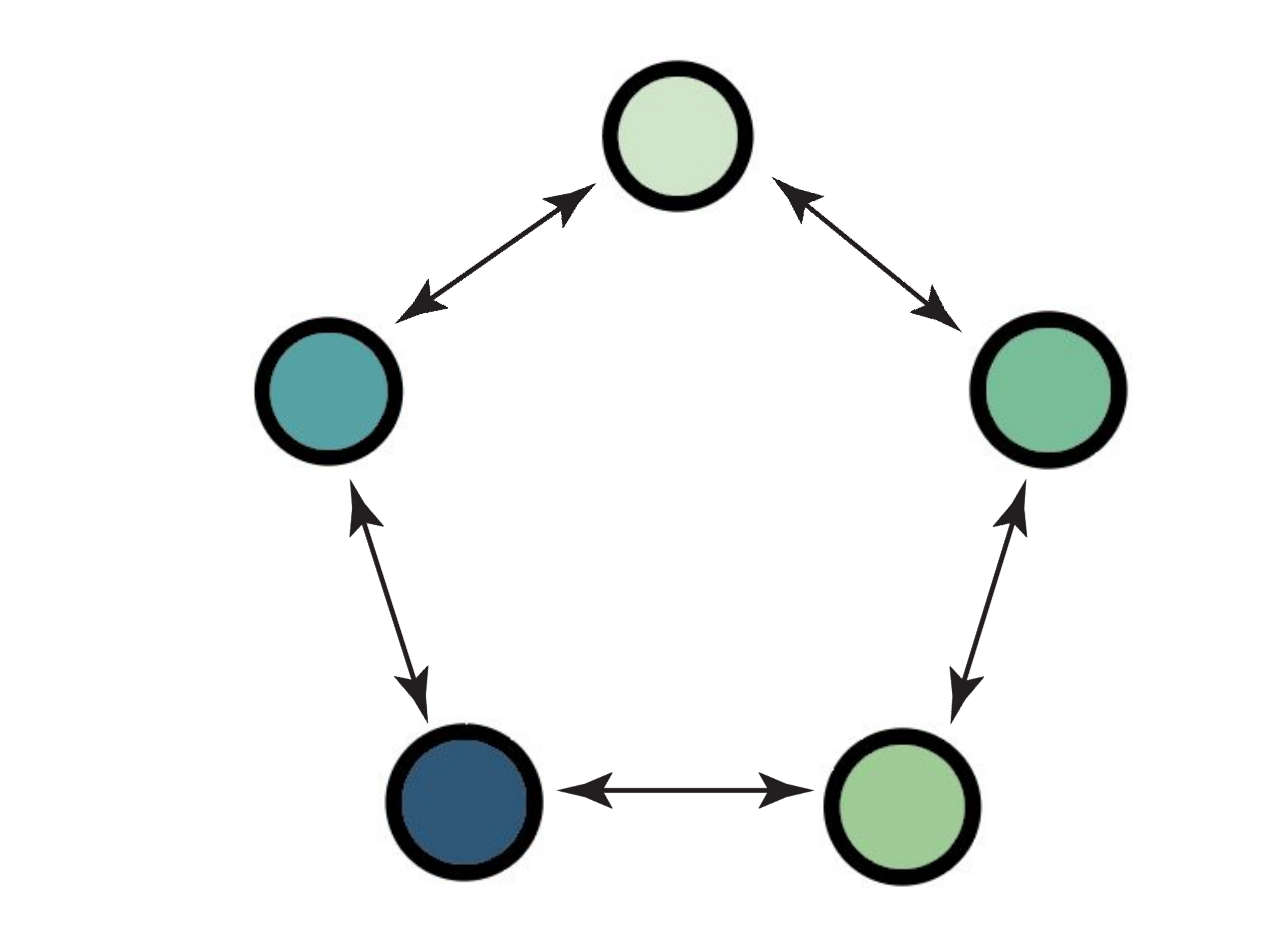}
  \includegraphics[width=0.19\textwidth]{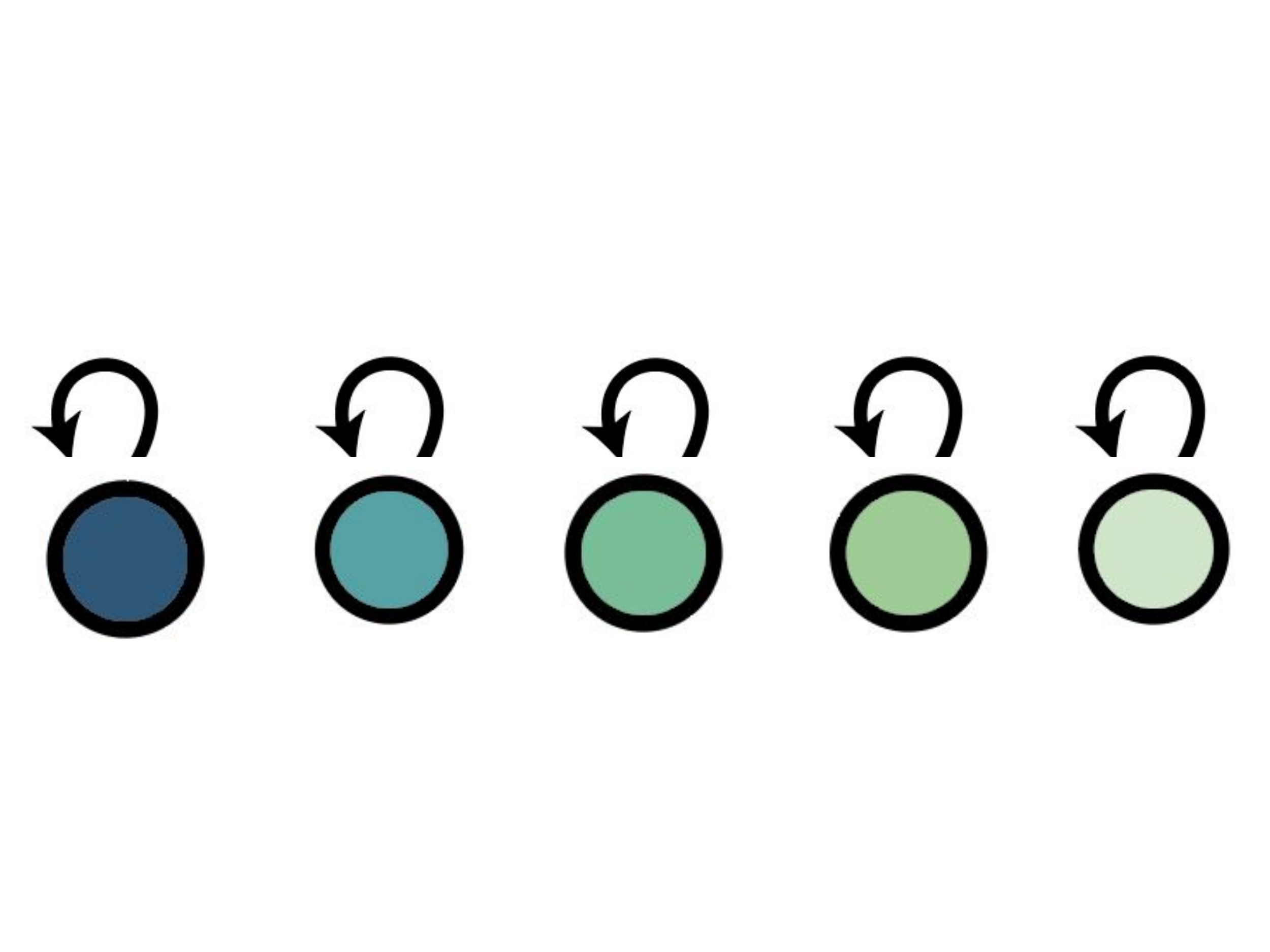}
    \includegraphics[width=0.19\textwidth]{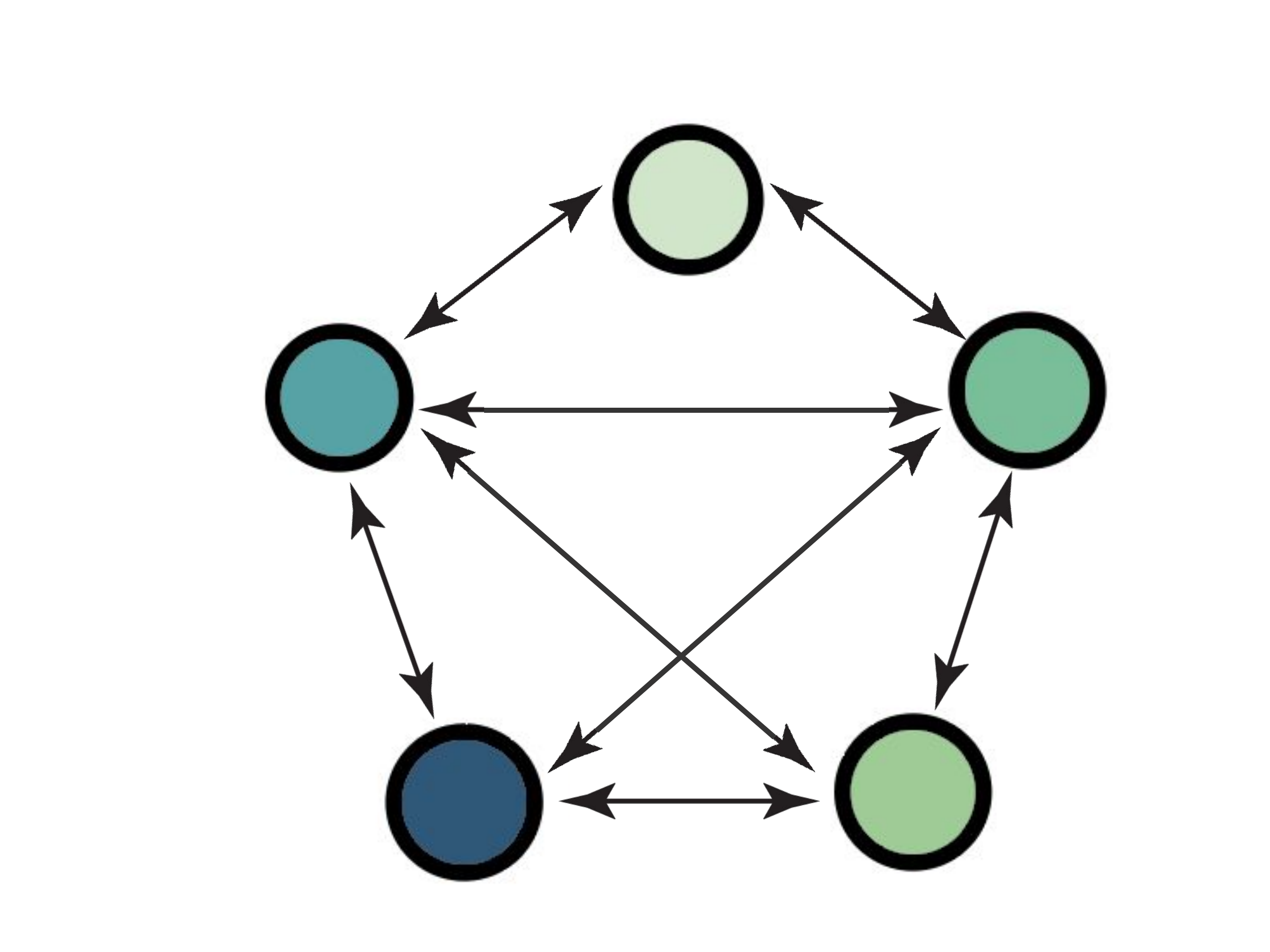}\vfill
    \includegraphics[width=0.22\textwidth]{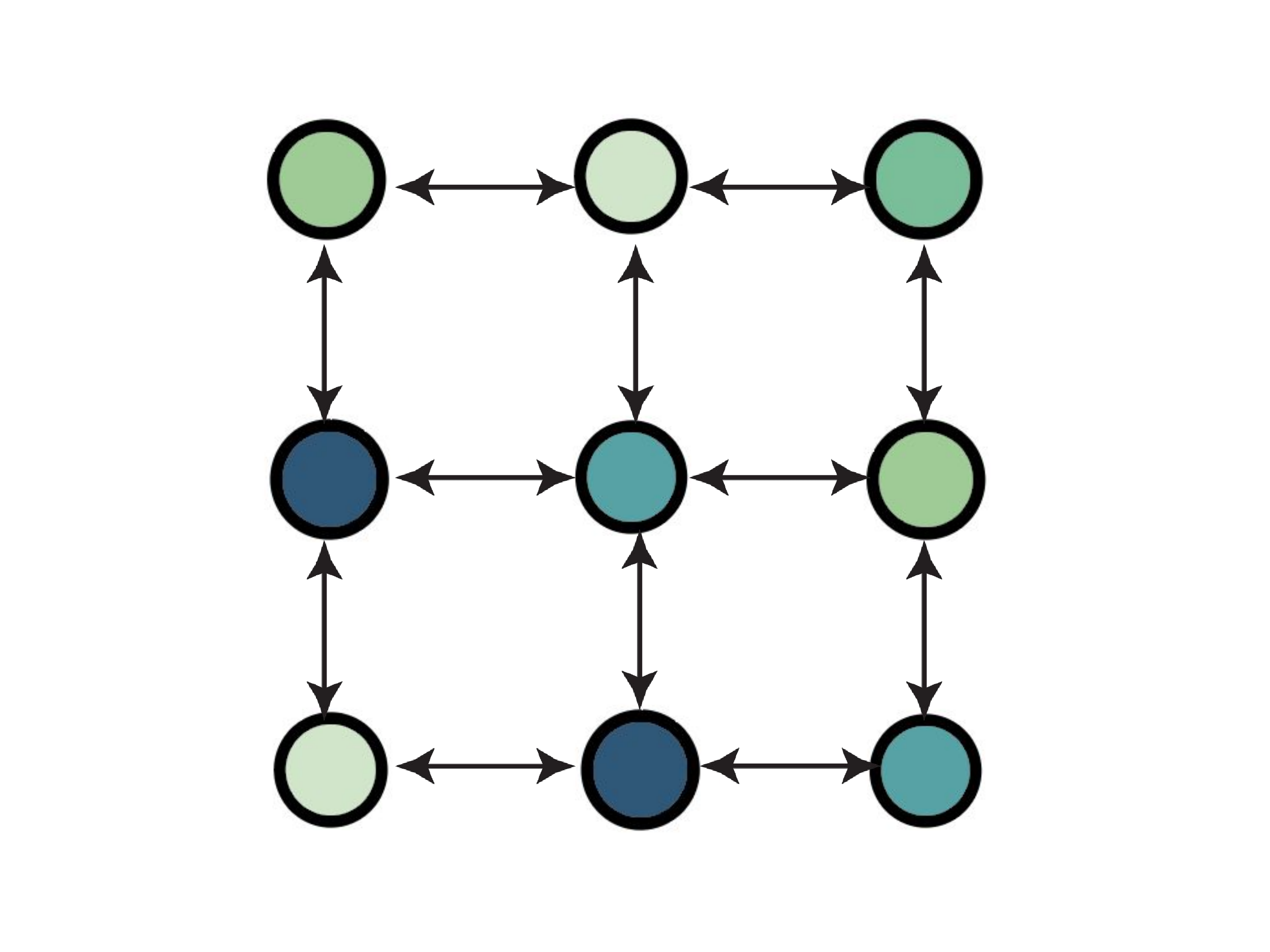}
    \includegraphics[width=0.22\textwidth]{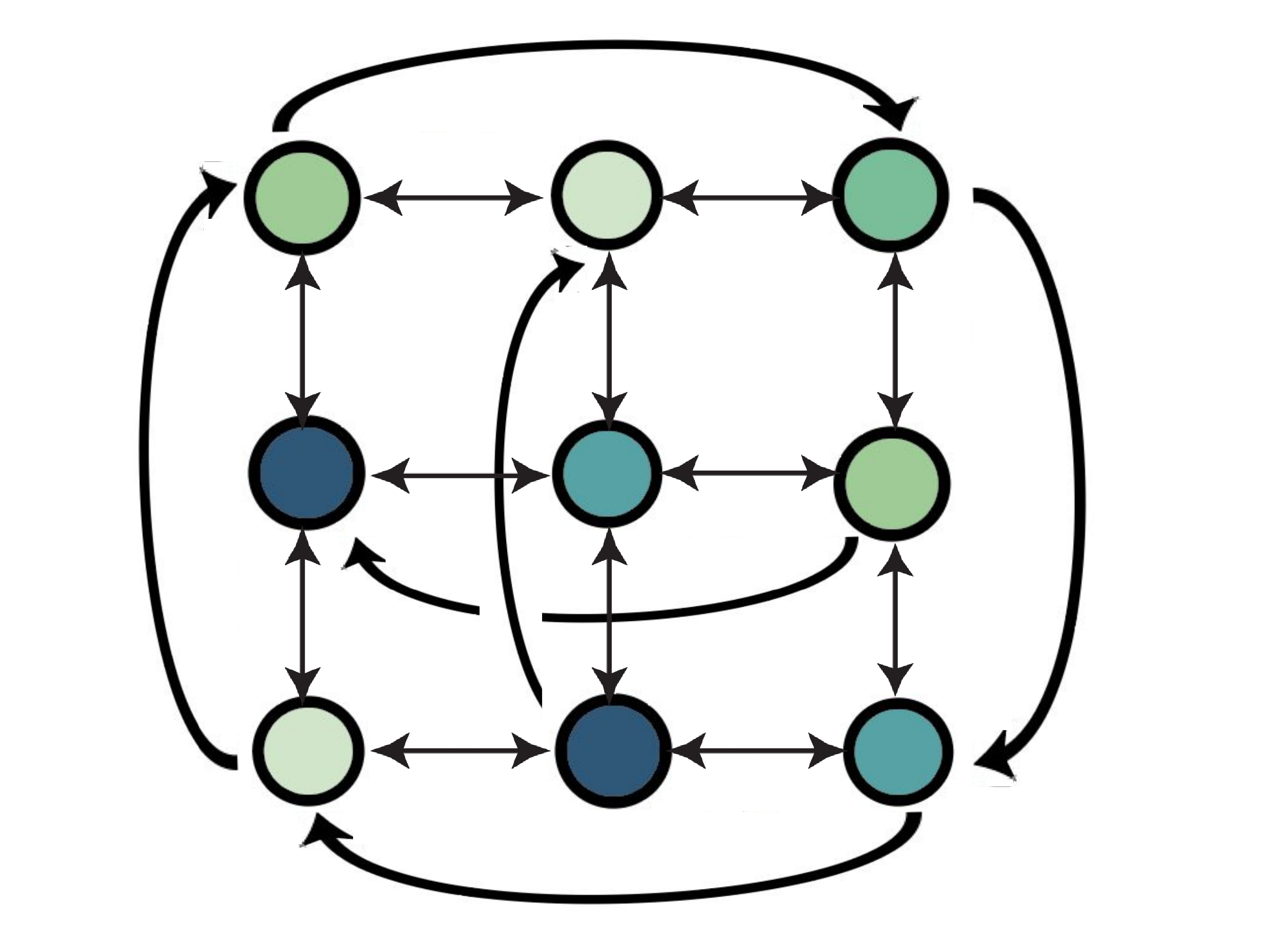}
  \caption{\textbf{Top:} Different graphical structures with $S=5$ states from left to right, Star, Chain,  Torus1d, Disconnected, Fullyconnected.
  \textbf{Bottom:} Two-dimensional graphical structures with $S=9$ states: from left to right, Openroom and Torus2d. 
  }
  \label{fig:graph:illustrations}
  \end{figure}
  
\begin{figure}[h!]
  \centering
  \includegraphics[width=0.22\textwidth]{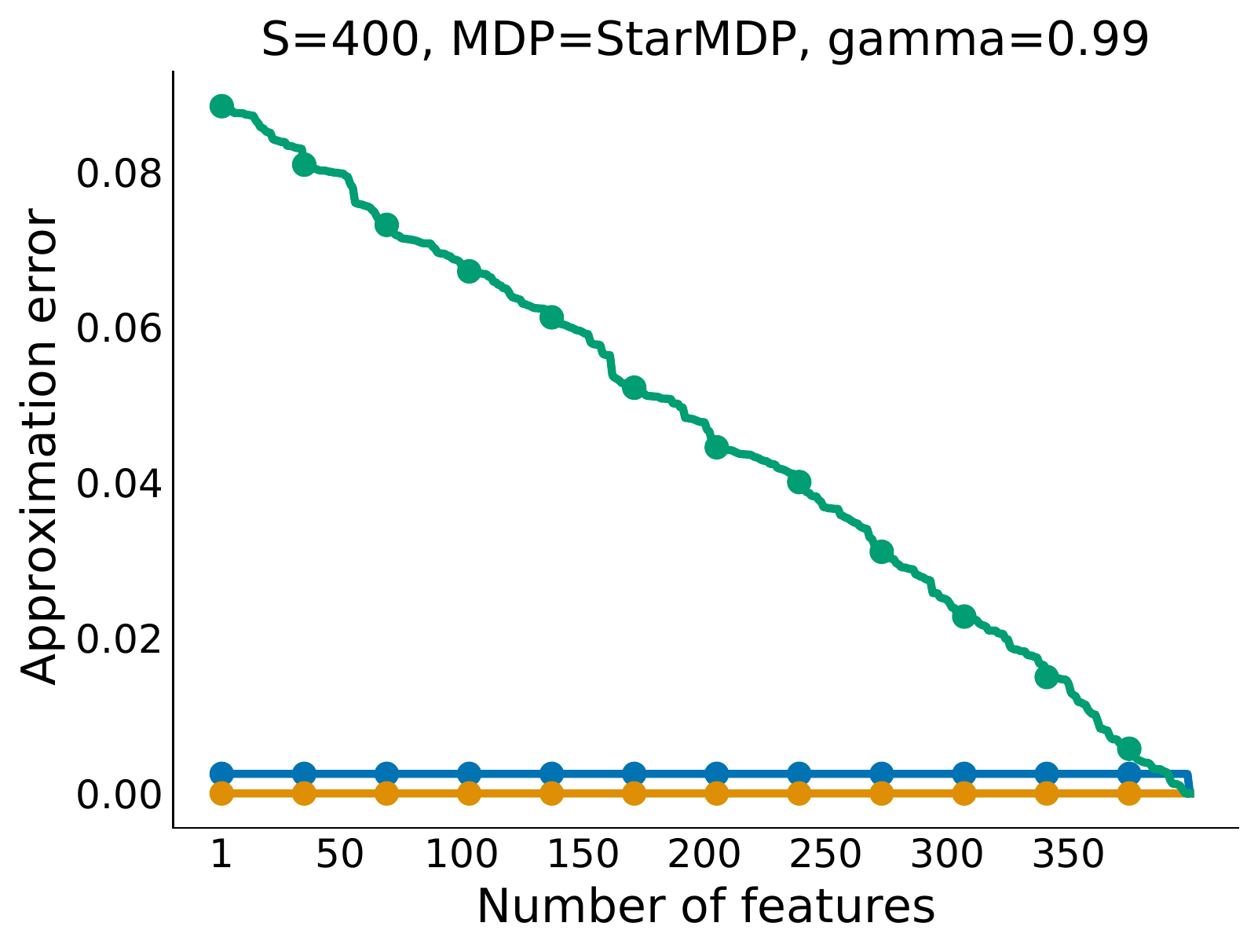}
    \includegraphics[width=0.22\textwidth]{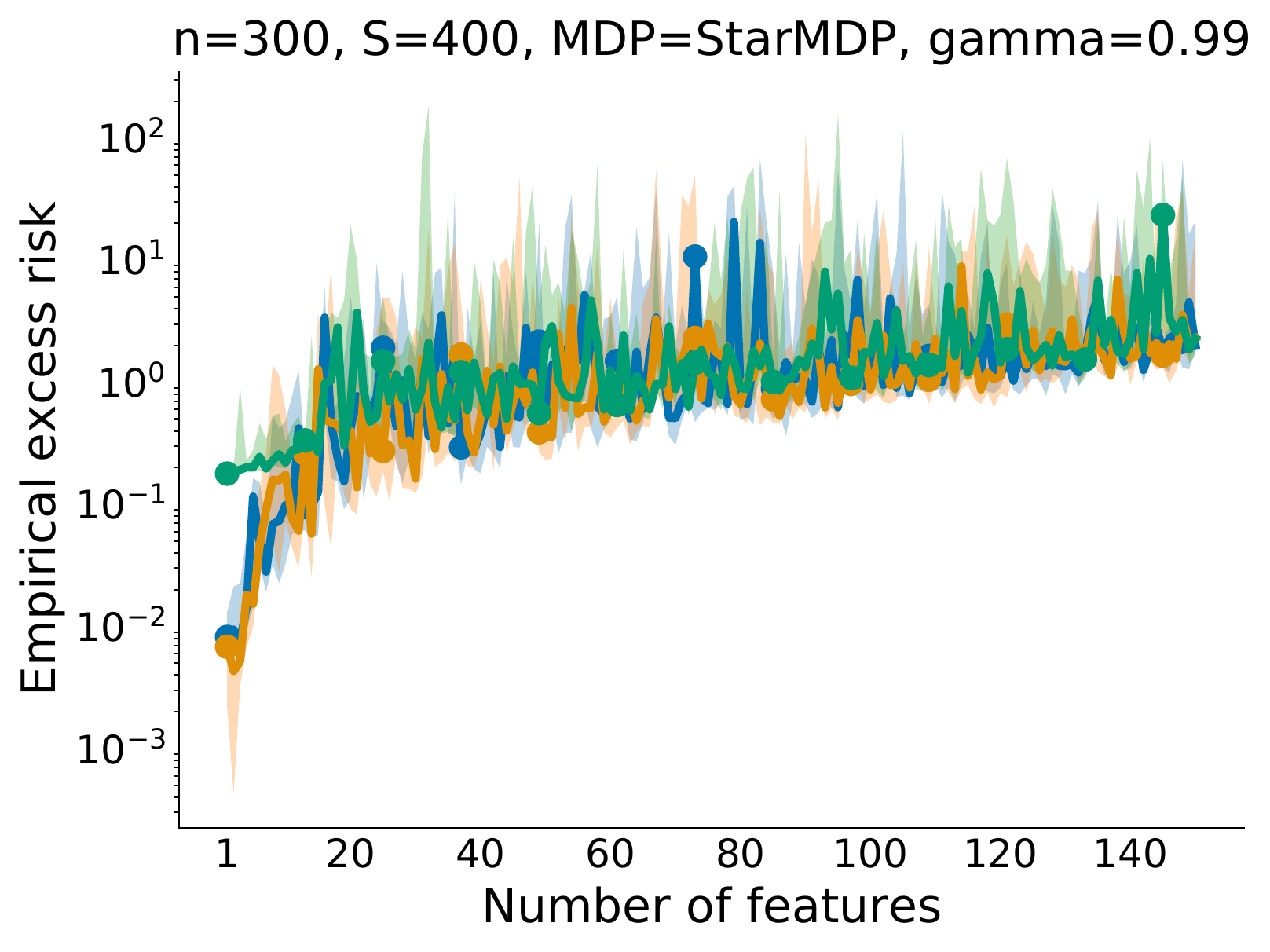}
      \includegraphics[width=0.22\textwidth]{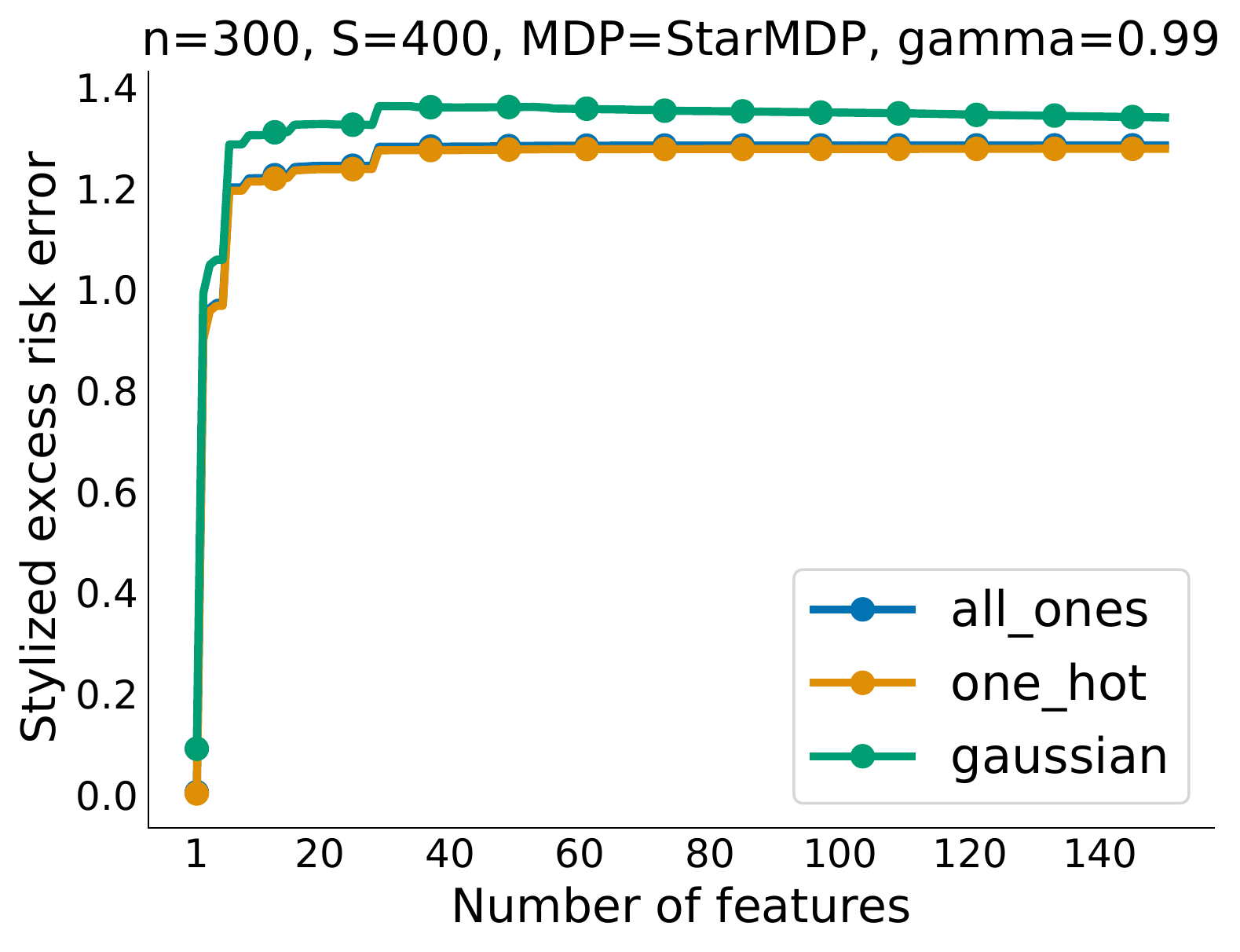}\vfill
    \includegraphics[width=0.22\textwidth]{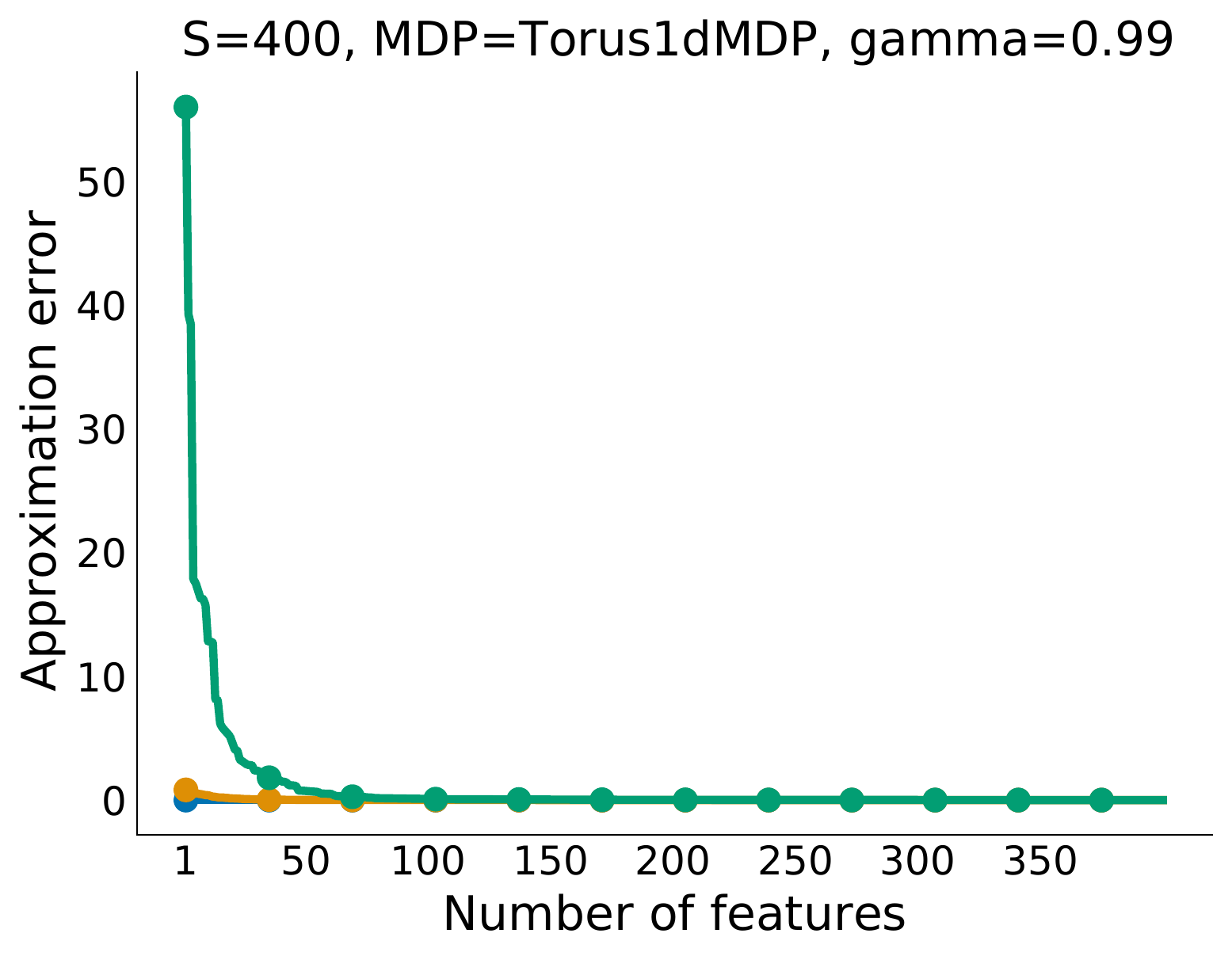}
        \includegraphics[width=0.22\textwidth]{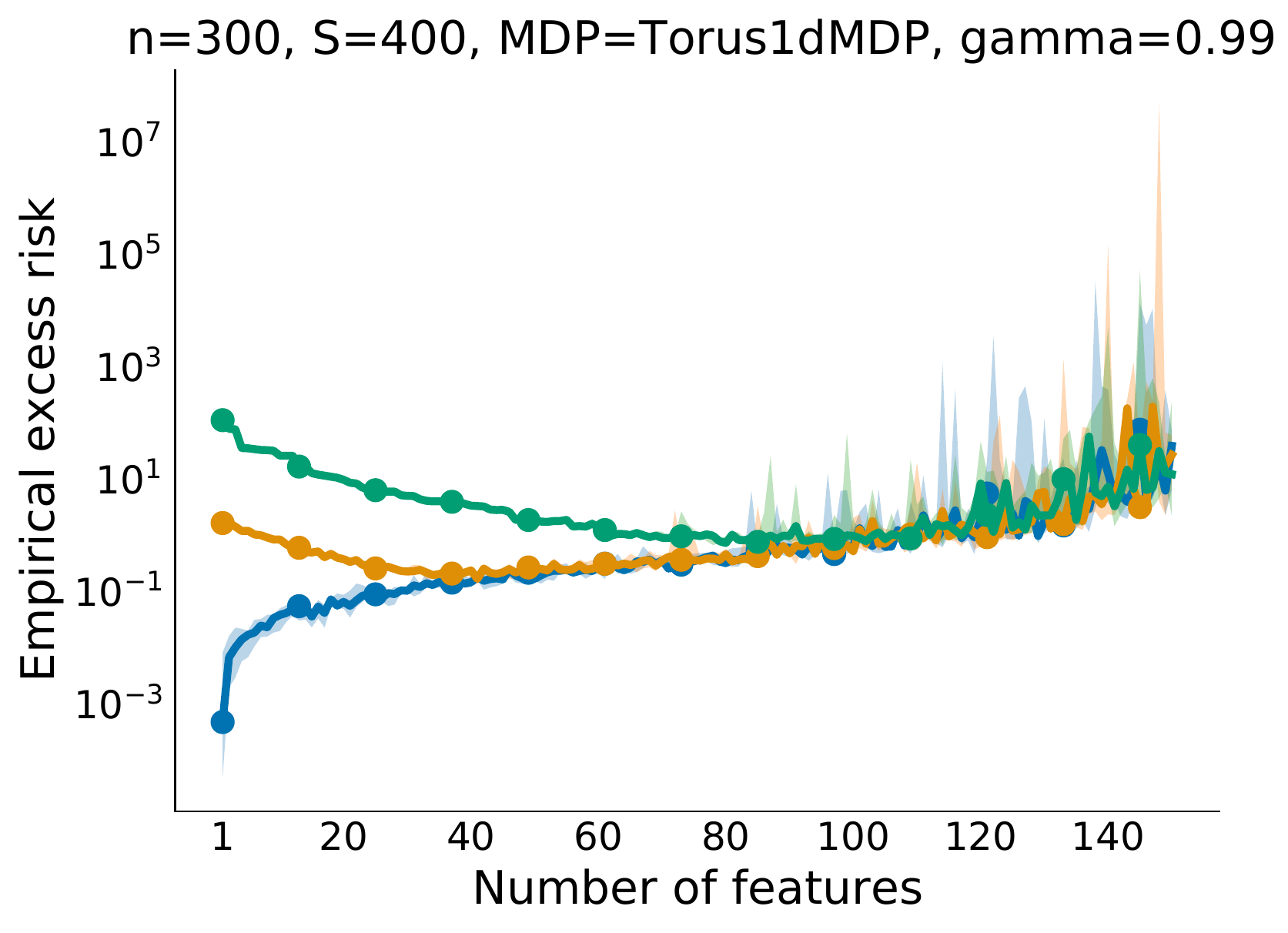}
         \includegraphics[width=0.22\textwidth]{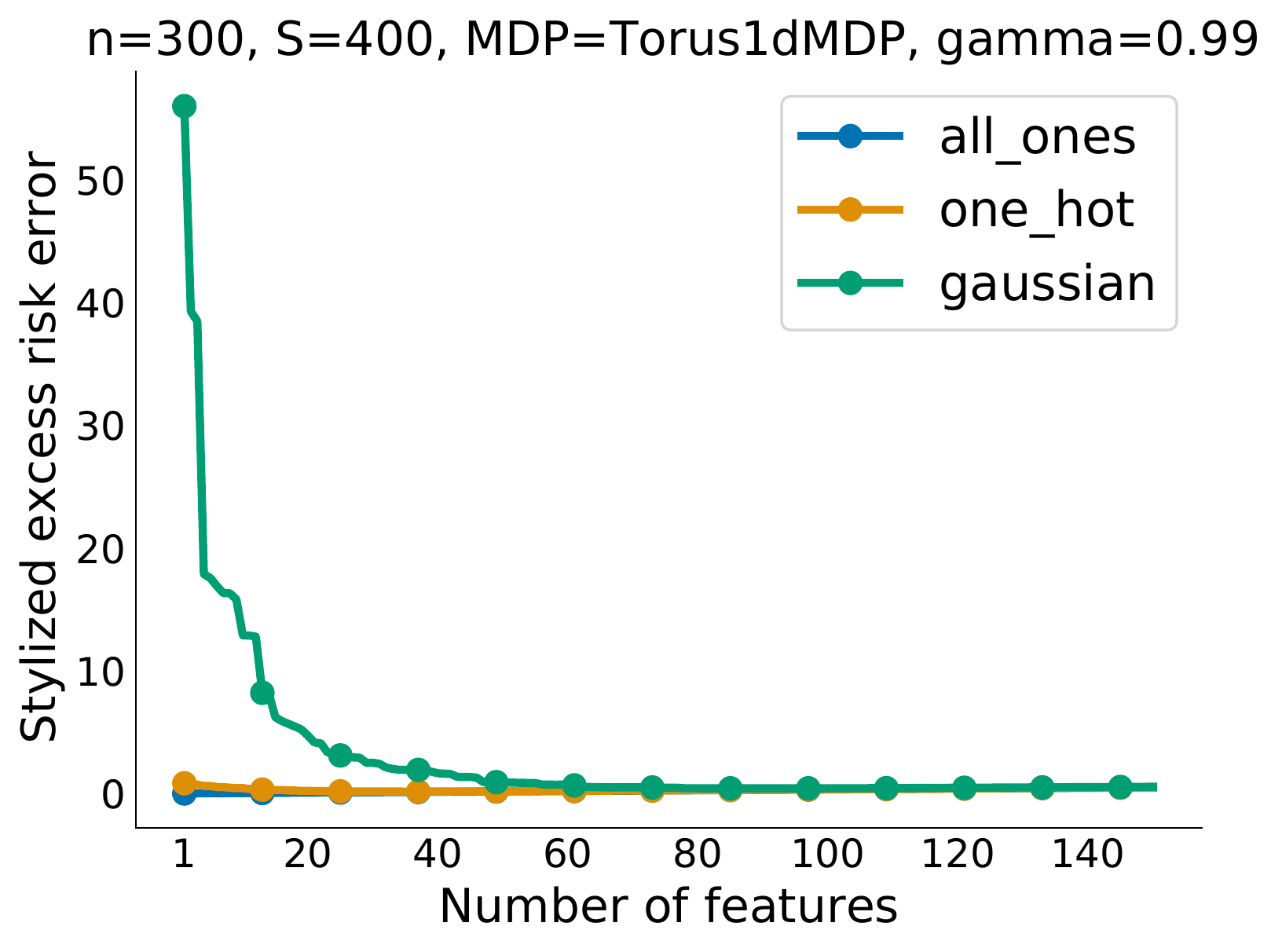}\vfill
     \includegraphics[width=0.22\textwidth]{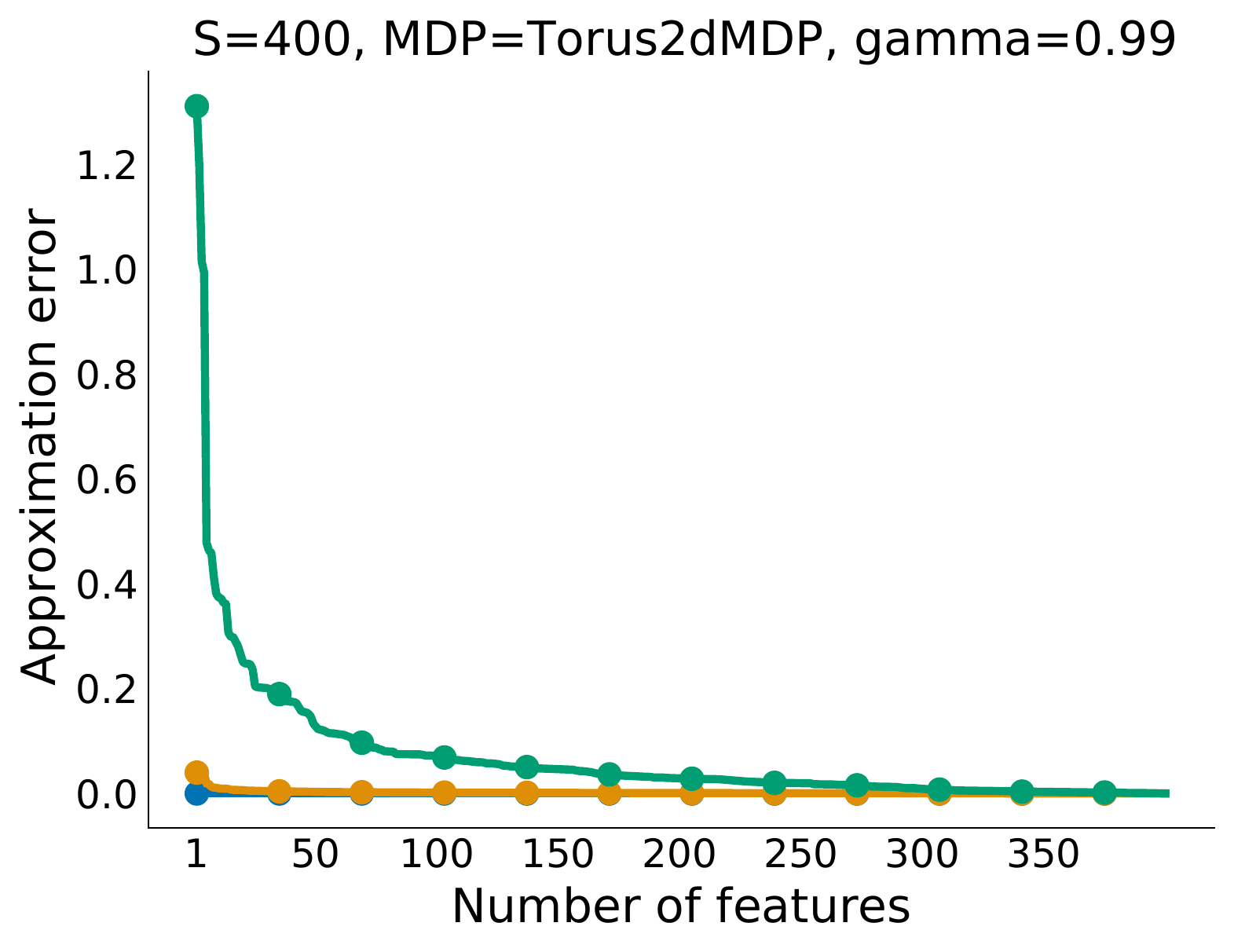}
        \includegraphics[width=0.22\textwidth]{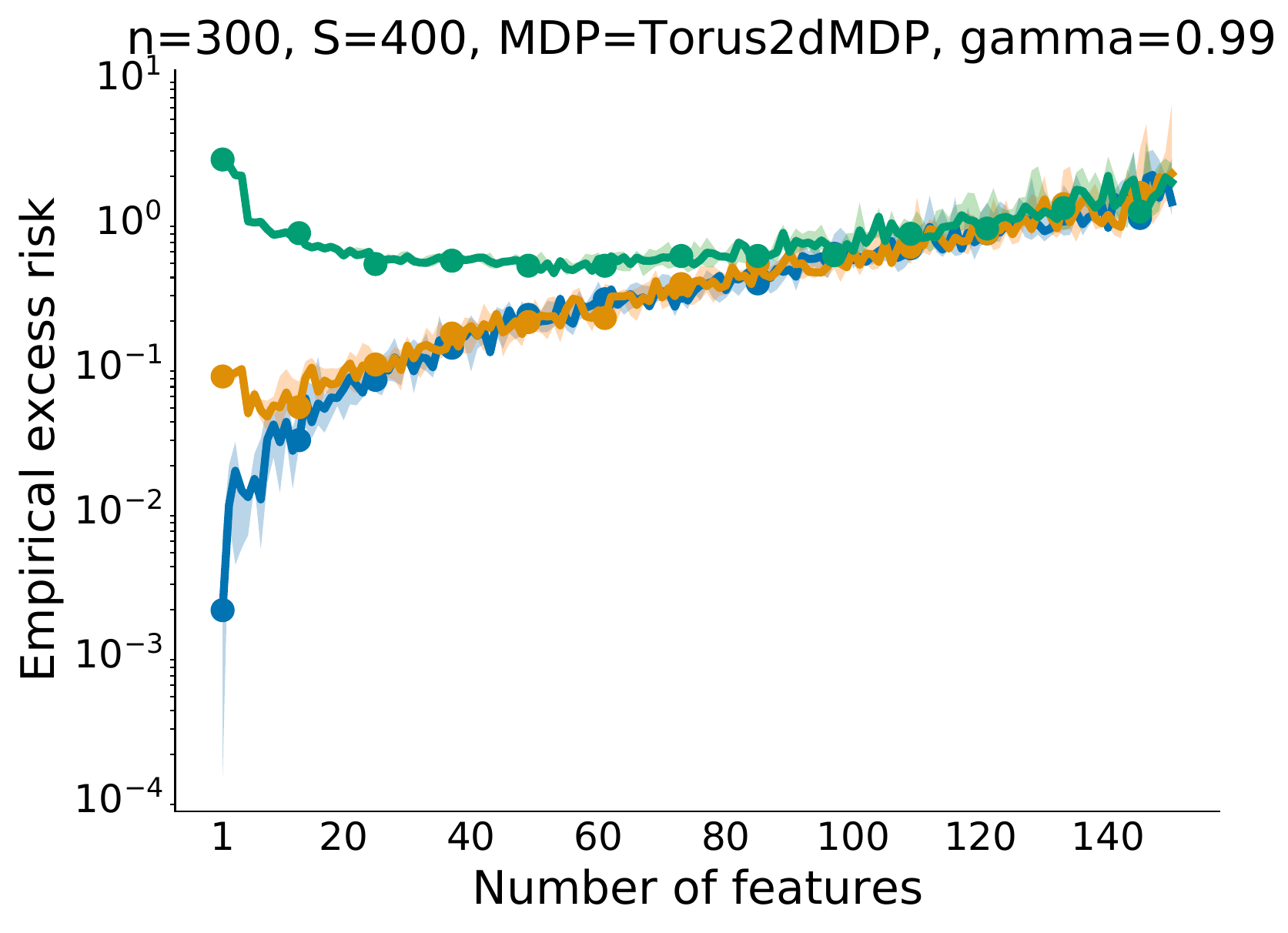}
         \includegraphics[width=0.22\textwidth]{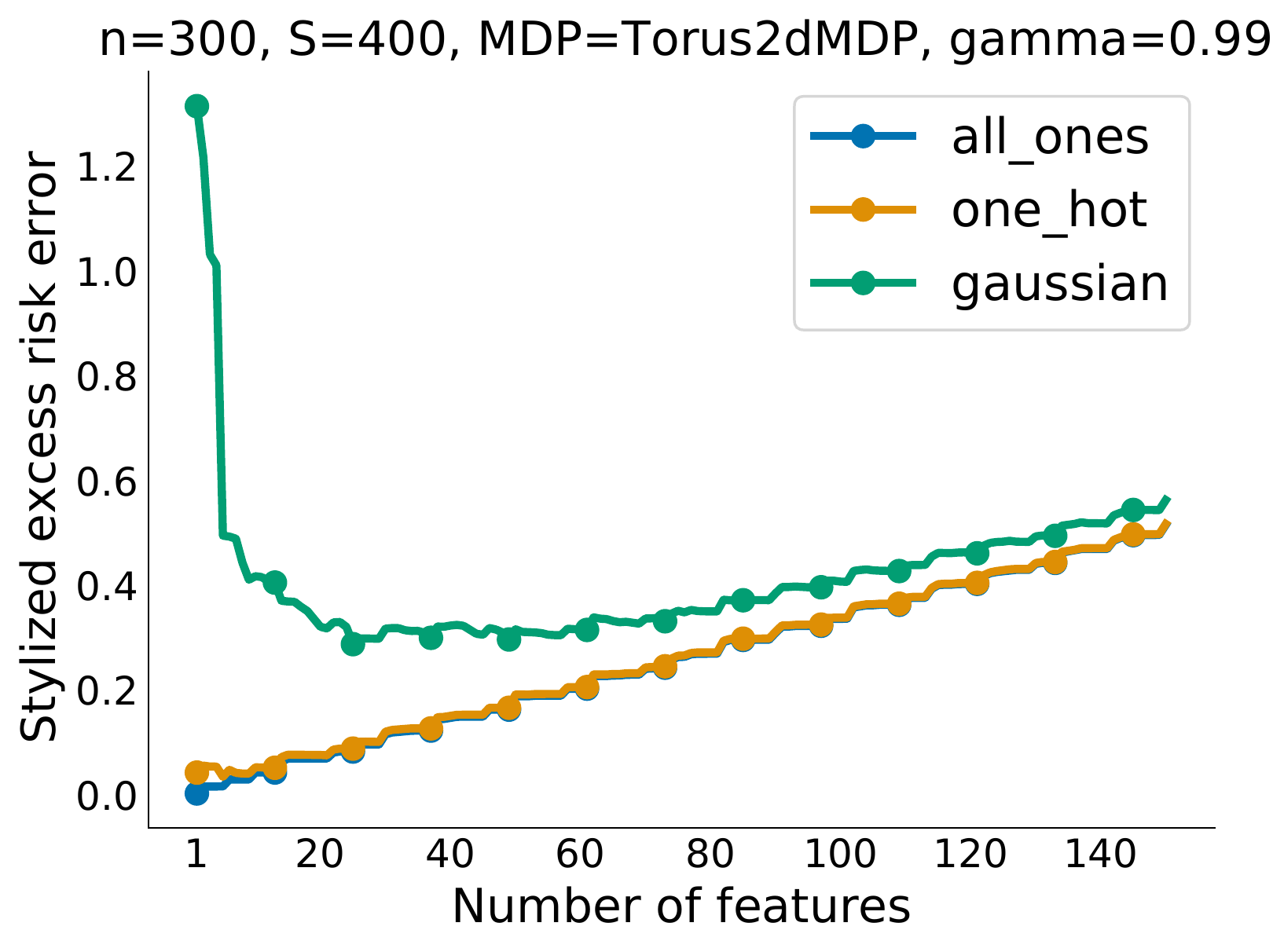}\vfill     
            \includegraphics[width=0.22\textwidth]{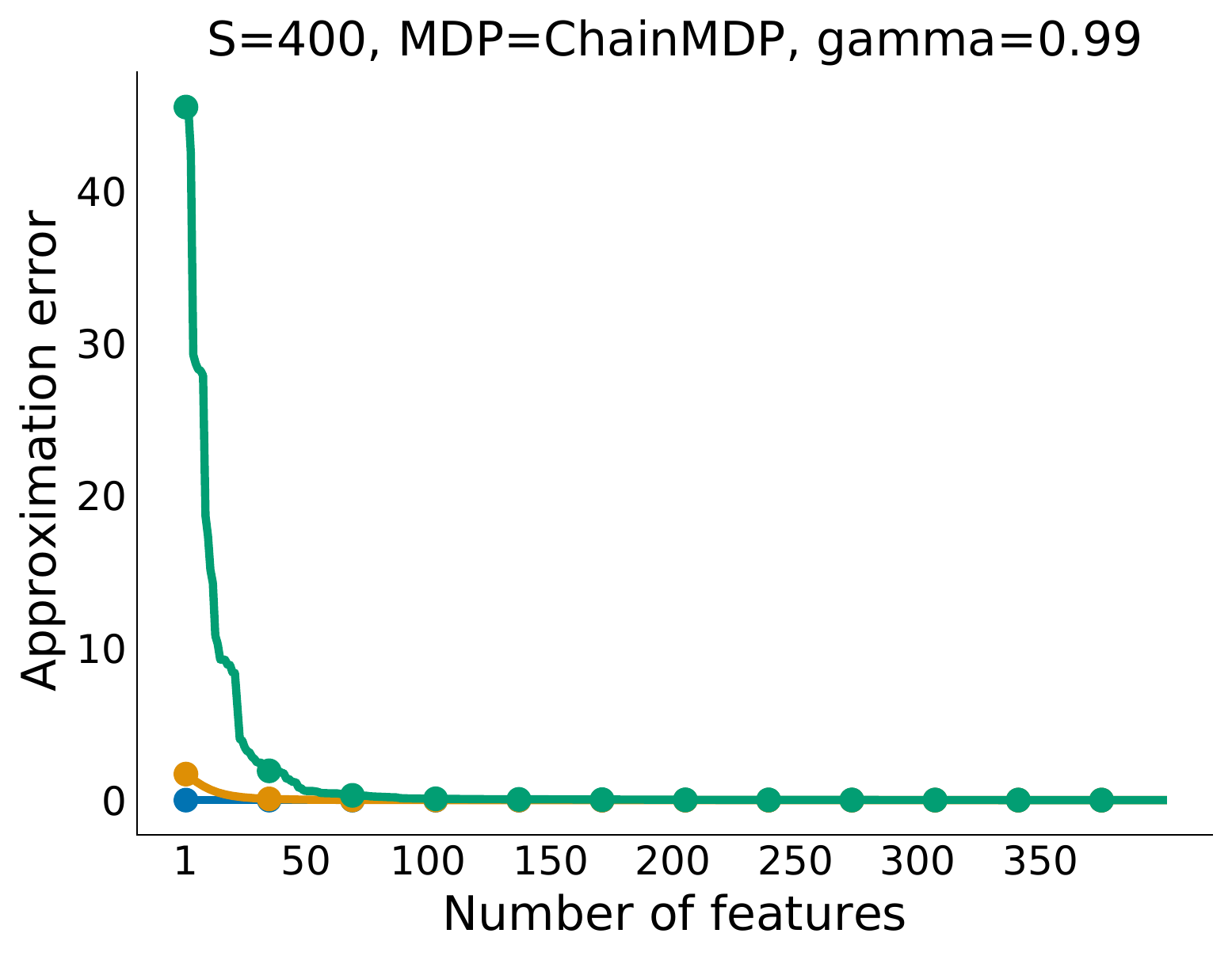}
  \includegraphics[width=0.22\textwidth]{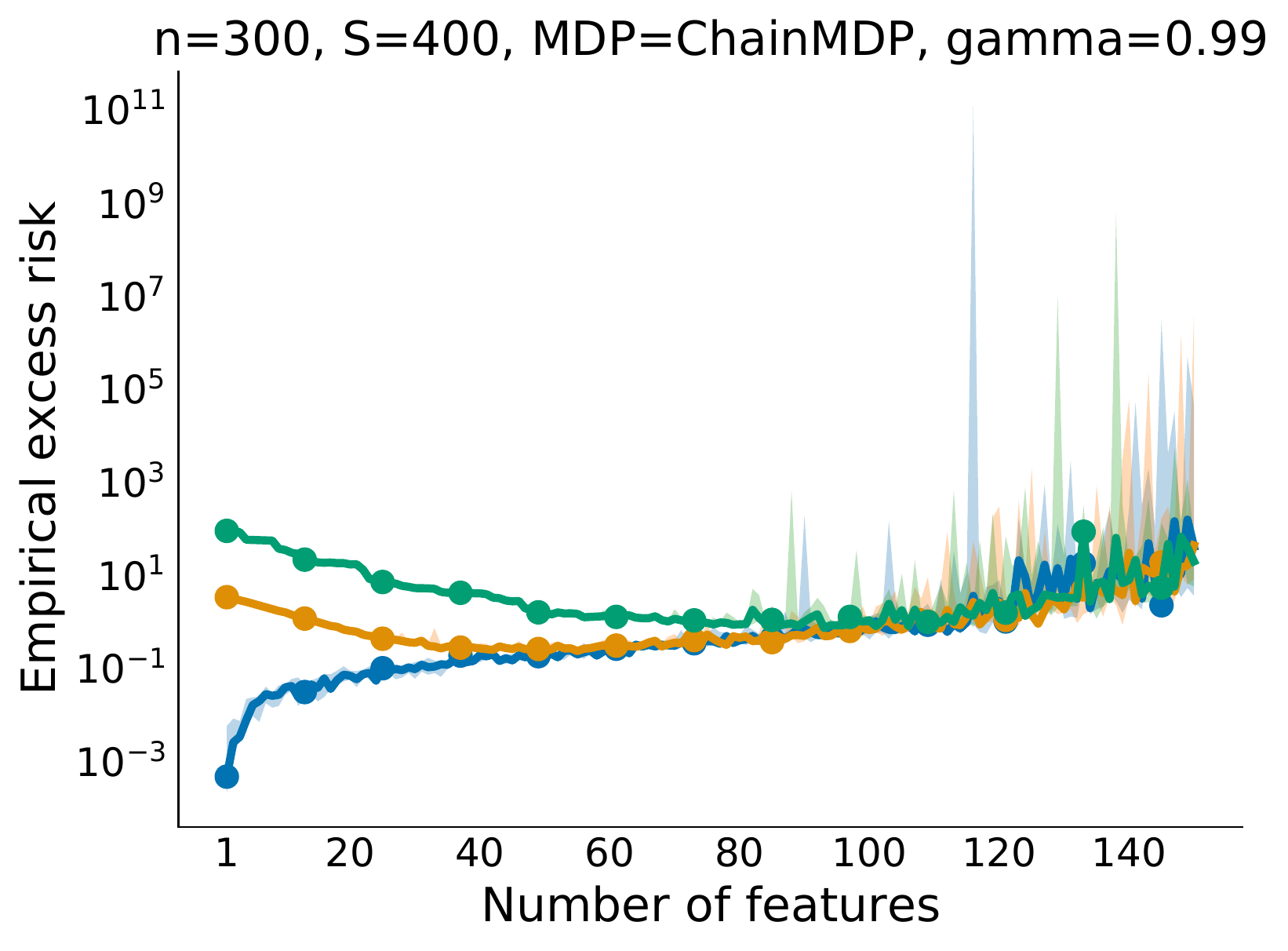}
  \includegraphics[width=0.22\textwidth]{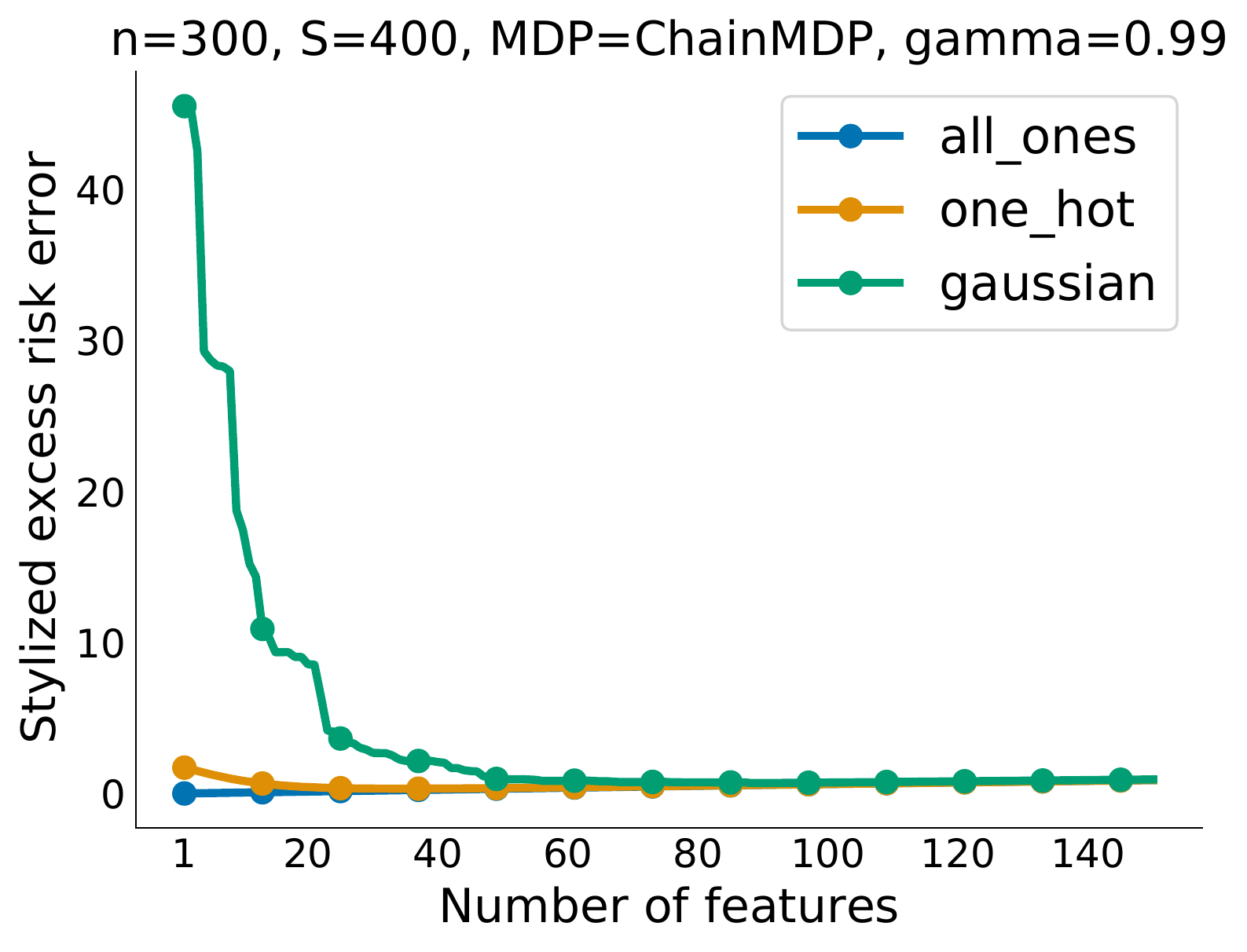} \vfill 
      \includegraphics[width=0.22\textwidth]{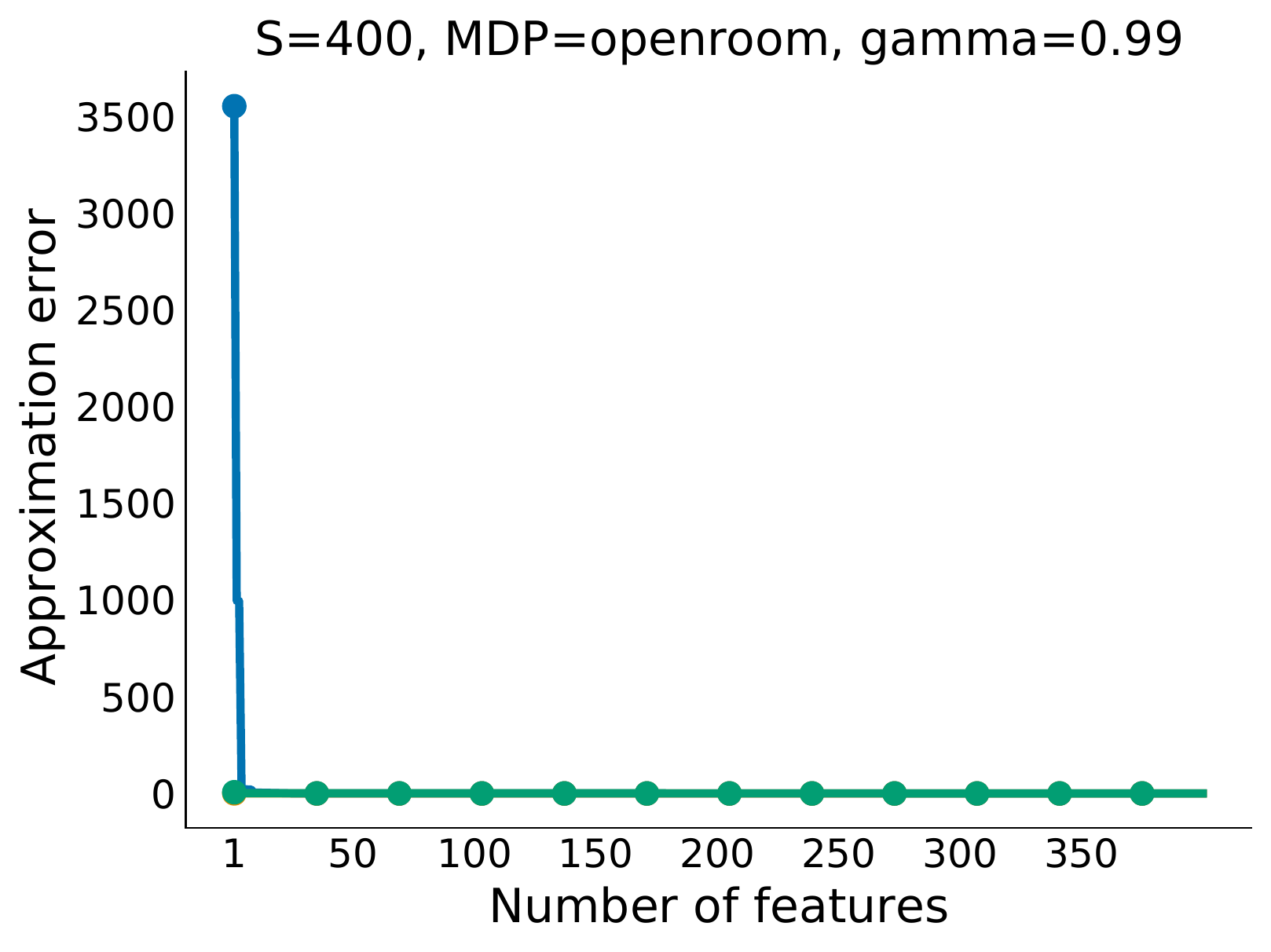}
    \includegraphics[width=0.22\textwidth]{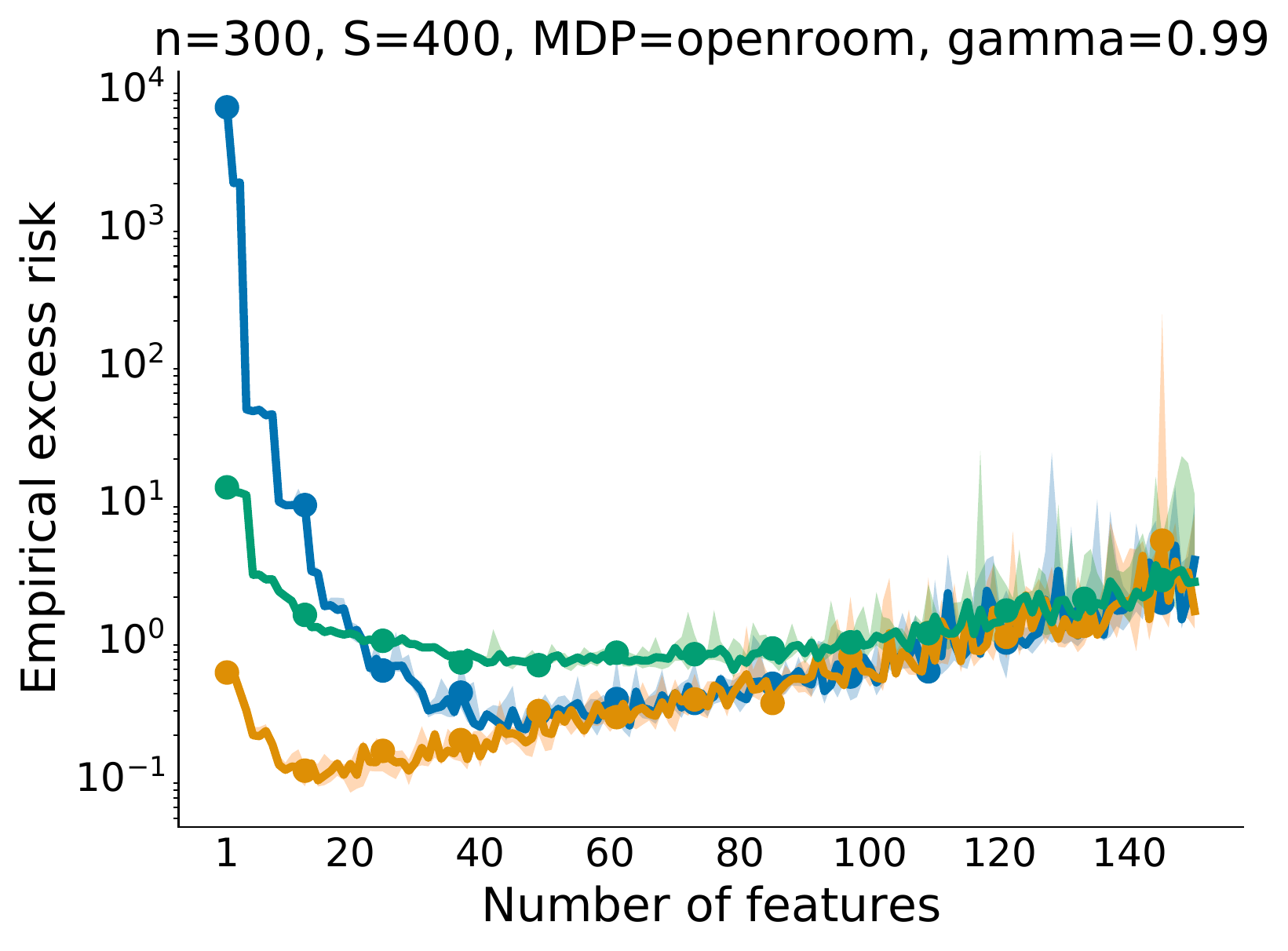}
  \includegraphics[width=0.22\textwidth]{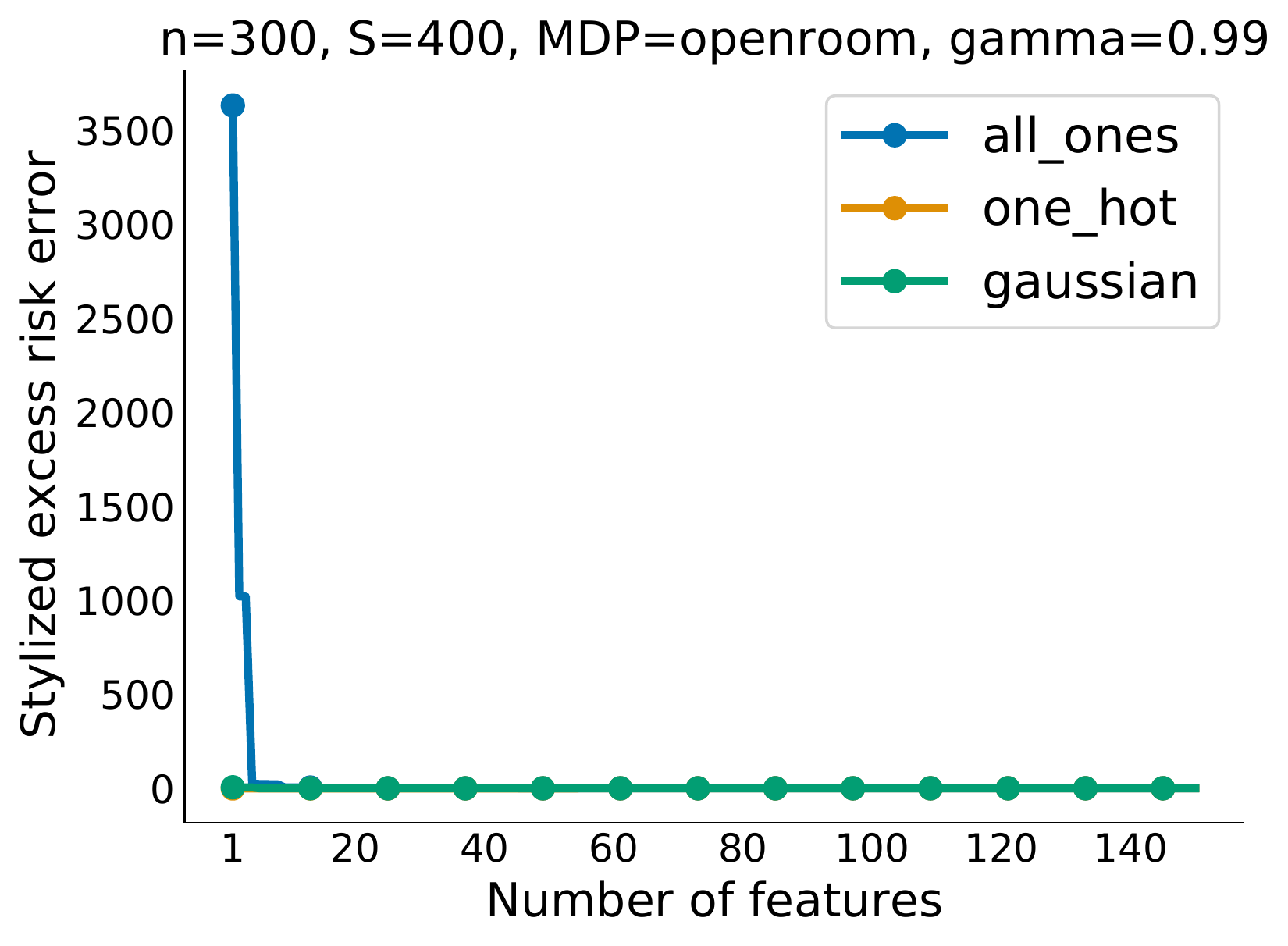} \vfill 
              \includegraphics[width=0.22\textwidth]{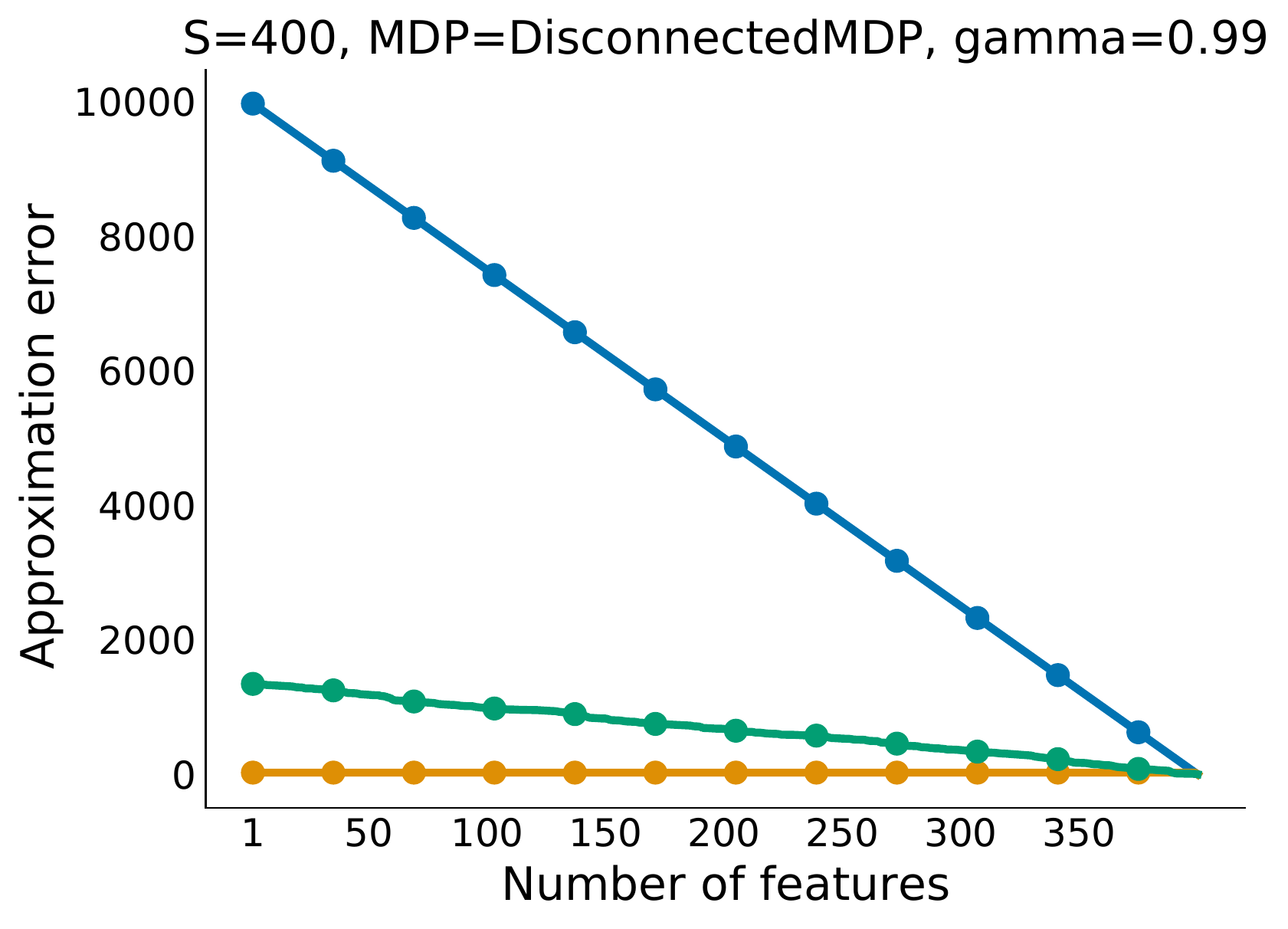}
  \includegraphics[width=0.22\textwidth]{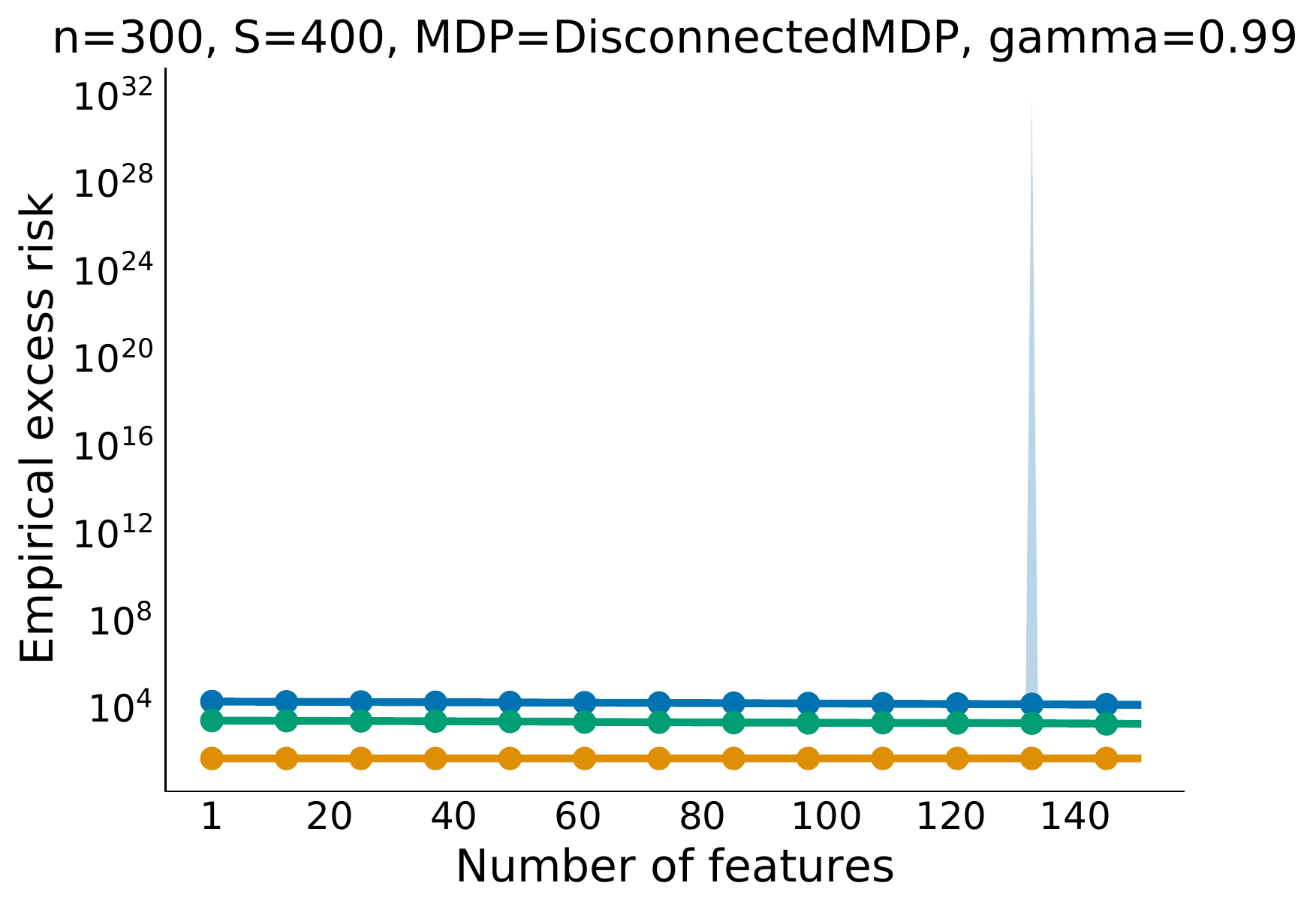}
  \includegraphics[width=0.22\textwidth]{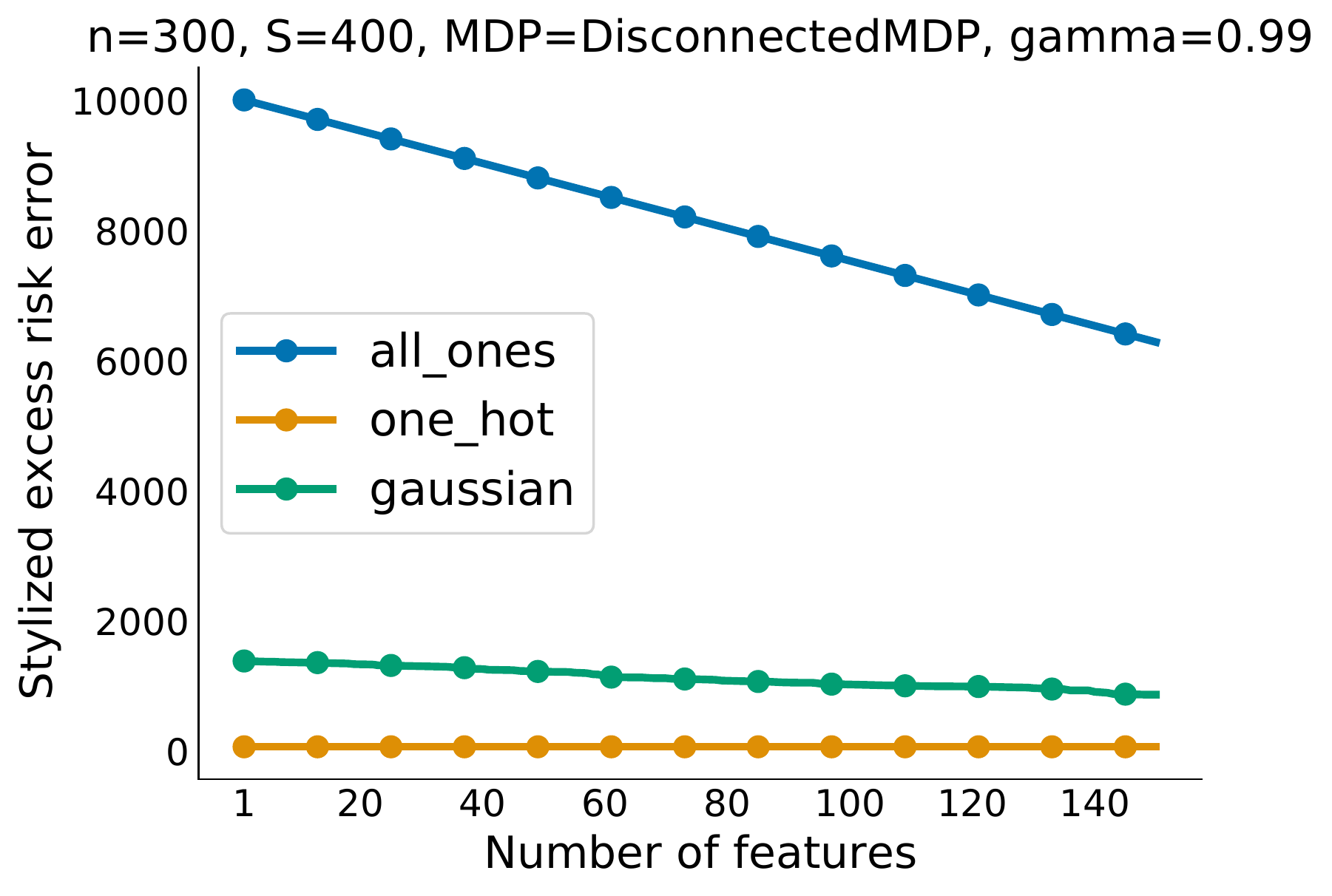} \vfill 
                \includegraphics[width=0.22\textwidth]{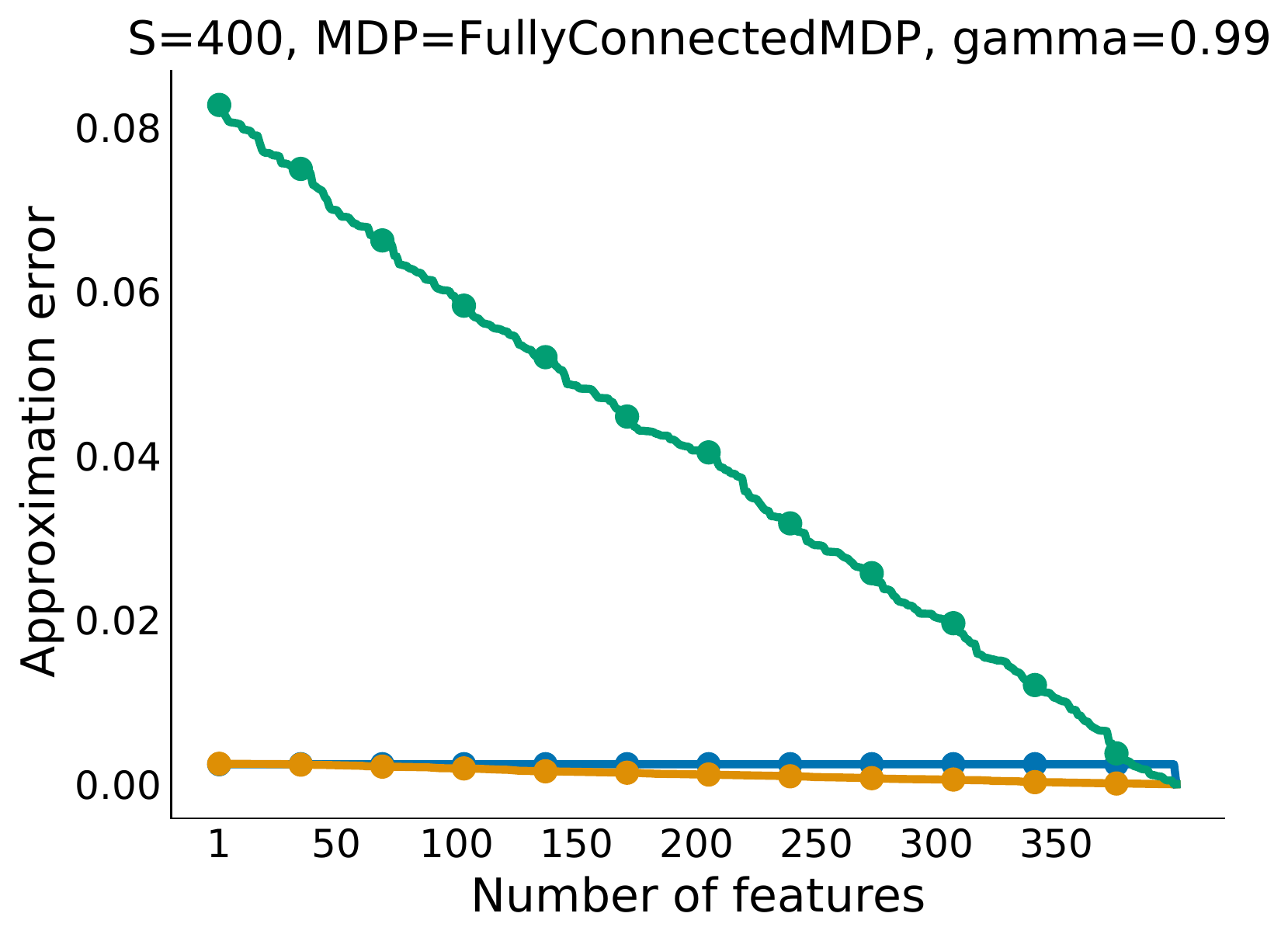}
  \includegraphics[width=0.22\textwidth]{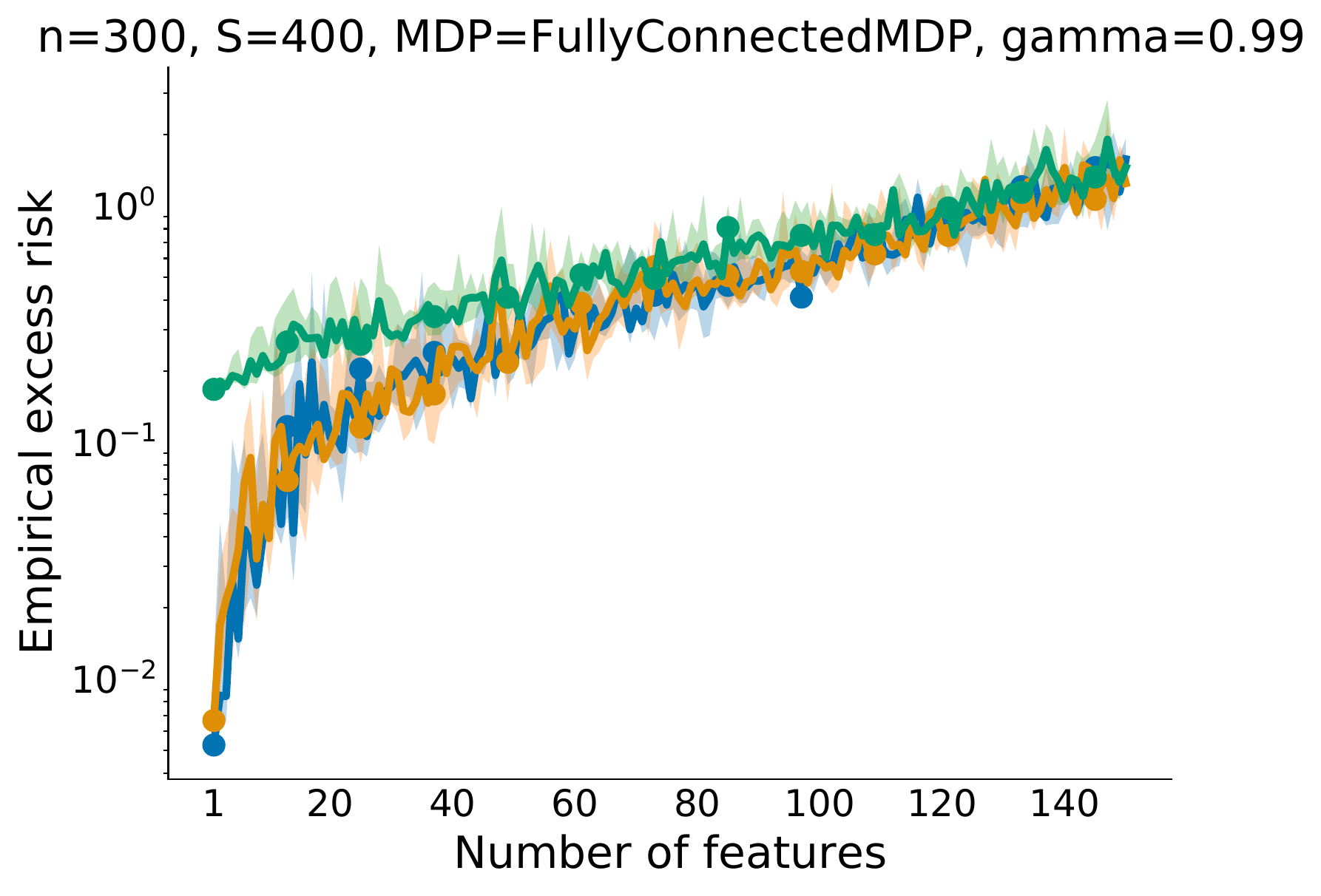}
  \includegraphics[width=0.22\textwidth]{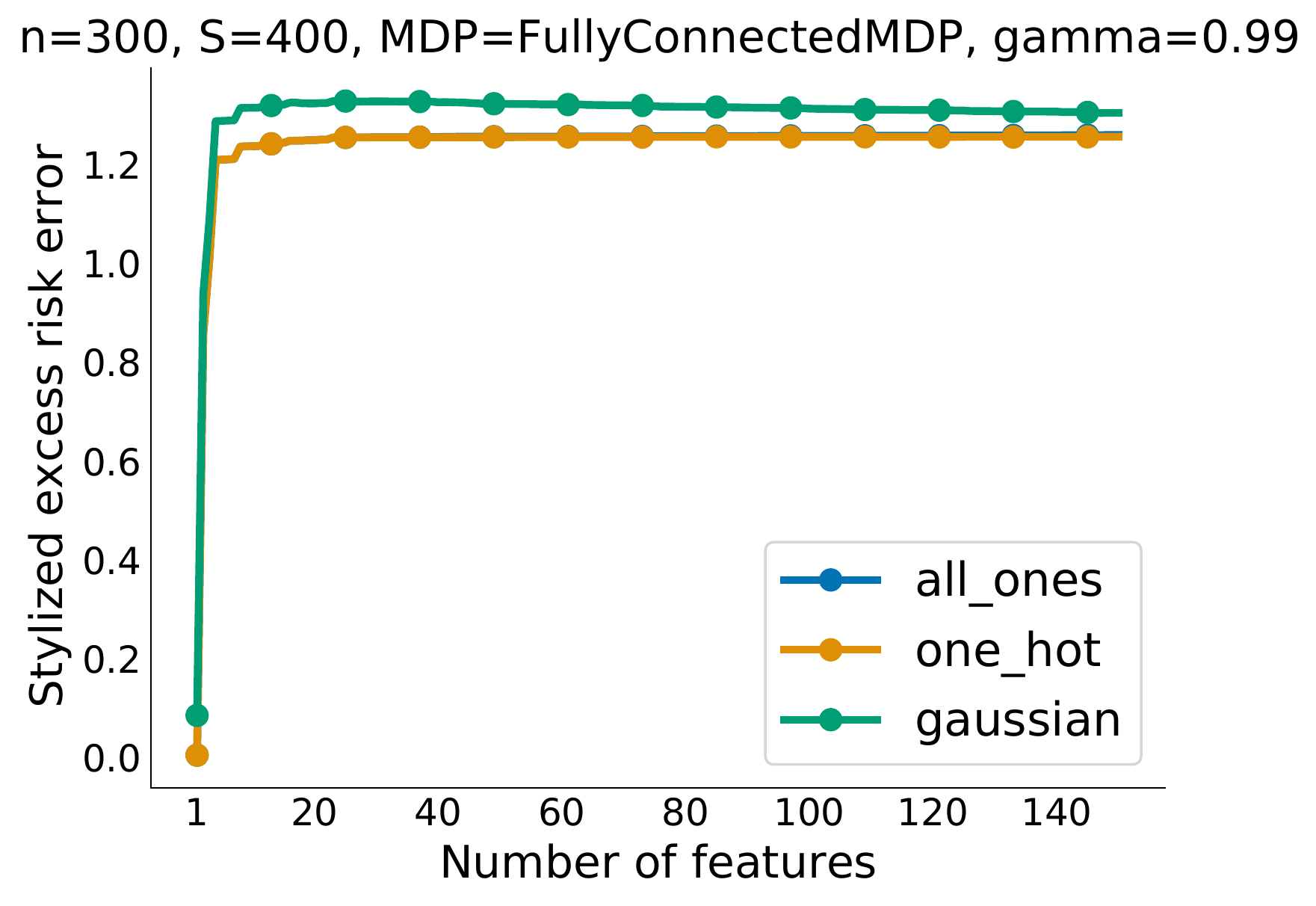} \vfill 
  \caption{\textbf{Left:} Approximation error $\|P_{F_k}V^{\pi}\|$ given a one-hot, all-ones and Gaussian reward vector and for MDPs with different graphical structures. \textbf{Middle:} Median empirical excess risk $\cE(V_{F_k, \hat{w}})$ given a one-hot, all-ones and Gaussian reward vector. \textbf{Right} Theoretical excess risk for a representation $\Phi_k=F_k$ and a one-hot, all-ones and Gaussian reward vector. The median is over 5 random seeds and shading gives $95\%$  confidence intervals.
  }
  \label{fig:toyplots}
\end{figure}
In this section, we study the generalization characteristics of the representations induced by the SVD of the successor representation for several environment transition structures. We illustrate the different graphs over which we define a random walk, studied in \cref{sec:effect-structure} as well as some new ones, in \cref{fig:graph:illustrations}.

Our experiment consists in evaluating the value function on these different transition structures when $S=400$ states. We consider three different reward vectors $r_\pi \in \rR^S$: the all ones vector, the one-hot feature vector $e_S$, and a vector whose entries are drawn from zero-mean Gaussian distribution and normalized such that $\norm{r_\pi}_{\infty}=1$.
We then sampled a dataset $D$ of $n=300$ pairs $(s_i, y_i)$ where we performed a Monte Carlo rollout to obtain the returns $(y_i)_{i=1}^n$. 
The targets are the value functions induced by the random walk.

We are interested in comparing our generalization bound to the empirical excess risk on these domains. Our bound looks at the regime $ n \geq \deff(F_k)$. We choose $k \leq \frac{n}{2}$ as an heuristic way of achieving this.
We report in \cref{fig:toyplots} the approximation error (\cref{fig:toyplots} Left), the empirical excess risk (\cref{fig:toyplots} Middle) and the theoretical excess risk (\cref{fig:toyplots} Right) obtained when using the representation $\phi=F_k$ on these different graph structures.

\textbf{Star}: Baird's star graph \citep{baird1995residual} consists in $S-1$ states which are the star corners and a state $S$ which is the star center. A random walk on this star graph induces a transition function such that all star corners transition to the star center and the star center goes to the star corners.
There are two extreme cases in terms of rewards: either the reward is the same for all $s_i$, $i \neq S$, (e.g. the all ones reward vector or the one-hot vector $e_S$) or not. If the reward is the same for all $s_i, i \neq S$, then this is effectively a 2 state structure, so we only really need 1 feature to distinguish between the value of the star corners and the value of the star center.
However, if the reward is different for all $s_i(i \neq S)$ then we effectively have $(S-1)$ tuples $(s_i, s_S)$ which can be thought of as independent graphical structures and we thus expect to need all the features to distinguish between their values. We can see this in \cref{fig:toyplots} that for the all ones reward vector and the one-hot reward vector $e_S$, the error with $k=1$ is very good but for the Gaussian reward, the error with $k=1$ is high.

\textbf{Chain}: This is a $S$-state connected graph with 2 pendant states and $(n-2)$ states of degree two. The shapes of the curves are similar to the Torus1d but we can notice that the errors are larger for each feature dimension $k$. This is intuitive as for instance in the case of an all ones reward vector, the values are not the same for each state due to the two end states of the chain, implying that more than one feature is needed to generalize the value function.

\textbf{Openroom}: This is a two-dimensional grid with $S$ states. States strictly inside the grid have four neighbours. States belonging to one (reps. two) edges are of degree three (resp. two). As we observed in \cref{fig:effectivedim_toymdps}, the Openroom domain does not generalizes as well as the Torus2d which can be explained by their difference in effective dimension.

\textbf{Torus1d}: This is a wrap-around version of the Chain. State $i$ transitions to state $(i+1)$ mod $S$ and state $(i-1)$ mod $S$. We can see that the curve showing the empirical excess risk (Middle) corresponding to the Gaussian reward vector has a sweet spot which is also predicted by our theory. Moreover, when all states have the same reward, their values are identical. Hence, in that case, only one feature is enough to have very low error which is shown both empirically and by our theoretical bound on \cref{fig:toyplots}.

\textbf{Torus2d}: It is a wrap-around version of the Openroom domain such that each state has four different neighbors. We can see in \cref{fig:effectivedim_toymdps} that the Torus1d and Torus2d have similar effective dimension but the decay of the singular values is faster in the case of Torus2d translating into smaller approximation errors in \cref{fig:toyplots} (Middle). This results in overall lower excess risk for the Torus2d indicating it generalizes in general better than its one-dimensional counterpart. Just like for the Torus1d, in the case of the Gaussian reward vector, there is a non trivial optimal number of features $k$ minimizing the excess risk, which we can notice is smaller than for the Torus1d. 

\textbf{Disconnected}: This graph consists of $S$ states that self-transition. We do not expect the successor representation to generalize well within this MDP as we cannot leverage knowledge from one feature state to another. This idea was already captured by the effective dimension shown in \cref{fig:effectivedim_toymdps}. The plots in \cref{fig:toyplots} corroborates this both empirically and theoretically showing that its excess risk is indeed the highest across all transition structures considered.

\textbf{Fullyconnected}: This is a connected graph of $S$ states where each state can transition to $(S-1)$ states. The first singular vector, which is the constant vector, is very good in terms of effective dimension but the second vector has high effective dimension. When the rewards are the same in each state, their values are identical. In that case, one feature is enough to distinguish between the $S$ states leading to good generalization in that case. Additional features must be misleading as the excess risks rises significantly from a number of features $k=2$.

\subsection{Full Atari Results}
\label{app:fullatariresults}
For all experiments, we used the hyperparameters provided by Dopamine \citep{castro18dopamine}.

\textbf{Compute.} For our experiments on Atari, we used Tesla V100 GPUs and P100 for all runs. To obtain the pretrained deep representations for each deep RL agent, we ran a total of 5 runs / game $\times$ 60 games / algorithm $\times$ 5 algorithms $=1500$ runs. Each of these runs takes around 5 days. Additionally, for the auxiliary loss experiment, we ran a total of 5 runs / game $\times$ 5 games / algorithm $\times$ 2 algorithms $=50$ runs. In this setting, each run takes around 1 day. Overall, the amount of compute is of 7050 days of GPU training.

We provide a per-game comparison of the effective dimension of the representations induced by DQN, DQN (Adam), Rainbow, IQN and M-IQN throughout training in \cref{fig:effdim_allgames} for all 60 Atari games in the online setting to complement the results presented in \cref{fig:aggregate_atari} in the main part of the paper.

For the offline experiment presented in \cref{fig:atari_offline}, we use the same mini-batch sampled for the temporal-difference loss $\mathcal{L}_{\text{TD}}$ for computing the auxiliary loss $\mathcal{L}_\phi$. Our combined loss is then $\mathcal{L}_{\alpha} = (1-\alpha)\mathcal{L}_{\text{TD}} + \alpha \mathcal{L}_\phi$. We ran a hyperparameter sweep over $\alpha$ on the five games displayed in \cref{fig:sweep} and found that a value of $\alpha = 0.1$ worked well. We provide per-game training curves for IQN agents for 17 Atari games in \cref{fig:atari_offline_pergame} as well as the effective dimension (see \cref{fig:atari_offline_pergame_effdim}) of their induced representations computed with a batch size of $2^{15}$. We also complement these results with the rank of these representations as a function of training in \cref{fig:atari_offline_pergame_rank} and \cref{fig:atari_offline_IQM_rank} as a proxy for the approximation error.

\begin{figure*}[h!]
  \centering
   \includegraphics[width=\textwidth]{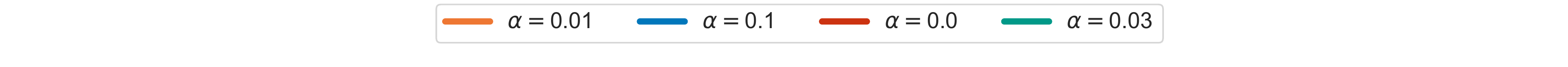}
     \includegraphics[width=\textwidth]{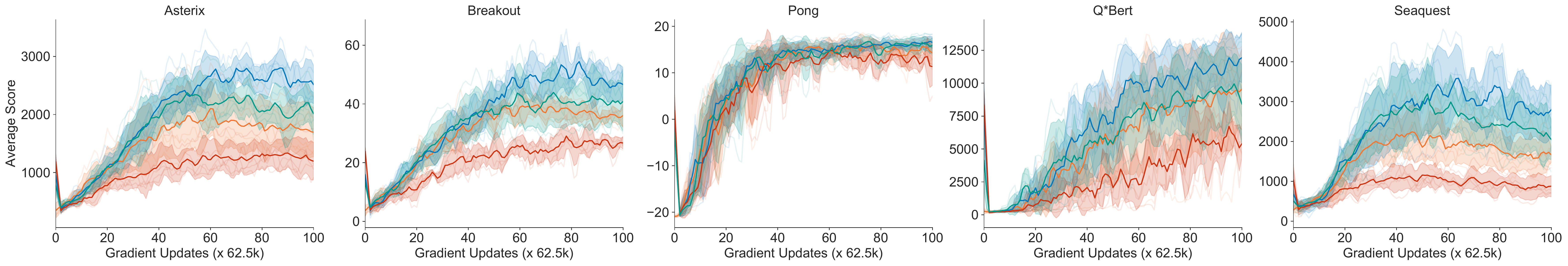}
  \caption{Sweeping over various values of $\alpha$ when adding the auxiliary loss $\mathcal{L}_\phi$ to IQN.}
  \label{fig:sweep}
\end{figure*}

\begin{figure*}[t!]
  \centering
   \includegraphics[width=0.7\textwidth]{images/legend.pdf}
  \includegraphics[width=0.70\textwidth]{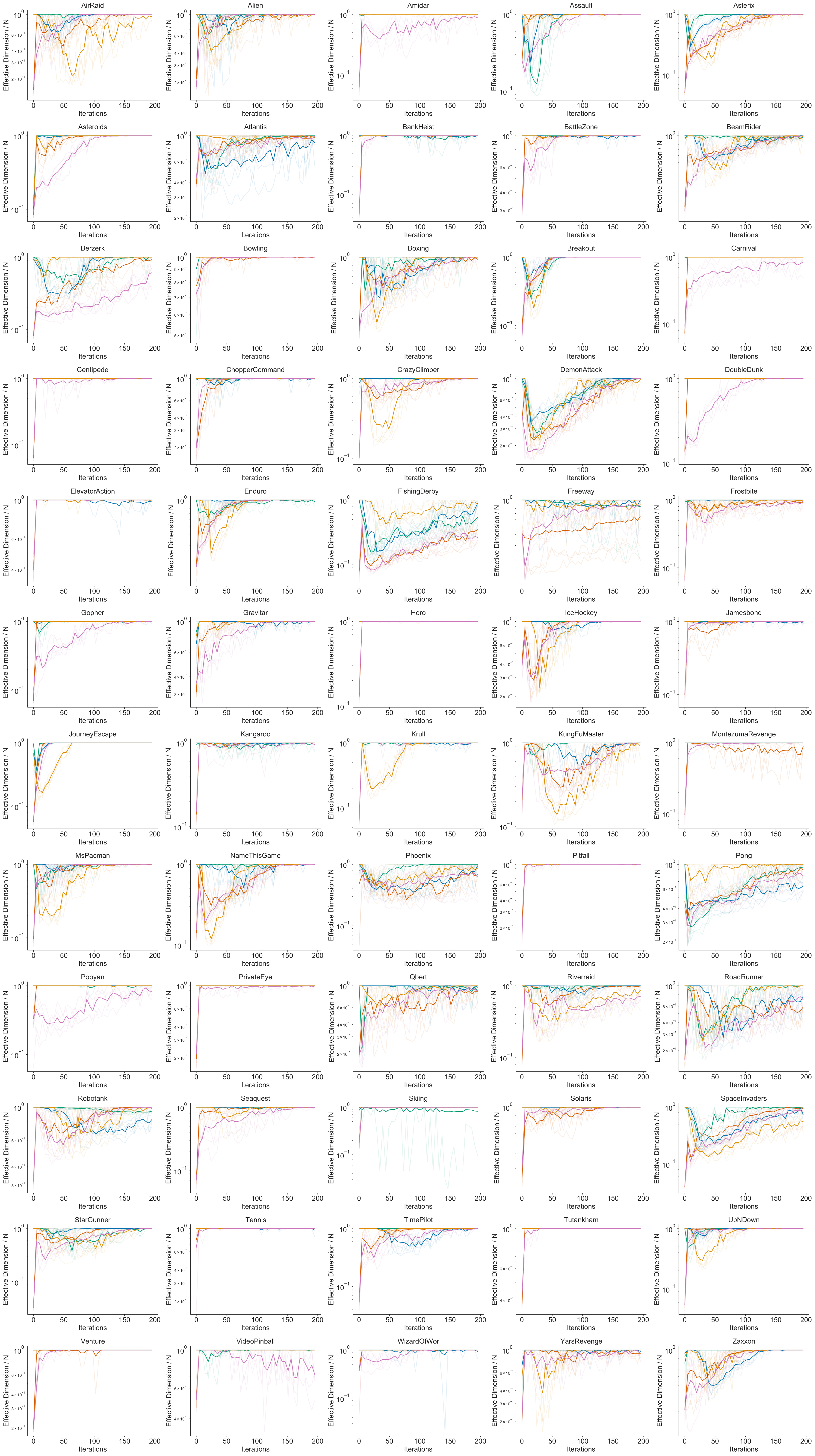}
  \caption{Average estimate (darker color) of the effective dimension normalized by the batch size used $N=2^{15}$ on DQN(Nature), DQN(Adam), Rainbow, IQN and M-IQN on all 60 Atari games computed using 5 independent runs. Individual runs are shown with a lighter color.}
  \label{fig:effdim_allgames}
\end{figure*}

\begin{figure*}[h!]
  \centering
     \includegraphics[width=0.85\textwidth]{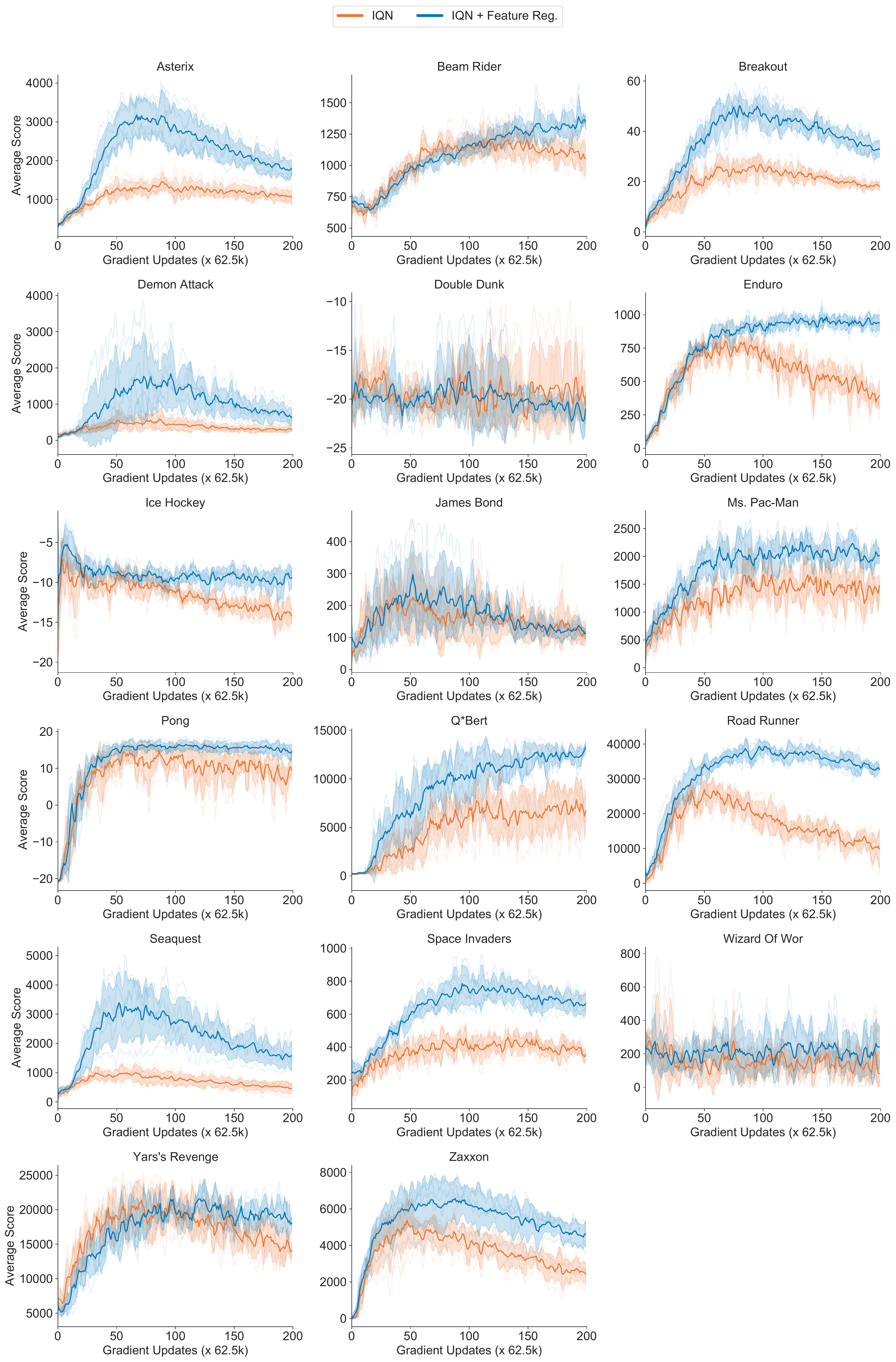}
  \caption{Per-game learning curves of IQN and IQN with feature regularization $L_\phi$ on 17 Atari games in the offline RL setting.}
  \label{fig:atari_offline_pergame}
\end{figure*}

\begin{figure*}[h!]
  \centering
     \includegraphics[width=0.85\textwidth]{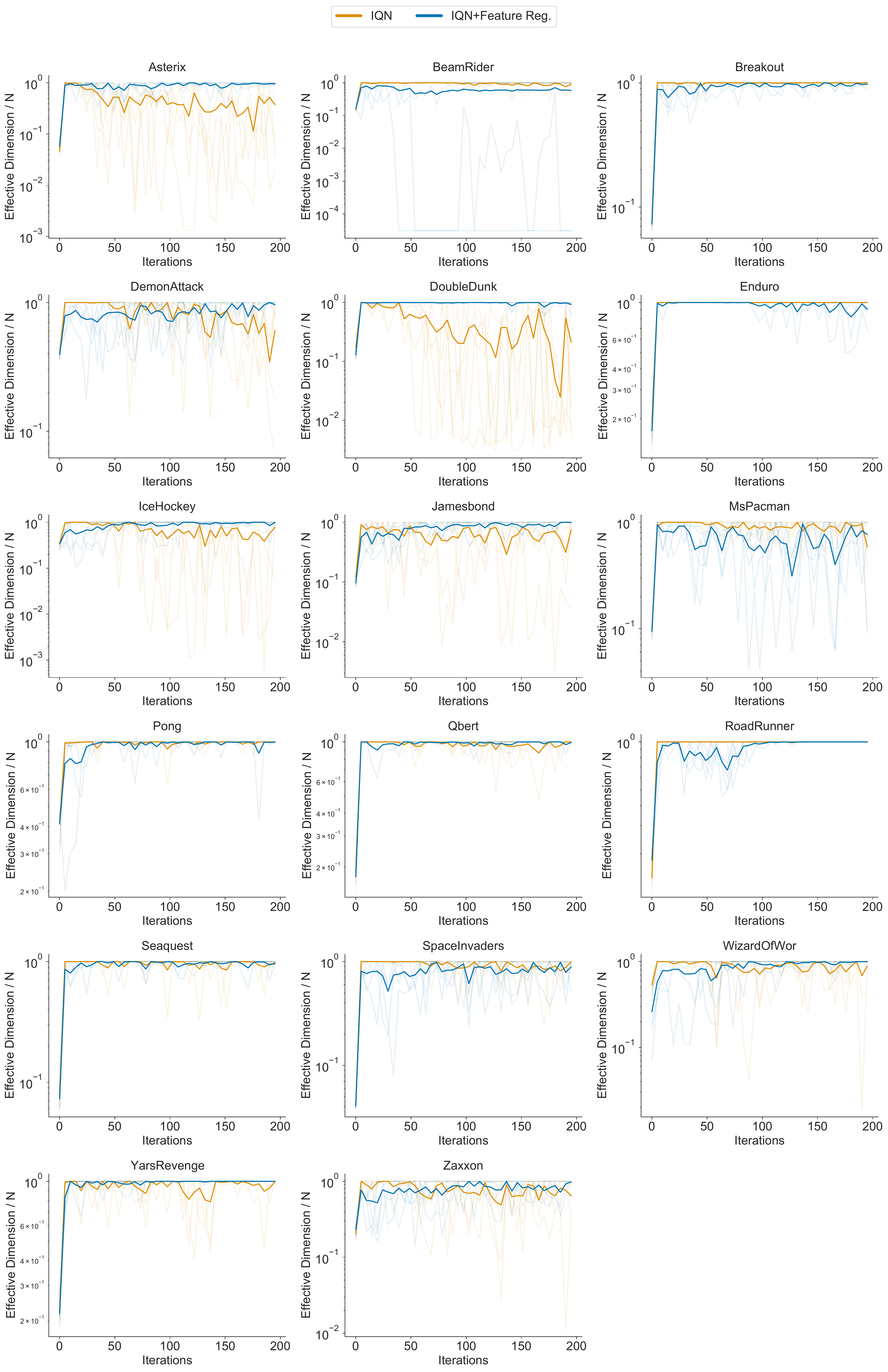}
  \caption{Per-game effective dimension normalized by the batch size $N=2^{15}$ of IQN and IQN with feature regularization $L_\phi$ on 17 Atari games in the offline RL setting, using 5 independent runs. Individual runs are shown with a lighter color.}
  \label{fig:atari_offline_pergame_effdim}
\end{figure*}

\begin{figure*}[h!]
  \centering
     \includegraphics[width=0.85\textwidth]{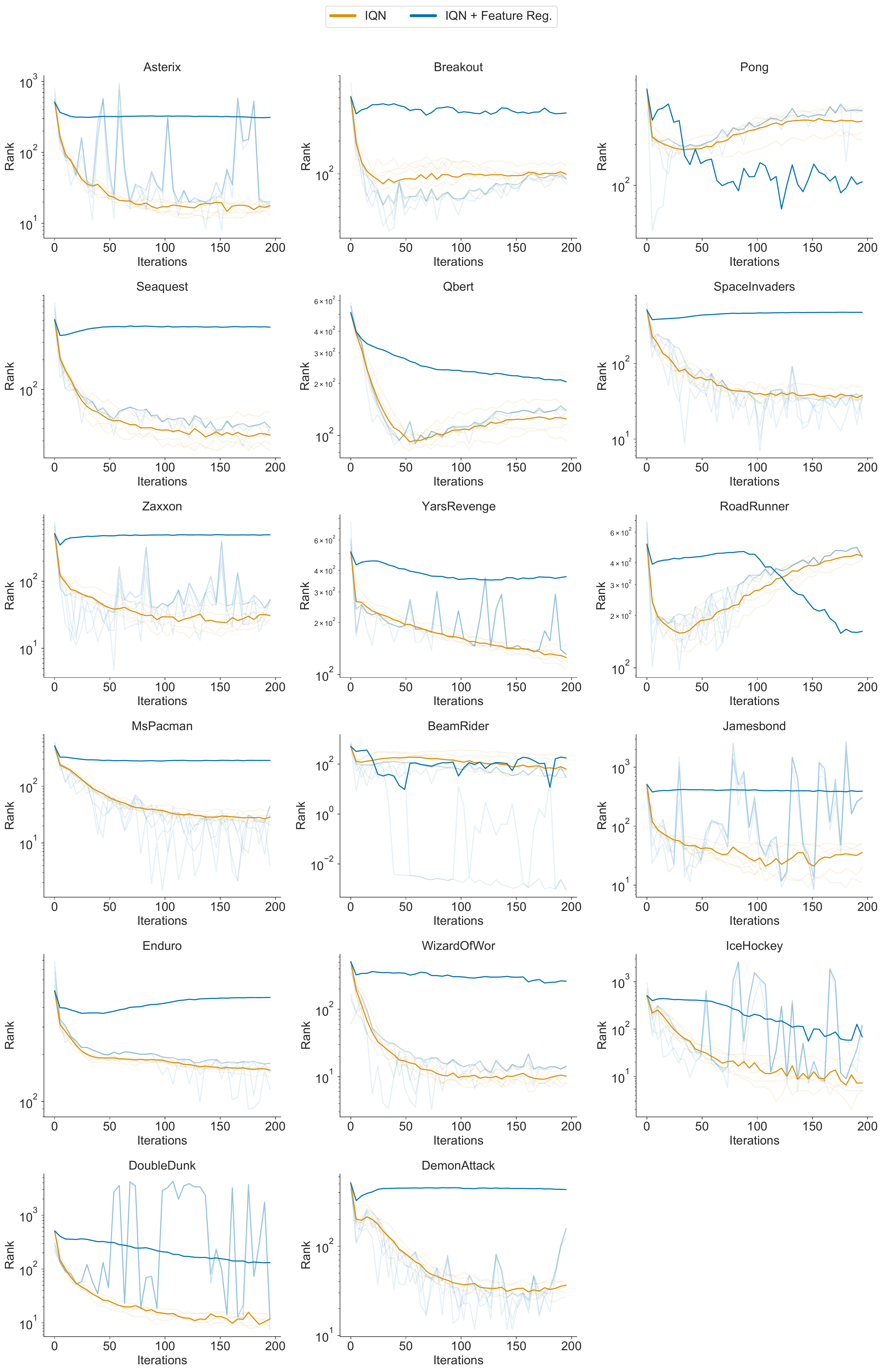}
  \caption{Per-game rank of IQN and IQN with feature regularization $L_\phi$ computed with a batch size $N=2^{15}$ on 17 Atari games in the offline RL setting, using 5 independent runs. Individual runs are shown with a lighter color.}
  \label{fig:atari_offline_pergame_rank}
\end{figure*}

\clearpage
\begin{figure*}[h!]
  \centering
  \includegraphics[width=\textwidth]{images/legend_offline.pdf}
     \includegraphics[width=0.4\textwidth]{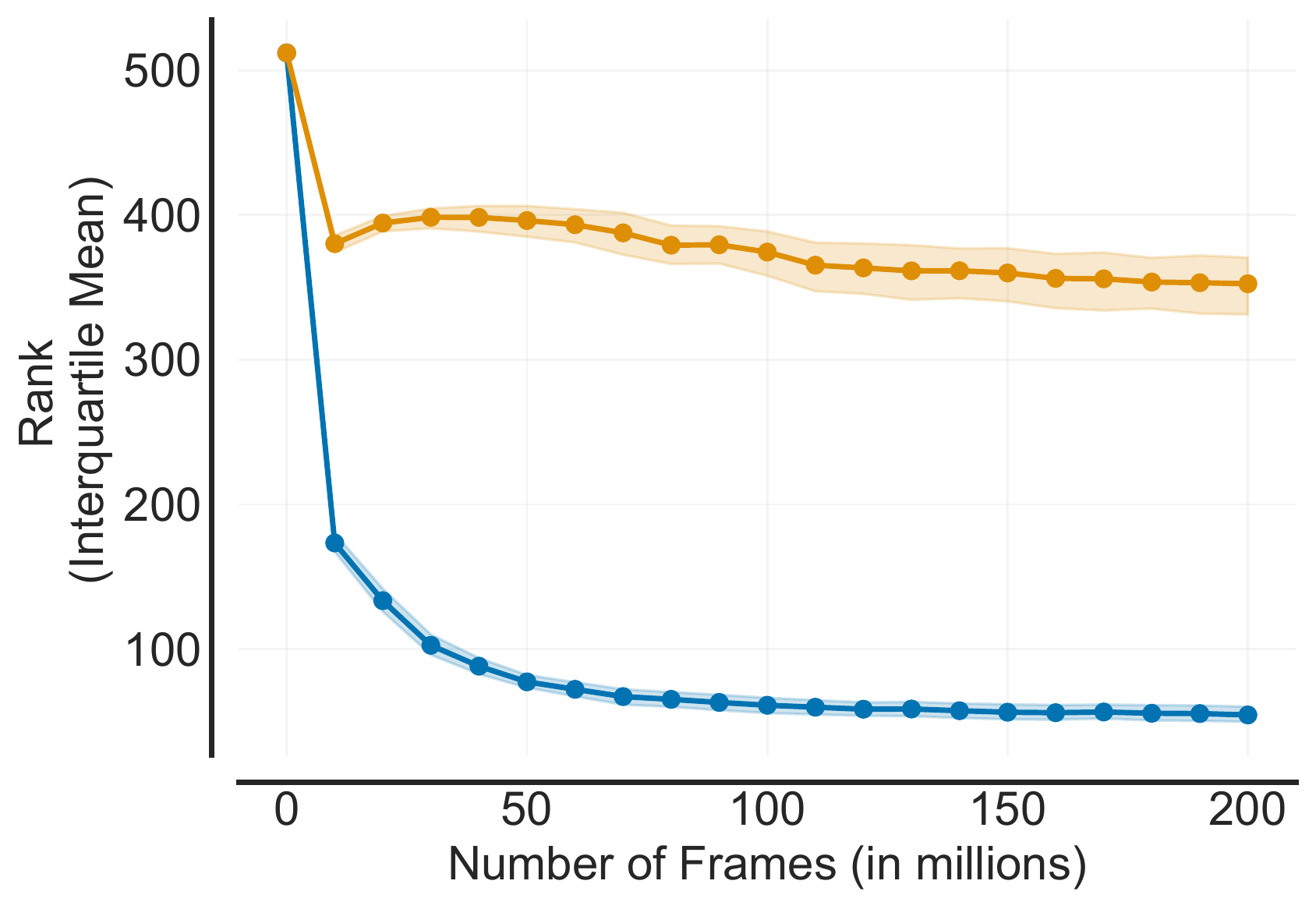}
  \caption{Interquartile mean (IQM)~\citep{agarwal2021deep} for the rank of representations induced by IQN and IQN with  feature regularization $L_\phi$ computed with a batch size $N=2^{15}$ on 17 Atari games in the offline setting.}
  \label{fig:atari_offline_IQM_rank}
\end{figure*}

\section{SOCIETAL IMPACT}
This paper contributes to the fundamental understanding of state representations, characterizing their generalization capacity. Our work suggests that algorithms making use of representations minimized by the excess risk bound from \cref{thm:main_gen_error} can improve their performance. However, when making the choice of such a representation, we did not focus on other important factors like the computational cost of learning these representations, their scalability or the biases these representations can propagate resulting into possible discriminatory outcomes or dangerous behaviours. We suggest that practitioners should not only consider our generalization characterization of representations but also ethical deliberations.

\end{appendix}

\end{document}